\documentclass[aos]{imsart-alt}

\RequirePackage{amsthm,amsmath,amsfonts,amssymb}
\RequirePackage[numbers,sort&compress]{natbib}
\RequirePackage[colorlinks,citecolor=blue,urlcolor=blue]{hyperref}
\RequirePackage{graphicx}
\RequirePackage{clipboard}

\RequirePackage{algorithm}
\RequirePackage[noend]{algpseudocode}
\RequirePackage{comment}
\RequirePackage{mathrsfs}

\startlocaldefs
\theoremstyle{plain}
\newtheorem{theorem}{Theorem}[section]
\newtheorem{lemma}[theorem]{Lemma}
\newtheorem{proposition}[theorem]{Proposition}
\newtheorem{corollary}[theorem]{Corollary}

\theoremstyle{definition}

\newclipboard{myclipboard}

\newcommand{\cut}[1]{}
\newtheorem{orule}{Orientation Rule}

\newcommand{\proofcase}[1]{\vspace{3pt} {\noindent \it Case #1:}\ }

\def\textversionnumber{0} 
\newcommand{\TextVersion}[2]{
\ifnum\textversionnumber=0%
#1%
\else 
#2%
\fi }

\newcommand{\reptheorem}[2]{
    \begin{theorem}
    \label{#1}
    \Copy{#1}{#2}
    \end{theorem}
}
\newcommand{\reftheorem}[1]{
    {\textsc{Theorem}~\ref{#1}} 
    \emph{\Paste{#1}}
}

\newcommand{\repnamedproposition}[3]{
    \begin{proposition}[#2]
    \label{#1}
    \Copy{#1}{#3}
    \end{proposition}
}
\newcommand{\repproposition}[2]{
    \begin{proposition}
    \label{#1}
    \Copy{#1}{#2}
    \end{proposition}
}
\newcommand{\refproposition}[1]{
    {\textsc{Proposition}~\ref{#1}} 
    \emph{\Paste{#1}}
}

\newcommand{\repcorollary}[2]{
    \begin{corollary}
    \label{#1}
    \Copy{#1}{#2}
    \end{corollary}
}
\newcommand{\refcorollary}[1]{
    {\textsc{Corollary}~\ref{#1}} 
    \emph{\Paste{#1}}
}

\DeclareMathSymbol{\mhyphen}{\mathord}{AMSa}{"39}

\newcommand{\Wlog}{Without loss of generality}

\newcommand\DefineConstant[1]{\text{\rmfamily\upshape #1}}

\newcommand{\set}[1]{\{\, #1 \,\}}
\newcommand{\pair}[2]{\seq{#1, #2}}
\newcommand{\triple}[3]{\seq{#1, #2, #3}}
\newcommand{\cardinality}[1]{\lvert #1 \rvert}

\newcommand{\powersetsymbol}{2}
\newcommand{\powerset}[1]{\powersetsymbol^{#1}}

\newcommand{\carteseanProduct}{\mathrel \times}

\newcommand{\carteseanBinary}[2]{{#1}\carteseanProduct{#2}}

\DeclareMathOperator{\true}{True}
\DeclareMathOperator{\false}{False}

\newcommand{\function}[2]{{#1} \mapsto {#2}}

\newcommand{\booleanfunction}[1]{\function{#1}{\set{\true,\false}}}

\newcommand{\characteristicfunctionS}[1]{\delta_{#1}}

\newcommand{\setcomprehensionLSE}[3]{\set{{#1}\in {#2}\ | \ {#3}}}

\DeclareMathOperator{\reverse}{rev}
\DeclareMathOperator{\last}{last}
\DeclareMathOperator{\first}{first}
\newcommand{\append}{\mathbin{+}}    

\newcommand{\seq}[1]{\left( #1 \right)} 

\newcommand{\seqLIJ}[3]{\seq{{#1}_{#2},\ldots,{#1}_{#3}}}
\newcommand{\seqvIJ}[2]{\seqLIJ{v}{#1}{#2}}

\DeclareMathOperator{\adj}{Adj}
\DeclareMathOperator{\ant}{Ant}
\DeclareMathOperator{\propant}{PropAnt}
\DeclareMathOperator{\post}{Post}

\newcommand{\edge}[2]{{e(#1, #2)}}
\newcommand{\edgeprime}[2]{{e'(#1, #2)}}

\newcommand{\shaft}{\text{---}}
\newcommand{\tailarrow}{\ \shaft\!\!\!\!\succ\! \ }
\newcommand{\arrowtail}{\ \!\prec\!\!\!\!\shaft \ }
\newcommand{\arrowarrow}{\ \!\prec\!\!\!\!\shaft\!\!\!\!\succ\! \ }
\newcommand{\tailtail}{\ \shaft \ }

\newcommand{\vertexset}[1]{\DefineConstant{Ver}({#1})}
\newcommand{\vertexseq}[1]{\DefineConstant{ver}({#1})}
\newcommand{\sectionseq}[1]{\DefineConstant{sec}({#1})}

\newcommand{\edgeset}{{\cal E}}
\newcommand{\walkset}{{\cal W}}
\newcommand{\walksetgraph}{W(G)}

\newcommand{\miw}{minimal inducing walk}
\newcommand{\Miw}{Minimal inducing walk}

\newcommand{\mdiw}{minimal discriminating inducing walk}

\newcommand{\iw}{inducing walk}

\newcommand{\diw}{discriminating inducing walk}
\newcommand{\siw}{self-inducing walk}

\newcommand{\defwalkWLIJ}[4]{\vertexseq{#1}=\seqLIJ{#2}{#3}{#4}}

\newcommand{\defwalkomega}{\defwalkWLIJ{\omega}{v}{1}{n}}
\newcommand{\defwalkomegaprime}{\defwalkWLIJ{\omega'}{v}{1}{n}}

\newcommand{\defsectionWLIJ}[4]{\sectionseq{#1}=\seq{{#2}_{#3}, \ldots, {#2}_{#4}}}

\newcommand{\defsectionomega}{\defsectionWLIJ{\omega}{\sigma}{1}{m}}

\newcommand{\defsectionomegaprime}{\defsectionWLIJ{\omega'}{\sigma'}{1}{m}}

\newcommand{\indepsym}{\perp\!\!\!\!\perp}
\newcommand{\depsym}{{\hspace{.3em} \indepsym \hspace{-1.11em} \raisebox{0.16em}{$\smallsetminus$} \hspace{0.18em}}}
\newcommand{\indep}[4]{{#1} \indepsym_{#4} {#2}\ |\ {#3}}
\newcommand{\dep}[4]{{#1} \depsym_{#4}\  {#2}\ |\ {#3}}

\newcommand{\equivclass}[1]{[#1]}
\newcommand{\equivclassF}[1]{[#1]_\GraphFamilyF}
\newcommand{\equivclassG}[1]{[#1]_\GraphFamily}

\newcommand{\imodel}[1]{{I}{(#1)}}

\newcommand{\pairwise}{\DefineConstant{Pairwise}}

\newcommand{\cgp}{canonical graph projection}

\newcommand{\cwp}{collider walk projection}

\newcommand{\adjrep}{\DefineConstant{ADJ}}
\newcommand{\miwrep}{\DefineConstant{MIW}}
\newcommand{\mdiwrep}{\DefineConstant{MDIW}}
\newcommand{\iarrowrep}{\DefineConstant{IARROW}}

\newcommand{\GraphFamily}{{\cal G}}
\newcommand{\GraphFamilyF}{{\cal F}}

\newcommand{\GraphFamilyH}{{\cal H}}

\newcommand{\AncestralFamily}{{\cal G}_{anc}}
\newcommand{\MaxAncestralFamily}{{\cal G}_{anc}^{max}}

\newcommand{\AnterialFamily}{{\cal G}_{ant}}
\newcommand{\MaxAnterialFamily}{{\cal G}_{ant}^{max}}
\newcommand{\CMGFamily}{{\cal G}_{cm}}
\newcommand{\MaxCMGFamily}{{\cal G}_{cm}^{max}}

\newcommand{\SepAnterialFamily}{{\cal G}_{ant}^{sep}}

\newcommand{\SimpleFamily}{\GraphFamily_{simp}}
\newcommand{\AcycFamily}{\GraphFamily_{acyc}}
\newcommand{\MaxFamily}{\GraphFamily^{max}}

\newcommand{\UndirectedFamily}{\GraphFamily_{und}}

\newcommand{\EssentialFamily}[1]{\mathcal{E}(#1)}
\newcommand{\EssentialFamilyF}{\EssentialFamily{\GraphFamilyF}}

\newcommand{\SepFamily}{{\cal G}_{sep}}
\newcommand{\WSepFamily}{\EssentialFamily{\SepFamily}}

\newcommand{\DistFamily}{{\cal P}}

\newcommand{\indeptest}{independence test}

\newcommand{\statmodel}[1]{M_{#1}}

\newcommand{\statmodelwithDistFamily}[1]{M_{#1}^\DistFamily}

\newcommand{\gim}{graphical model}
\newcommand{\gimf}{graphical model family}

\newcommand{\gimfs}{graphical model families}

\newcommand{\gmfF}{\ModelFamily_{\GraphFamilyF}}

\newcommand{\gmfH}{\ModelFamily_{\GraphFamilyH}}

\newcommand{\gmfWSep}{\ModelFamily_{\WSepFamily}}

\newcommand{\gmfAcyc}{\ModelFamily_{\AcycFamily}}
\newcommand{\gmfSepAnterial}{\ModelFamily_{\SepAnterialFamily}}

\newcommand{\genstruct}{G^*}
\newcommand{\gendist}{P^*}

\newcommand{\sia}{structure identification algorithm}
\newcommand{\Sia}{Structure identification algorithm}
\newcommand{\gsia}{graphical structure identification algorithm}

\newcommand{\ModelFamily}{{\cal M}}

\newcommand{\simplify}{\DefineConstant{Anterialize}}

\newcommand{\IV}{\operatorname{IV}}
\newcommand{\IS}{\operatorname{IS}}

\newcommand{\edgedb}{\edge{d}{b}}
\newcommand{\edgeprimedb}{\edgeprime{d}{b}}

\newcommand{\iarrowhead}{induced arrowhead}

\newcommand{\inducedarrowheads}{\DefineConstant{InducedArrowheads}}

\newcommand{\sgi}{\DefineConstant{SGI}}

\newcommand{\equivclassSRL}[3]{\left[{#3}\right]_{#1}^{#2}}
\newcommand{\equivclassSL}[2]{\left[{#2}\right]_{#1}}

\newcommand{\largest}{\DefineConstant{Largest}}
\newcommand{\graphtransform}{graph transform}

\newcommand{\antsetX}[1]{\ant(#1,G)}
\newcommand{\antsetij}{\antsetX{\set{i,j}}}
\newcommand{\antsetCij}{\antsetX{C\cup\set{i,j}}}
\newcommand{\Cij}{C\cup \set{i,j}}
\newcommand{\propantij}{\ant(\set{i,j},G)\setminus \set{i,j}}
\newcommand{\propantCij}{\propant(C\cup\set{i,j},G)}
\newcommand{\Cmij}{C\setminus \set{i,j}}

\newcommand{\defwalktau}{\defwalkWLIJ{\tau}{v}{1}{n}}

\newcommand{\walkedgeLI}[2]{\edge{#1_{#2}}{#1_{#2 + 1}}}
\newcommand{\walkedgev}[1]{\walkedgeLI{v}{#1}}
\newcommand{\walkedgeprimeLI}[2]{\edgeprime{#1_{#2}}{#1_{#2 + 1}}}
\newcommand{\walkedgeprimev}[1]{\walkedgeprimeLI{v}{#1}}

\newcommand{\wdf}{walk decomposition function}
\newcommand{\wcf}{walk composition function}

\newcommand{\decomp}{\DefineConstant{decomp}}
\newcommand{\inverse}[1]{#1^{-1}}

\newcommand{\wdecomp}{\ensuremath{\decomp_w}}
\newcommand{\walkdecompI}[1]{\seq{\tau_1, \gamma_1, \tau_2,\ldots, \gamma_{#1-1},\tau_#1}}
\newcommand{\defwalkdecompWI}[2]{\wdecomp(#1)=\walkdecompI{#2}}
\newcommand{\defwalkdecompomegaI}[1]{\defwalkdecompWI{\omega}{#1}}

\newcommand{\maxdecompose}{\ensuremath {\DefineConstant{\decomp}_{m}}}
\newcommand{\maxcompose}{\ensuremath \inverse{\maxdecompose}}

\newcommand{\iwdecompose}{\ensuremath {\DefineConstant{decomp}_{iw}}}
\newcommand{\iwcompose}{\ensuremath \inverse{\iwdecompose}}

\newcommand{\edgevIJ}[2]{\edge{v_{#1}}{v_{#2}}}
\newcommand{\edgevivj}{\edge{v_i}{v_j}}

\newcommand{\edgeab}{\edge{a}{b}}
\newcommand{\edgewalk}[2]{\seq{\edge{#1}{#2}}}

\newcommand{\walkedgeundirected}[1]{v_{#1} \tailtail v_{#1 + 1}}

\newcommand{\walkedgerightarrow}[1]{v_{#1} \tailarrow v_{#1 + 1}}

\newcommand{\antomegaendpoints}{\ant(\set{v_1,v_n},G)}
\newcommand{\propantomegaendpoints}{\propant(\set{v_1,v_n},G)}

\newcommand{\mcw}{maximal collider walk}
\newcommand{\mwd}{maximal walk decomposition}
\newcommand{\mt}{maximal trek}

\newcommand{\defgraphG}{G=\seq{V,E}}

\endlocaldefs

\begin{document}
\begin{frontmatter}
\title{Characterizing and identifying\\ separable graphical models}
\runtitle{Separable graphical models}

\begin{aug}
\author[A]{\fnms{Christopher}~\snm{Meek} \ead[label=e1]{cameek@uw.edu}} \and
\author[B]{\fnms{Kayvan}~\snm{Sadeghi}\ead[label=e2]{k.sadeghi@ucl.ac.uk}}
\address[A]{Department of Statistics,
University of Washington\printead[presep={ ,\ }]{e1}}

\address[B]{Department of Statistical Science,
University College London \printead[presep={,\ }]{e2}}
\end{aug}

\begin{abstract}

We study a broad class of graphical models whose independencies correspond to vertex separation in mixed graphs with directed, undirected, and bidirected edges, that are capable of encoding independence structures arising from feedback, latent and selection mechanisms. 
In particular, we introduce separable graphs, in which each missing edge implies the existence of a separating set for its endpoints, and essentially separable graphs, those graphs separation equivalent to a separable graph.
We show that these models include many existing graph families used to define graphical models an provide several characterizations of separable graphs and essentially separable graphs.
We also provide multiple characterizations of separation equivalence for separable graphs.
One is a graphical characterization—extending earlier results for specific subfamilies—in terms of ordinary graph properties.
Another is a separational characterization, depending only on graph separation properties.
Finally, we provide a canonical representation for the equivalence classes of essentially separable graphs and develop an algorithm that, under suitable assumptions, identifies the equivalence class of any essentially separable graph.

\TextVersion{}
{

We study a broad class of graphical models represented by mixed graphs with directed, undirected, and bidirected edges, capable of encoding independence structures arising from feedback, latent and selection mechanisms. 
We introduce separable graphical models, in which each missing edge corresponds to a conditional independence, and essentially separable graphical models, whose independence structure is representable by a separable graphical model.
We characterize separable and essentially separable graphs under generalized d-separation and $\sigma$-separation, and show that every separation equivalence class of separable graphs contains a representative with at most one edge between any pair of vertices. We also provide several equivalent characterizations of the separation equivalence of separable graphs and develop algorithms for testing the equivalence of two separable graphical models and for identifying the equivalence class of any essentially separable graphical model.}

\end{abstract}



\end{frontmatter}


\section{Introduction}
Graphical models provide a simple intuitive means of describing sets of independence facts that arise in different statistical contexts. A graphical model consists of a graph and a statistical model --- a family of probability distributions --- such that each probability distribution in the family satisfies a set of statistical independence facts defined by vertex separation in the graph. When considering probability distributions over product measurable spaces with index set V,\footnote{A measurable space $\seq{\times_{v\in V}\mathscr{X}_v, \bigotimes_{v\in V}\mathscr{F}_v}$ where, for $v\in V$, $\mathscr{F}_v$ is a $\sigma$-field over the set $\mathscr{X}_v$.} one uses a graph $G=\seq{V,E}$ whose vertices correspond to the index set $V$ and a set of edges $E$ in conjunction with a separation criterion to describe the set of independence facts that are required to hold. Such a separation criterion defines the global Markov properties of a graphical model defined by the graph G.

Different classes of graphs arise naturally in different applications. For example, undirected graphs arise in physics \cite{Gibbs1902}, 
directed acyclic graphs arise in genetics \cite{Wright1925}, and chain graphs generalize these classes \cite{Lauritzen1996}. Cyclic graphs arise when modeling systems with feedback \cite{Richardson1996, Koster1996}. Graphs with bidirected edges and undirected edges arise in systems in which there are latent and selection processes \cite{sgs,RichardsonSpirtes2002}.
Different separation criteria also arise naturally in different contexts. The most common separation criterion is d-separation and its various generalizations (e.g., m-separation, c-separation). Additional examples of separation criteria are $\sigma$-separation used for non-linear feedback systems \cite{Spirtes1995cyclic,ForreMooij_UAI_19}, p-separation used for modeling independence among regression residuals \cite{AMP2001}, and $\delta$-separation used for dynamical systems \cite{Didelez2008}.
In this paper, we consider mixed graphs that allow for multiple edges between any pair of vertices where an edge can either be undirected, directed, or bidirected edges 
and a generalization of d-separation. The family of graphical models defined by these mixed graphs can model feedback, latent and selection processes and includes as subclasses many existing classes of graphical models.

We introduce two types of mixed graphs: separable graphs and essentially separable graphs. Separable graphs are those in which every missing edge between two vertices corresponds to the existence of a separating set for the two vertices. Essentially separable graphs are those graphs that are separation equivalent to some separable graph.
Separable and essentially separable graphs unify a large set of existing graphs used to define graphical models. Separable graphical models generalize the families of undirected, directed, chain graphical models, maximal ancestral graphs, maximal chain mixed graphs, and maximal anterial graphs.
Essentially separable graphical models additionally include ancestral \cite{RichardsonSpirtes2002}, anterial \cite{Sadeghi2016}, and chain mixed graphical models \cite{LauritzenSadeghi2018}. As both separable and essentially separable graphs include some cyclic graphs, graphical models defined in terms of these families include models that can be used to describe systems with feedback.

The defining property of separable graphs has appeared previously \cite{RichardsonSpirtes2002,LauritzenSadeghi2018} as an 
emergent property of a maximization process in which one adds edges to a graph to form an equivalent larger graph until one cannot further augment the graph. 
In particular,  \cite{RichardsonSpirtes2002} shows that maximal ancestral graphs and \cite{LauritzenSadeghi2018} shows that maximal chain mixed graphs have the emergent property that every missing edge corresponds to the existence of a separating set.
In contrast, we use the property to define separable graphs. This allows us to define a more general family of graphs including some cyclic graphs.
We show, however, that for the class of mixed graphs, a maximal graph need not be separable under d-separation.

We provide characterizations of separable graphs and essentially separable graphs.
In one graphical characterization we show that essentially separable graphs are essentially acyclic graphs, that is, they are separation equivalent to an acyclic graph. Furthermore, we show that the family of anterial graphs introduced by \cite{Sadeghi2016} is essentially equivalent to the family of essentially separable graphs, that is, 
the family of statistical models associated with the graphical models defined by anterial graphs is identical to the family of statistical models associated with essentially separable models.
This result has significant potential computational benefit due to the fact that anterial graphs are simple acyclic graphs.

We provide graphical characterizations of several graph families including the family of essentially undirected graphs; those graphs that are separation equivalent to some undirected graph. A graphical characterization of chain graphs that are equivalent to undirected graphs was provided by \cite{wermuth1990substantive}. Unlike this characterization undirected graphs, our characterization is purely based on the separation properties of the graph. In particular, we introduce the concept of an inducing vertex; a vertex that induces a dependence between separated sets of vertices when added to a separating set.
We show that a graph is essentially undirected if and only if it contains no inducing vertex. In addition to its utility in this characterization, we show that the existence of an inducing vertex in a graph provides information about anterior relationships between the inducing vertex and other vertices in the graph. We also provide a characterization of essential graphs in terms of their separation properties.

We also provide three characterizations of separable equivalence for separable graphs, describe a canonical representation for the equivalence class of essentially separable graphs, and provide a \sia\ for essentially separable graphical models.
The first characterization of separation equivalence is in terms of two graph having the same adjacencies and the same \miw s. This result is a direct generalization of the characterizations of separation equivalence for directed acyclic graphs \cite{VP1990}, chain graphs \cite{Frydenberg1990}, and ancestral graphs \cite{Zhao2005} to separable graphs. Our second characterization shows that we only need to consider adjacency and a subset of \miw s called \diw s.

Our final characterization is given in terms of two separational properties of pairs of vertices; vertex separability and the induced arrowhead property. 
Our final characterization is that two separable graphs are separation euqivalent if and only if the graphs have the same vertex separability and the same \iarrowhead s.
This characterization is particularly useful for developing \sia\ as it is a separational characterization which implies that, in principle, an \sia\ using independence tests can identify the equivalence class based only on the properties used in the characterization.

Our representation for the equivalence classes of separable graphs relies on this characterization and, particularly, the induced arrowhead property. We show that we can represent the equivalence class of an essentially separable graph by a graph with the same adjacencies as a separable graph in the equivalence class and whose only arrowheads on edges are induced arrowheads.

Finally, our \sia\ for essentially separable graphs, the \sgi\ algorithm, relies on identifying separating sets and inducing vertices. The \sgi\ algorithm generalizes the CI and FCI algorithms of \cite{sgs}.
In particular, the CI and FCI algorithms identify the structure of graphical models defined by ancestral graphs, whereas the \sgi\ identifies the structure of a graphical model defined by separable graphs. The algorithm is conceptually simple: the algorithm searches for separating sets for pairs of vertices. If a separating set is found, the edge between the separated vertices is remove edges. Then the separating set is used to identify inducing vertices. If inducing vertices are found, they are are used to orient edges. The algorithm then continues to search for separating sets.  We provide sufficient conditions for the algorithm to correctly identify the equivalence class of an essentially separable graph and a polynomial bound on the number of independence tests required.

\newcommand{\objectDomain}{{\cal O}}
\newcommand{\objectInDomain}{o}
\newcommand{\triples}{{\cal T}}
\newcommand{\testsForDomain}{\triples_{\objectDomain}}
\newcommand{\testsForV}{\triples_{V}}
\newcommand{\indepmodel}{independence model}
\newcommand{\Indepmodel}{Independence model}
\newcommand{\defIndepModelIForObjectDomain}{I=\seq{\objectDomain,T}}
\newcommand{\testFunctionFromIndepModel}[1]{\testFunctionName_{#1}}
\newcommand{\testcaseABC}{\seq{A,B,C}}

\newcommand{\indepABCI}{\indep{A}{B}{C}{I}}
\newcommand{\depABCI}{\dep{A}{B}{C}{I}}
\newcommand{\indepABC}{\indep{A}{B}{C}{}}
\newcommand{\indepabC}{\indep{a}{b}{C}{}}

\newcommand{\depABC}{\dep{A}{B}{C}{}}
\newcommand{\depabC}{\dep{a}{b}{C}{}}

\newcommand{\indepABCG}{\indep{A}{B}{C}{G}}
\newcommand{\depABCG}{\dep{A}{B}{C}{G}}
\newcommand{\indepabCG}{\indep{a}{b}{C}{G}}
\newcommand{\depabCG}{\dep{a}{b}{C}{G}}

\newcommand{\indepABCP}{\indep{A}{B}{C}{P}}
\newcommand{\depABCP}{\dep{A}{B}{C}{P}}
\newcommand{\indepabCP}{\indep{a}{b}{C}{P}}
\newcommand{\depabCP}{\dep{a}{b}{C}{P}}

\newcommand{\indepstatement}{independence statement}
\newcommand{\depstatement}{dependence statement}
\newcommand{\indepprop}{independence property}
\newcommand{\sepequivalent}{separation equivalent}

\newcommand{\AllIModels}{{\cal I}}

\newcommand{\markov}[1]{\DistFamily_{#1}}

\section{Preliminaries}

In this section we introduce the key objects of study in this paper including \indepmodel s, statistical \indepmodel s  induced by probability distributions, vertex \indepmodel s induced by mixed graphs, mixed graphical models, and the problem of structure identification of graphical models.

\subsection{Independence models}
\label{sec:indepmodels}
An \emph{independence model} is a representation of the 
independencies and dependencies that hold in a system of objects $\objectDomain$.
Such models are well-studied (e.g., \cite{studeny2005}) and used in a variety of contexts including two contexts relevant to this paper; vertex separation in graphs --- a system of vertices, and statistical independence in probability distributions --- a system of variables. 
An \emph{independence model} $\defIndepModelIForObjectDomain$ where $\objectDomain$ is a set of objects where $T\subseteq \testsForDomain$ is a set of triples 
and  $\testsForDomain$ is the set of all 
possible triples $\seq{A,B,C}$ of disjoint subsets of $\objectDomain$.  
For \indepmodel\ $\defIndepModelIForObjectDomain$ and triple $\testcaseABC$,
if $\testcaseABC\in T$, then we say that $A$ is independent of $B$ given $C$ and write $\indepABCI$. If $\testcaseABC \not \in T$ then we say $A$ and $B$ are dependent given $C$  and write $\depABCI$.  
For two independence models $I=\seq{\objectDomain, T}$ and $I'=\seq{\objectDomain',T'}$, we define $I\subseteq I'$ if $\objectDomain \subseteq \objectDomain$ and $T\subseteq T'$.
Finally, we denote the set of all independence models by $\AllIModels$.
\TextVersion{}{
An \emph{independence model} $\defIndepModelIForObjectDomain$ where $\objectDomain$ is a set of objects and $T\subseteq \testsForDomain$ is a set of triples where $\testsForDomain$ is the set of all possible triples $\seq{A,B,C}$ of disjoint subsets of $\objectDomain$.  
For \indepmodel\ $\defIndepModelIForObjectDomain$ and \testcase\ $\testcaseABC$,
if $\testcaseABC\in T$, then we say that $A$ is independent of $B$ given $C$ and write $\indepABCI$. If $\testcaseABC \not \in T$ then we say $A$ and $B$ are dependent given $C$  and write $\depABCI$.
The \emph{\itf} for independence model $\defIndepModelIForObjectDomain$ is the 
Boolean function $\testFunctionFromIndepModel{I} = \characteristicfunctionS{T}$, the characteristic function of the set $T$. 
}
\subsubsection{Independence properties}

An \emph{\indepprop} is a logical formula of independence and dependence statements where upper-case letters indicate sets of variables, lower-case letters indicate a singleton set of variables and distinct letters in a formula indicate that the sets are disjoint. Furthermore, two or more adjacent letters are used to indicate union of those sets. For instance, $BD$ in the weak union property in Table~\ref{tab:semi-graphoid} is used in place of $B\cup D.$ 
An independence property \emph{holds} in an \indepmodel\ for $\objectDomain$ if for every possible assignment of disjoint sets of $\objectDomain$ to letters in the logical formula, the logical formula is a true statement about the \indepmodel. 
For instance, the symmetry property in Table~\ref{tab:semi-graphoid} holds in the independence model $I$ for $V$ if for all $\seq{A,B,C}\in \testsForV$ it is the case that if $\indepABCI$ then $\indep{B}{A}{C}{I}$.
Table~\ref{tab:semi-graphoid} contains a set of independence properties that are called the \emph{semi-graphoid} properties.\cite{pearl1988}
An \indepmodel\ $I$ is a \emph{semi-graphoid} if all of the semi-graphoid properties hold in $I$.

\begin{table}[ht]
\centering
\begin{tabular}{l l l }
$\indep{A}{\emptyset}{C}{}$ & & trivial \\
$\indep{A}{B}{C}{}$ &$\implies \indep{B}{A}{C}{}$ & symmetry \\
$\indep{A}{BD}{C}{}$ &$\implies \indep{A}{B}{C}{} \wedge \indep{A}{D}{C}{}$ & decomposition \\
$\indep{A}{BD}{C}{}$ &$\implies \indep{A}{B}{CD}{}$ & weak union \\
$\indep{A}{B}{C}{} \wedge \indep{A}{D}{CB}{}$ & $\implies \indep{A}{BD}{C}{}$ & contraction \\
\end{tabular}
\caption{The semi-graphoid independence properties.}
\label{tab:semi-graphoid}
\end{table}

Table~\ref{tab:additional-independence-properties} contains 
several additional independence properties. An independence model $I$ is a \emph{compositional graphoid} if the semi-graphoid properties given in Table~\ref{tab:semi-graphoid} hold in $I$ and the intersection and composition properties of given in Table~\ref{tab:additional-independence-properties} hold in $I$.

\begin{table}[ht]
\centering
\begin{tabular}{l l l }
$\indep{A}{B}{CD}{I} \wedge \indep{A}{D}{CB}{I}$ & $\implies \indep{A}{BD}{C}{I}$ & intersection \\
$\indep{A}{B}{C}{I} \wedge \indep{A}{D}{C}{I}$ & $\implies \indep{A}{BD}{C}{I}$ & composition \\
$\indep{A}{B}{C}{I} \ \wedge \ \indep{A}{B}{Cd}{I}$ & $\implies \indep{A}{d}{C}{I} \ \vee \ \indep{d}{B}{C}{I}$ & weak transitivity\\
\end{tabular}
\caption{Additional independence properties}
\label{tab:additional-independence-properties}
\end{table}

\newcommand{\compindepmodel}{dyadic \indepmodel}
\newcommand{\Compindepmodel}{Dyadic \indepmodel}

\subsubsection{\Compindepmodel s}

\newcommand{\pairwiseclosure}{pairwise closure}
\newcommand{\pairwiseclosureFunc}{\DefineConstant{closure}}
\newcommand{\indepmodelprojection}{\indepmodel\ projection}
\newcommand{\depmodelequivalent}{equivalent}
\newcommand{\pairwisedepmodelequivalent}{pairwise \depmodelequivalent}

\TextVersion{
Next we provide a useful characterization of the equivalence of two \indepmodel s.
An \indepmodel\ $I$ is a \emph{\compindepmodel} if the trivial, symmetry, composition, and decomposition properties hold in $I$.
The function $\pairwise$ is the \indepmodelprojection\ that maps an \indepmodel\ to the \indepmodel\ consisting of all and only the pairwise triples that hold in the given \indepmodel. In particular,
$\pairwise(\seq{\objectDomain,T})=\seq{\objectDomain,\setcomprehensionLSE{\testcaseABC}{T}{|A|=1 \wedge |B|=1}}$.
Two \indepmodel s $I$ and $I'$ are \emph{\pairwisedepmodelequivalent} if $\pairwise(I)=\pairwise(I')$.

The following proposition shows that for \compindepmodel s, equivalence and pairwise equivalence  are themselves equivalent.\footnote{Note that if
a theorem, proposition, lemma, or corollary has no citation then its proof can be found in an Appendix.}

\repproposition{prop:sep-equiv::pairwise-sep-equiv}{
If $I$ and $I'$ are \compindepmodel s then $I$ and $I'$ are equivalent (i.e., $I = I'$) if and only $I$ and $I'$ are \pairwisedepmodelequivalent\ (i.e., $\pairwise(I)=\pairwise(I')$).}

The importance of Proposition~\ref{prop:sep-equiv::pairwise-sep-equiv} is that one need only verify pairwise independence facts to guarantee equivalence of any \compindepmodel.

}{

\newcommand{\cep}{conditional projection}

\newcommand{\indepmodelCondExpName}{\DefineConstant{E}}
\newcommand{\indepmodelConditionalExpectation}[2]{\indepmodelCondExpName^{#1}_{#2}}
\newcommand{\indepmodelLS}[1]{\indepmodelConditionalExpectation{L}{S}({#1})}
\newcommand{\indepmodelES}[1]{\indepmodelConditionalExpectation{\emptyset}{S}({#1})}
\newcommand{\indepmodelLE}[1]{\indepmodelConditionalExpectation{L}{\emptyset}({#1})}
\newcommand{\indepmodelEE}[1]{\indepmodelConditionalExpectation{\emptyset}{\emptyset}({#1})}

\subsubsection{Marginalization and conditioning}

Marginalization and conditioning are important operations for probability  distributions and graphical models. These operations are useful when one considers models with latent variables and selection bias \cite{spirtes1995causal, RichardsonSpirtes2002}. We define the 
\cep\ of an \indepmodel\ that captures both the marginalization and conditioning operations; this projection was defined by \cite{RichardsonSpirtes2002} but was not named.
The \emph{\cep} of \indepmodel\ $\indepmodelLS{I}$ of a \indepmodel\ $\defIndepModelIForObjectDomain$ is 
$$\indepmodelLS{\seq{\objectDomain,T}}=\seq{\objectDomain\setminus (L\cup S),T^L_S}$$
where
$$T^L_S = \setcomprehensionLSE{\testcaseABC}{\testsForDomain}{\seq{A,B,C\cup S}\in T \wedge (A\cup B \cup C)\cap (L \cup S) = \emptyset}.$$

Examples of the \cep\ applied to an \indepmodel\ $\defIndepModelIForObjectDomain$ where $\objectDomain=\set{a,b,c,d}$ and $T=\set{\seq{a,b,d},\seq{c,d,\set{a,b}}}$ are given in Table~\ref{tab:imodel-example}.

\begin{table}[ht]
\centering
\begin{tabular}{l l l }
\vspace{0.03in}
$\indepmodelEE{I}$ & $=$ & 
$\seq{\objectDomain,T}$\\
\vspace{0.03in}
$\indepmodelConditionalExpectation{d}{\emptyset}(I)$ & $=$ & 
$\seq{\set{a,b,c},\emptyset}$\\
\vspace{0.03in}
$\indepmodelConditionalExpectation{\emptyset}{d}(I)$ & $=$ & 
$\seq{\set{a,b,c},\set{\seq{a,b,\emptyset}}}$\\
\vspace{0.03in}
$\indepmodelConditionalExpectation{\emptyset}{c}(I)$ & $=$ & $\seq{\set{a,b,d},\emptyset}$\\
\vspace{0.03in}
$\indepmodelConditionalExpectation{c}{\emptyset}(I)$ & $=$ & $\seq{\set{a,b,d},\set{\seq{a,b,d}}}$\\
\vspace{0.03in}
$\indepmodelConditionalExpectation{d}{c}(I)$ & $=$ & $\seq{\set{a,b},\emptyset} 
= \indepmodelConditionalExpectation{cd}{\emptyset}(I) 
= \indepmodelConditionalExpectation{\emptyset}{cd}(I)$ 
\\
\vspace{0.03in}
$\indepmodelConditionalExpectation{c}{d}(I)$ & $=$ & $\seq{\set{a,b},\set{\seq{a,b,\emptyset}}} $
\\

\end{tabular}
\caption{Examples of the \cep\ applied to an \indepmodel\ $I$.}
\label{tab:imodel-example}
\end{table}

The \emph{marginalization} over $L$ of a \indepmodel\ $I$ is the \indepmodel\ $\indepmodelLE{I}$ and the \indepmodel\ obtained by \emph{conditioning} on $S$ is $\indepmodelES{I}$.
A notable property of the \cep\ is that it removes the sets of objects $L$ and $S$ from an \indepmodel\ implying that the \indepmodel\ resulting from an \cep\ can only be used to answer dependence and independence questions about the objects that remain. 

Some basic properties the \cep\ that follow from basic properties of sets are given in the next proposition.

\repproposition{prop:cep-properties}{
For disjoint subsets $L_1, L_2,S_1,S_2$ of $\objectDomain$ and \indepmodel\ $\defIndepModelIForObjectDomain$

\begin{center}
\begin{tabular}{l l l l}
\vspace{0.03in}
$\indepmodelEE{I}$ & $=$ & $I$ & (1)\\
\vspace{0.03in}
$\indepmodelConditionalExpectation{L_1\cup L_2}{S_1\cup S_2}$ & $=$
& $\indepmodelConditionalExpectation{L_1}{S_1}\circ \indepmodelConditionalExpectation{L_2}{S_2}$ & (2) \\
\vspace{0.03in}
$\indepmodelConditionalExpectation{L_1}{S_1}\circ \indepmodelConditionalExpectation{L_2}{S_2}$ & $=$ & $\indepmodelConditionalExpectation{L_2}{S_2}\circ \indepmodelConditionalExpectation{L_1}{S_1}$ & (3)
\end{tabular}
\end{center}
}

}

\subsection{Distributions and statistical \indepmodel s}

We use $P$ to denote a probability distribution over variables $V$ and assume that $P$ is from some \emph{defining family} of probability distributions $\DistFamily$.\footnote{More formally, the defining family has a common product measurable space $\seq{\times_{v\in V}\mathscr{X}_v,\bigotimes_{v\in V}\mathscr{F}_v}$ indexed by $V$. Typically, one also assumes that there is a common dominating measure but this is not required here.}
Note that we will use $V$ to refer to both a set of vertices in a graph and a set of variables. The statistical independence or statistical dependence of two sets of variables given a third set of variables is a property of a distribution. We denote the fact that $A$ and $B$ are independent given $C$ in $P$ by $\indepABCP$ and denote the fact that $A$ and $B$ are dependent given $C$ in $P$ by $\depABCP$. We use statistical independence to define the \emph{statistical \indepmodel} for distribution $P$ as  
$$\imodel{P}=\setcomprehensionLSE{\testcaseABC}{\testsForV}{\forall a\in A \ \forall b\in B \ \indepabCP}.$$

A statistical \indepmodel\ derived from a distribution is guaranteed to satisfy  the semi-graphoid independence properties as shown in the following proposition. A proof of this can be found in \cite{studeny2005}.

\repproposition{prop:compprob}{
    If $P$ is a probability distribution over $V$ then 
    $\imodel{P}$ is a semi-graphoid.}

\subsection{Mixed graphs}\label{sec:graphs}
The family of \emph{mixed graphs} $\GraphFamily$ is the family of loop-less mixed multigraphs with the symmetric edges $\set{\tailtail,\arrowarrow}$ and asymmetric edge $\set{\tailarrow}$ over a set of vertices $V$.
A graph $G=\seq{V,E} \in \GraphFamily$ has a finite set of vertices $V$ and a finite set of edges $E \subseteq \edgeset = V\times T \times V$ where edges have a type $t\in T=\set{\tailtail, \arrowarrow, \tailarrow}$, 
there are no \emph{loops} (i.e., $\seq{a,t,a}\not\in E$), and the set of edges is \emph{symmetric} (i.e., if $\seq{a,t,b}\in E$ and $t\in \set{\tailtail, \arrowarrow}$ then $\seq{b,t,a} \in E$). The set $\edgeset$ is the set of all possible edges with endpoints in $V$. 

\subsubsection{Edges and endmarks}
An edge $\seq{a,t,b}\in E$ \emph{connects} its \emph{endpoints} $a$ and $b$ and has \emph{edge type} $t\in \set{-,\arrowarrow, \tailarrow}$.
A pair of distinct vertices $a,b \in V$ in a graph $G\in \GraphFamily$ can have up to four distinct edges connecting them and each edge is either an \emph{undirected} ($a \tailtail b$),  a \emph{bidirected} ($a \arrowarrow b$), or one of two distinct \emph{directed} edges ($a\tailarrow b$ or $b\tailarrow a$).
An edge that connects $a$ and $b$ is often denoted by $\edgeab$ or $\edge{b}{a}$. This denotation indicates only the endpoints of the edge but not its type.
We also use the letter $e$, often with a subscript (e.g., $e_i$) to denote an edge without reference to either its type or its endpoints.
We often depict edges both in text and in figures. For instance, if $\edgeab = \seq{a,\tailarrow,b}\in E$ then we can depict the edge as either $a \tailarrow b$ or $b \arrowtail a$ and this is true regardless of whether the edge type is symmetric or asymmetric because the symbols used for symmetric edge types are laterally symmetric and the symbol for the asymmetric edge type is not laterally symmetric.
Two vertices $a$ and $b$ are \emph{adjacent} in a graph if there is an edge that connects them in the graph. We denote the set of vertices adjacent to a vertex $b$ in graph $G$ by $\adj(b,G)$.
Each endpoint of an edge in a graph $G\in \GraphFamily$ has an \emph{endmark} which is either an \emph{arrowhead} or a \emph{tail}. For instance, the edge $a \tailarrow b$ has a tail at $a$ and an arrowhead at $b$. 
An endmark of a vertex $a$ for an edge $\edgeab$ in a graph is \emph{exclusive} if all edges that connect $a$ and $b$ have the same endmark at $a$. If an exclusive endmark is an arrowhead it is called an \emph{exclusive arrowhead} and otherwise it is called an \emph{exclusive tail}.
If an edge $\edgeab$ has an arrowhead at $b$ then we say that the edge is \emph{into} $b$ and if it has a tail at $a$ we say that the edge is \emph{out of} $a$.

\subsubsection{Walks: edge sequences, vertex sequences, and section sequences}
\label{sec:walks}

In this section we define walks in graphs as edge sequences that satisfy particular constraints and define terminology and notation related to walks. We also define the vertex sequence of a walk and the section sequence of a walk.

A \emph{walk} in graph $G$ is an edge sequence in which (i) every terminal edge shares an endpoint with its adjacent edge, and  (ii) every internal edge shares one of its endpoints with one of its adjacent edges and its other endpoint with its other adjacent edge.
We use lower-case Greek letters $\omega$, $\gamma$, and $\tau$ both with and without subscripts to denote walks.
We will often depict a walk by an alternating sequence of vertices and edge types such as $a \tailarrow b \tailtail c$.

Associated with every walk $\omega$ is its unique vertex sequence $\vertexseq{\omega}$. For instance, the walk $\omega=a \tailarrow b \tailtail c$ has vertex sequence $\vertexseq{\omega}=\seq{a,b,c}$. We use $\vertexset{\omega}$ to denote the set of vertices that appear on a walk.
A \emph{chord} of a walk $\gamma$ with vertex sequence $\vertexseq{\gamma}=\seq{v_1,\ldots,v_n}$ is an edge $\edge{v_r}{v_s}$ between a pair of non-consecutive vertices $v_r$ an $v_s$. A walk is \emph{chordless} if it has no chords.
We often describe vertices relative to a walk using the following natural definitions.
A vertex $v$ is \emph{on} (or \emph{appears on}) a walk $\gamma$ if $v$ is on $\vertexseq{\gamma}$.
A vertex $v$ is called an \emph{internal vertex} on a walk $\gamma$ if $v$ is an internal vertex of $\vertexseq{\gamma}$.
A vertex $v$ is the \emph{first} vertex of a walk $\gamma$ if $v$ is the first vertex of $\vertexseq{\gamma}$.
A vertex $v$ is the \emph{last} vertex of a walk $\gamma$ if $v$ is the last vertex of $\vertexseq{\gamma}$. We use $\first(\gamma)$ and $\last(\gamma)$ to denote, respectively, the first and last vertex of walk $\gamma.$
A vertex $v$ is called an \emph{endpoint} of a walk $\gamma$ is either the first or last vertex of $\gamma$.
If the first vertex of a walk $\gamma$ is $v_1$ and the last is $v_n$ then the walk $\gamma$ is \emph{between} $v_1$ and $v_n$.

Also associated with a walk $\omega$ is a \emph{section sequence} $\defsectionomega$ such that (i) $\vertexseq{\gamma}=\sigma_1 \append \ldots \append \sigma_n$, (ii) if vertex sequence $\sigma_i$ has more than one vertex then $\omega(\sigma_i)$ is an undirected walk, (iii) no pair of adjacent sections can be combined to form a longer undirected subwalk of $\omega$; that is $\omega(\sigma_{i-1} \append \sigma_i)$ is not an undirected walk. 
A \emph{section} of a walk $\gamma$ is any vertex subsequence that appears on the section sequence of the walk $\sectionseq{\gamma}$. 
We use lower-case Greek letters $\sigma$ and $\sigma'$ both with and without subscripts to denote sections on a walk.
Note that we choose to use vertex sequences to represent sections as a section may consist of a single vertex. For instance, for the walk $\omega=a \tailarrow b \arrowarrow c$ every section is a vertex sequence of length one; $\sectionseq{\omega}=\seq{\seq{a},\seq{b},\seq{c}}$. 
A section $\sigma_i$ of a walk $\omega$ with $\defsectionomega$ is a \emph{collider section} if $1 < i < m$, the edge $\edge{\last(\sigma_{i-1})}{\first(\sigma_i)}$ on $\omega$ has an arrowhead at $\first(\sigma_i)$, and the edge $\edge{\last(\sigma_i)}{\first(\sigma_{i+1}}$ has an arrowhead at $\last(\sigma_i)$.
Any section on a walk that is not a collider section is a \emph{non-collider section}. Note that the endpoints of a walk are always on non-collider sections.

\subsubsection{Types of walks and anterior sets}\label{sec:types-of-walks}

In this section we define several types of walks, and define anterior and posterior vertex sets that are defined in term of walks in a graph $G\in \GraphFamily$.

First we define several type of walk.
A walk $\omega$ with $\defwalkomega$ is a \emph{circuit} if the first and last vertices of the walk are identical (i.e., $v_1=v_n$).\footnote{Unlike some definitions of circuit, we allow for repeated vertices and repeated edge.}
A walk is \emph{undirected} if all of the edges on the walk are undirected edges.
A walk $\gamma$ with vertex sequence $\vertexseq{\gamma}=\seq{v_1,\ldots,v_n}$
is an \emph{anterior} walk if it contains only undirected and  directed edges and if the edge $\walkedgev{i}$ on $\omega$ is directed then it is oriented $\walkedgerightarrow{i}$.
A walk $\omega$ with vertex sequence $\defwalkomega$
is a \emph{directed} walk if it is an anterior walk containing only directed edges.
A \emph{semi-directed} walk is an anterior walk with at least one directed edge.
If $\omega$ with $\defwalkomega$ is either a directed walk, anterior walk, or a semi-directed walk then the walk is \emph{from} $v_1$ \emph{to} $v_n$.

Next we define the concepts of anterior sets in a graph.
The vertices \emph{anterior} to a set of vertices $A$ in graph $G$ is the set $\ant(A,G)$ containing $A$ and all vertices $j\in V$ such that there is a anterior walk from $j$ to some vertex $a\in A$. A set of vertices $A$ is an \emph{anterior set} in a graph $G$ if $A=\ant(A,G).$
A set of vertices $A$ is \emph{anterior} to a set of vertices $B$ in graph $G$ if $A\subseteq \ant(B,G)$.

\newcommand{\vsc}{vertex separation criterion}

\subsubsection{Connecting walks and vertex separation}\label{sec:separation}

A \vsc\ captures the separation properties of a graph and one typically defines a \vsc\ in terms of connecting walks in the graph.
In this section, we extends previous definitions of d-connecting walks to mixed graphs in order to define a \vsc\ for mixed graphs.

A walk $\omega$ between distinct vertices $a$ and $b$ is a \emph{connecting walk} given $C\subseteq V$ in graph $G\in \GraphFamily$ if (1) every collider section of $\omega$ has a vertex in $C$, and (2) no non-collider section has a vertex in $C$. Note that $C$ cannot contain either endpoint of the walk if the walk is connecting given $C$.
If there is a connecting walk between $a$ and $b$ given $C$ in graph $G$ then $\indepabC$ holds in $G$.
If there is no connecting walk between $a$ and $b$ given $C$ in graph $G$ then $\depabC$ holds in $G$. If $\depabC$ holds in $G$ then we write $\depabCG$ and, otherwise $\indepabCG$.
Two vertices $a$ and $b$ in graph $G$ are \emph{separated} if there is a set $C$ such that $\indepabCG$.

\subsubsection{Independence models of mixed graphs} Vertex separation in a graph is then used to define an independence model over the vertices. In particular,
the \emph{graph \indepmodel\ for graph} $\defgraphG$ is 
$$\imodel{G}=\seq{V,\setcomprehensionLSE{\testcaseABC}{\testsForV}{\forall a\in A \ \forall b\in B \ \indepabCG}}.$$

The independence model of a mixed graph satisfies additional independence properties.  The following proposition was proved in \cite{SadeghiLauritzen2011}.

\repproposition{prop:compgraph}{
    If $G$ is a mixed graph then 
    the compositional graphoid properties hold in $\imodel{G}$.}

In addition, the weak transitivity independence property was shown to hold in mixed graph in \cite{Sadeghi2017}. 

\repproposition{prop:weaktrans}{
    If $G$ is a mixed graph then weak transitivity holds in  $\imodel{G}$.}

\subsubsection{Separation equivalence and separation equivalence classes} Separation equivalence and separation equivalence classes play an important role in characterizing graphical models and in identifying their structure.

The \emph{separation equivalence} of two graphs is defined in terms of their \indepmodel s as $G \equiv G'$ if $\imodel{G}=\imodel{G'}$.\footnote{Separation equivalence is also called Markov equivalence.}
Note that the independence model of a graph is decomposable, thus, Proposition~\ref{prop:sep-equiv::pairwise-sep-equiv} implies that one only needs to consider walk in order to establish the separation equivalence of two graphs. In particular, two graphs $G$ and $H$ are separation equivalent if and only if (1) if there exists a connecting walk given $C$ in $G$ then there is a connecting walk given $C$ in $H$, and (2) if there exists a connecting walk given $C$ in $H$ then there is a connecting walk given $C$ in $G$.

The separation equivalence class of a graph is the set of graphs that are separation equivalent to the graph.
The equivalence class of a graph $G$ with respect to $\equiv$ relative to a graph family $\GraphFamilyF$ is the set of graphs 
$\equivclassF{G}=\setcomprehensionLSE{G'}{\GraphFamilyF}{G\equiv G'}$.  Note that we do not explicitly mention the equivalence relation $\equiv$ when  denoting the separation equivalence class of a graph because the only equivalence classes of graphs that we consider are separation equivalence classes. When no family is mentioned the graph family is take to be the set of all mixed graphs $\GraphFamily$, that is, $\equivclass{G}=\equivclassG{G}$.

\subsection{Graph properties and graph separation properties}
\label{sec:separational}

A \emph{graph property} is a Boolean function of one or more graphs.

A graph property $f(G_1,\ldots,G_k)$ is a \emph{graph separation property} if it
can be expressed as a Boolean function of the independence models of the graphs;
that is, for some Boolean function $h$,

\[
f(G_1,\ldots,G_k)
  = h\bigl(I(G_1),\ldots,I(G_k)\bigr).
\]

Such a property depends on each graph only through its induced independence model.

Throughout this paper we distinguish between \emph{graphical}
characterizations, which are defined in terms of arbitrary graph properties, and \emph{separational} characterizations, which are defined in terms of graph separation properties.  A key advantage of separational characterizations is that they depend only on the independence models of the graphs and are therefore invariant under separation equivalence.

\subsection{Graphical models, separational and statistical equivalence}

A \emph{graphical model} is defined with respect to a \emph{defining graph structure} $G$ and a \emph{defining set of distributions} $\DistFamily$.
A \emph{graphical model} is a tuple $\statmodelwithDistFamily{G}=\seq{G,\markov{G}}$ such that $G\in \GraphFamily$ and $\markov{G}=\setcomprehensionLSE{P}{\DistFamily}{\imodel{G} \subseteq \imodel{P}}$.
If $\imodel{G} \subseteq \imodel{P}$ we say that $P$ is \emph{Markov} with respect to $G$.\footnote{The relationship $\imodel{G} \subset \imodel{P}$ --- all of the separation facts in the graph correspond to independence facts in the distribution --- is sometimes referred to as $P$ satisfying the global Markov condition with respect to graph $G$. Other concepts relating independence in a distribution and separation in a graph exist. See \cite{Lauritzen1996} for a discussion of other such relationships between graphs and probability distributions.} Thus $\markov{G}$ represents the set of distributions in $\DistFamily$ that are Markov with respect to $G$.
We use $\statmodel{G}$ to denote $\statmodelwithDistFamily{G}$ unless the choice of the $\DistFamily$ is relevant.

It is useful to compare graphical models in terms of the independence models of their defining graphs, and the set of distributions that are Markov with respect to their defining graph. 
Two graphical models $\statmodel{G}, \statmodel{H}$ are \emph{separation equivalent} if $G\equiv H$.
Two graphical models $\statmodel{G}, \statmodel{H}$ are \emph{statistically equivalent} if $\markov{G}=\markov{H}$.
The following proposition relates the separation equivalence and statistical equivalence of graphical models. It follows from the fact that the set of distributions that are Markov with respect to a graph is defined in terms of the independence model of the defining graph.

\repproposition{prop:sep-equiv:stat-equiv}{If two graph are separation equivalent (i.e. $F\equiv H$) then $\statmodelwithDistFamily{F}$ is statistically equivalent to $\statmodelwithDistFamily{H}$.}

Statistical equivalence can fail to imply separation equivalence if the family of defining distributions is not rich enough to distinguish between graphs that are not separation equivalent.

\subsection{Graph families and essential graph families}
\label{sec:graph-families}
A \emph{graph family} $\GraphFamilyF$ is any subset of the family of mixed graphs $\GraphFamily$ (i.e., $\GraphFamilyF\subseteq \GraphFamily$).
In this section we distinguish several types of mixed graphs and notation for the set of graphs of each type.
A graph is \emph{simple} if there is at most one edge between any pair of vertices. The family of simple graphs is $\SimpleFamily$.
A graph is \emph{acyclic} if it contains no semi-directed circuit and is \emph{cyclic} otherwise. The family of acyclic graphs is $\AcycFamily$.
A graph $G$ is \emph{maximal} if adding any edge between any two non-adjacent vertices in $G$ changes its independence
model. The family of maximal graphs is $\MaxFamily$.

We define essential graph families using the equivalence closure of a generating graph family with respect to separation equivalence. 
The equivalence closure of a set $\GraphFamilyF \subseteq \GraphFamily$ with respect to equivalence relation $\sim$ on $\GraphFamily$ is the set $\operatorname{cl}_\sim(\GraphFamilyF)=\setcomprehensionLSE{F}{\GraphFamilyF}{\exists G\in \GraphFamily, G \sim F}.$
The \emph{essential graph family} generated by $\GraphFamilyF$
is the set $\EssentialFamilyF=\operatorname{cl}_\equiv(\GraphFamilyF).$
Equivalently, this is the union of all separation--equivalence classes that intersect $\mathcal{F}$.

\subsection{Comparing the expressivity graph families}
We can compare graphical models in terms of their graphical expressivity and their separational expressivity.

We compare the \emph{graphical expressivity} of two graphical model families by comparing the graphs that they contain. For instance, if $\GraphFamilyF \subset \GraphFamilyH$ then $\GraphFamilyH$ is \emph{more graphically expressive} than $\gmfF$. Thus, the family of chain graphs is more graphically expressive than both the family of undirected and directed acyclic graphs.

We compare the \emph{separational expressivity} of two graph families using the sets of graphs contained in the essential families generated by the two families. 
If $\EssentialFamilyF \subset \EssentialFamily{\GraphFamilyH}$ then we say that $\GraphFamilyH$ is \emph{more separationally expressive} than $G$. We denote this relationship by $\GraphFamilyF \sqsubset \GraphFamilyH$. If $\EssentialFamilyF = \EssentialFamily{\GraphFamilyH}$ then the two families are \emph{essentially equivalent} and denote the relationship by $\GraphFamilyF \equiv \GraphFamilyH$.
Because essential families are defined as the separation-equivalent closure, such set-theoretic comparison compare the families of induced independence models associated with the graph family.
The implications of essential equivalence for graphical models is considered in the next section.

\subsection{Graphical model families}

A \emph{\gimf} is defined in terms of a \emph{defining family of graphs} $\GraphFamilyF \subseteq \GraphFamily$ and a defining family of distributions $\DistFamily$.
In particular, the \emph{graphical model family} defined by $\DistFamily$ and $\GraphFamilyF$ is the set of graphical models $\ModelFamily^\DistFamily_\GraphFamilyF = \bigcup_{G\in\GraphFamilyF}\statmodelwithDistFamily{G}$.
For instance the family of Gaussian anterial graphical models is the set of models $\ModelFamily_{\AnterialFamily}^\DistFamily$ if $\DistFamily$ is the set of all Gaussian probability distributions over the set of variables $V$. Other types of graphical models are defined by using other defining families of distributions and graphs. We will typically assume that different \gimfs\ share the same defining family of distributions $\DistFamily$ and use, for instance, $\gmfF$ to denote $\ModelFamily^\DistFamily_\GraphFamilyF$.

We compare the \emph{statistical expressivity} of two graphical model families using the set of models that they define. We use the relations $\equiv$ and $\sqsubseteq$ to compare statistical expressivity. These relations are defined as follows

\newcommand{\ModelsFromGraphicalModelFamily}[1]{\DefineConstant{Models}(#1)}

\begin{eqnarray*}
\gmfF \equiv \gmfH &\quad \mathrm{if}\quad& \ModelsFromGraphicalModelFamily{\gmfF}=\ModelsFromGraphicalModelFamily{\gmfH} \\
\gmfF \sqsubset \gmfH &\quad \mathrm{if}\quad& \ModelsFromGraphicalModelFamily{\gmfF} \subset \ModelsFromGraphicalModelFamily{\gmfH} \\
\end{eqnarray*}
where $\ModelsFromGraphicalModelFamily{\gmfH}=\set{\markov{H} \ | \ \seq{H, \markov{H}} \in \gmfH}$.
If $\gmfF \equiv \gmfH$ then the two graphical model families are \emph{modeling equivalent}. If $\gmfF \sqsubset \gmfH$ then $\gmfH$ is \emph{more statistically expressive} than $\gmfF$.

The following proposition captures the implication that the essential equivalence of two graph families on the statistical expressivity of the graphical model families that they define.
The proposition follows from the fact that essentially equivalent graph families have the same graphs modulo separation equivalence and Proposition~\ref{prop:sep-equiv:stat-equiv}.

\begin{proposition}
    If two graph families are essentially equivalent (i.e, $\GraphFamilyF \equiv \GraphFamilyH$) then for any defining family $\DistFamily$, the graphical model families that they define are modeling equivalent (i.e., $\ModelFamily^\DistFamily_\GraphFamilyF \equiv \ModelFamily^\DistFamily_\GraphFamilyH$).
\end{proposition}

\subsection{\Sia s and identification}\label{sec:structure-identification}

The goal of a \sia\ is to identify graphical properties of the structure of a \emph{generating graph} $\genstruct$ given an \emph{observed} distribution $\gendist$.\footnote{
It is also natural to consider a structure identification algorithm that takes data sampled from a generative distribution as input. Assuming that we observe the generative distribution allows us to focus on identification rather than other notions of correctness like asymptotic consistency or rates of identification.}
In this paper we focus on \indeptest\ based \sia s in which one uses only the results of statistical independence tests to identify the structure of $\genstruct$.
This type of \sia\ is often called a constraint-based learning algorithm.
Furthermore, we choose to represent the identified structure of $\genstruct$ with a graph.
Thus, our focus is on \gsia s, where a \emph{\gsia} is a function $l\in \function{\AllIModels}{\GraphFamily}$ that maps the \indepmodel\ $\imodel{\gendist}$  to a graph $G\in \GraphFamily$.

We begin by noting that two separation equivalent graphs cannot be distinguished by a \gsia. This is due to the fact that a \gsia\ is invariant under separation equivalence. In other words, the two separation equivalent graphs have equivalent independence models (i.e., $\imodel{G}=\imodel{H}$) which implies that any \gsia\ applied to these \indepmodel s will result in identical graphs (i.e., $l(\imodel{G})=l(\imodel{H})$). 
Thus separation equivalent graphs are indistinguishable based on statistical independence tests. Thus, a \gsia\ can only hope to distinguish separation equivalence classes of graphical models.

In order for a \gsia\ to correctly identify the separation equivalence class of a generating graph $\genstruct$ we need to make identifying assumptions connecting $\imodel{\gendist}$ and $\imodel{\genstruct}$.
A typical assumption between $\gendist$ and $\genstruct$ is the \emph{Markov assumption}. A distribution $\gendist$ is \emph{Markov} with respect to $\genstruct$ (i.e., $P \in \markov{G}$).
The Markov assumption, however, is not sufficient to guarantee the identification of the equivalence class of the generating graph.
The assumption only guarantees that $\imodel{\genstruct} \subseteq \imodel{\gendist}$.
If, for instance, there is 
a product distribution\footnote{A distribution is a product distribution if $P(V)=\prod_{v_i\in V}P(v_i)$ where $P(v_i)$ is the marginal distribution of $v_i$.}  $P\in \DistFamily$ then $P\in \markov{G}$ for all $G\in \GraphFamilyF$. Thus, if the generative distribution $\gendist$ is a product distribution, all graphical structures are indistinguishable if one only assume that $\gendist$ is Markov with respect to $\genstruct$. 

Thus, in order to identify the equivalence class of the generating graph $\genstruct$, one must make an additional compatibility assumption between $\gendist$ and $\genstruct$ generally, or between $\imodel{\gendist}$ and $\imodel{\genstruct}$ in the case of an \indeptest\ based \sia.
Such an assumption must ensure that at least some additional dependence facts hold in $\gendist$.
The typical assumption (e.g., \cite{sgs}) is that $\gendist$ is Markov and faithful with respect to $\genstruct$ or equivalently that $\gendist$ is perfect with respect to $\genstruct$. A probability distribution $\gendist$ is \emph{perfect} with respect to $\genstruct$ if $\imodel{\gendist}=\imodel{\genstruct}$.
A weaker assumption that is appropriate for a \gsia\ is \emph{perfect testing},  where one assumes that, for any $\testcaseABC$ considered during the execution of the algorithm, the result of statistical independence tests on $\gendist$ matches the result of a separation test on $\genstruct$ (i.e., $\indep{A}{B}{C}{\imodel{\gendist}}$ if and only if  $\indep{A}{B}{C}{\imodel{\gendist}}$). In this case we say $\gendist$ \emph{provides perfect testing} for $\genstruct$ with respect to a \gsia.
The sufficiency of the perfect testing assumption for identifying the equivalence class of $\genstruct$ for any graph family $\GraphFamilyF$ is due to the fact that $G\equiv H$ is defined as $\imodel{G}=\imodel{H}$.  Thus, if a \gsia\ uses enough statistical tests it can distinguish graphs that are not separation equivalent.

Finally, we define when a \gsia\ identifies the separation equivalence class of a graph family given a probability distribution under some compatibility condition.
A \emph{compatibility condition} is a Boolean function of a distribution and a graph; a function in $\booleanfunction{\carteseanBinary{\DistFamily}{\GraphFamily}}$.
A \gsia\ $l\in \function{\AllIModels}{\GraphFamily}$ \emph{identifies the separation equivalence class} of a graph family $\GraphFamilyF$ given $\gendist \in \DistFamily$ under compatibility condition $C$ if
for all $\genstruct \in \GraphFamilyF$ and for all $\gendist \in \DistFamily$ it is the case that if $C(\gendist, \genstruct)$ then  $l(\imodel{\gendist}) \equiv \genstruct$.

\section{Separable graphs}
\label{sec:separable-graphs}

A graph $G$ is \emph{separable} if every pair of non-adjacent vertices is separable.
The family of separable graphs is denoted by $\SepFamily.$ In this section, we provide a characterization of separable graphs and of separable vertices in mixed graphs.

The existence of separating sets is closely connected to the adjacency properties of a mixed graph. Recall that $C$ is a separating set for two vertices $i$ and $j$ in graph $G$ if $\indep{i}{j}{C}{G}$ and that that two vertices $i$ and $j$ are \emph{separable} in graph $G$ if there is a separating set $C$.
The following propositions capture the fact that the existence of a separating set for two vertices implies that the two vertices are not adjacent. 
This proposition follows from the fact that two vertices that are adjacent cannot be separated by any separating set.

\begin{proposition}\label{prop:separable:not-adjacent}
    If $i$ and $j$ are separable in $G$ then $i\not\in \adj(j,G)$.
\end{proposition}

The lack of a separating set for a pair of two vertices, however, does not imply the existence of an edge between the two vertices. Consider the pair of vertices $\seq{b,c}$ in Graph $G_2$ in Figure~\ref{fig:separable-and-weakly-separable-graphs}. These vertices are not separable and thus Graph $G_2$ is not separable. 
One can verify that there is no separating set for this pair of vertices by exhaustively considering the possible separating sets. The following theorem characterizes when two vertices are separable in a graph and provides a more direct approach to verifying that two vertices are separable.

\reptheorem{thm:vertex-separable:pairwise-anterial-separable}{{\rm \!\! [Pairwise anterior separation] \  }
    Two vertices $i$ and $j$ are separable in graph $G$ if and only if $\indep{i}{j}{\propantij}{G}$.
    }

This characterization of vertex separability generalizes the characterization of \cite{LauritzenSadeghi2018} from chain mixed graphs to mixed graphs. Again considering vertices $\seq{b,c}$ in Graph $G_2$, we can see that $\dep{b}{c}{ade}{G_2}$ which implies, by Theorem~\ref{thm:vertex-separable:pairwise-anterial-separable}, that $b$ and $c$ are not separable.

In order to characterize separable graphs we consider two types of walks: collider walks and \iw s. 
A \emph{collider walk} is a walk in which all internal vertices on the walk are on collider sections of the walk. Note that this implies that the endpoints of a collider walk are the only vertices on non-collider sections of the walk. A collider walk is \emph{trivial} if it contains no collider sections and it is \emph{non-trivial} otherwise. Any walk consisting of a single edge is a trivial collider walk. 
A collider walk in a graph is an \emph{\iw} if there is no edge in the graph between the endpoints of the walk. Note that an inducing walk must have at least one internal vertex and thus a walk consisting of a single edge is not an inducing walk because it has no internal vertices.
Inducing walks are important with respect to separability due to the fact that they can induce dependence between vertices that are not adjacent. They also play an important roll in other results in the paper including one of our characterizations of separation equivalence.

The next theorem characterizes separable graphs in terms of whether the graph contains a \siw.
A \emph{\siw} is an inducing walk such that every collider section is anterior to one of the endpoints of the inducing walk.\footnote{A \siw\ is closely related to both the primitive inducing path of \cite{RichardsonSpirtes2002} and the primitive inducing walk of \cite{LauritzenSadeghi2018}. In fact, if either a primitive inducing path or primitive inducing walk has no edge between its endpoints then it is a \siw.}

\reptheorem{thm:separable-graph::no-self-inducing-walks}{A graph is separable if and only if it contains no self-inducing walks.}

This characterization generalizes the results of  \cite{LauritzenSadeghi2018} and \cite{RichardsonSpirtes2002} to the family of mixed graphs.

The graph $G_1$ in Figure~\ref{fig:separable-and-weakly-separable-graphs} is an example of a  separable graph. Separability can be verified by the separation facts $\indep{b}{e}{d}{G_1}$, $\indep{a}{b}{\emptyset}{G_1}$, $\indep{a}{d}{\emptyset}{G_1}$ and $\indep{a}{e}{\emptyset}{G_1}$.

\begin{figure}[htbp]
  \centering
  \includegraphics[scale=0.6]{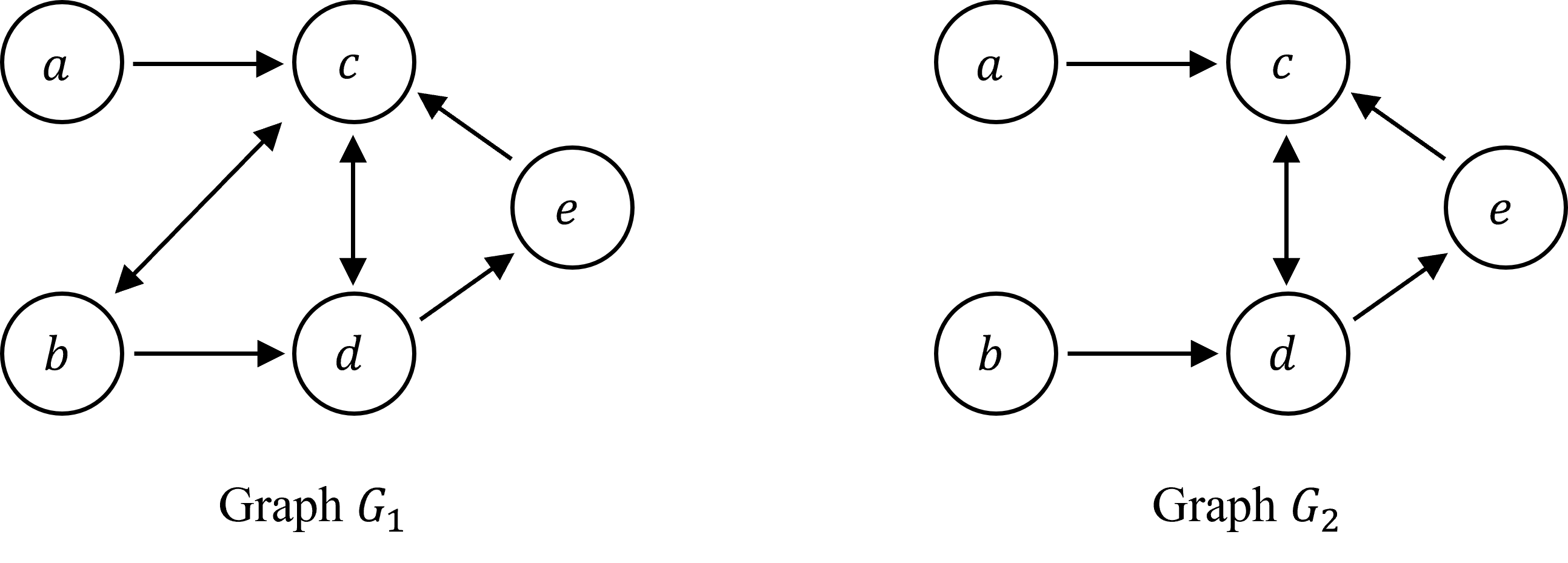}
  \caption{Separable graph $G_1$ and essentially separable graph $G_2$ that is not separable.}
  \label{fig:separable-and-weakly-separable-graphs}
\end{figure}

\section{Essentially separable graphs}
\label{sec:essentially-separable}
A graph $G$ is \emph{essentially separable} if there is a graph $G'\equiv G$ such that $G'$ is separable.
Thus, the family of essentially separable graphs is the essential graph family $\WSepFamily$ generated by separable graphs.
When considering graphical models defined by graphs, if a graphical model is defined in terms of a essentially separable graph then there is a statistically equivalent graphical model that is defined in terms of a separable graph.

Graph $G_2$ in Figure~\ref{fig:separable-and-weakly-separable-graphs} is not separable.
This follows from Theorem~\ref{thm:separable-graph::no-self-inducing-walks} and the fact that the walk $b \tailarrow d \arrowarrow c$ is a \siw.
The graph, however, is essentially separable. In particular, if edge $b \arrowarrow c$ is added to this graph then one obtains graph $G_1$, an equivalent separable graph.
The equivalence can be verified by enumerating the separation facts for the two graphs.

For acyclic graphs, a graph is separable if and only if it is maximal (see, Corollary~2 of \cite{LauritzenSadeghi2018}). This correspondence, however, fails for mixed graphs as illustrated by Graph $G_3$ of Figure~\ref{fig:non-weakly-separable-graph}.\TextVersion{}{\footnote{In Section~\ref{sec:sigma-separation}, 
we show that for mixed graphs under under $\sigma$-separation, $\sigma$-maximal graphs correspond to $\sigma$-separable graphs.}}
\TextVersion{
This graph is an example of a cyclic graph that is maximal but not essentially separable. 
Graph $G_3$ is not separable due to the fact that $b \tailarrow d \arrowtail c$ is a \siw\ and Theorem~\ref{thm:separable-graph::no-self-inducing-walks}. Note $a\tailarrow e \arrowtail d$ is also a \siw.
To see that the graph $G_3$ has no equivalent separable graph we start by noting that two separation facts are true in graph $G_3$: (1) $\indep{a}{b}{\emptyset}{G_3}$ and (2) $\indep{a}{b}{cd}{G_3}$.
Suppose there is a separable graph $G'_3$ that is equivalent to $G_3$. 
Graph $G'_3$ must have an edge $\edge{b}{c}$ and an edge $\edge{a}{c}$.
If $\edge{b}{c}$ has a tail at $c$ or $\edge{a}{c}$ has tail $c$ in $G'_3$ then fact (1) would not be true. However, if both edges have arrowheads at $c$ then fact (2) would not hold. Thus, no such separable graph $G'_3$ equivalent to $G_3$ exists.

\begin{figure}[htbp]
  \centering
  \includegraphics[scale=0.6]{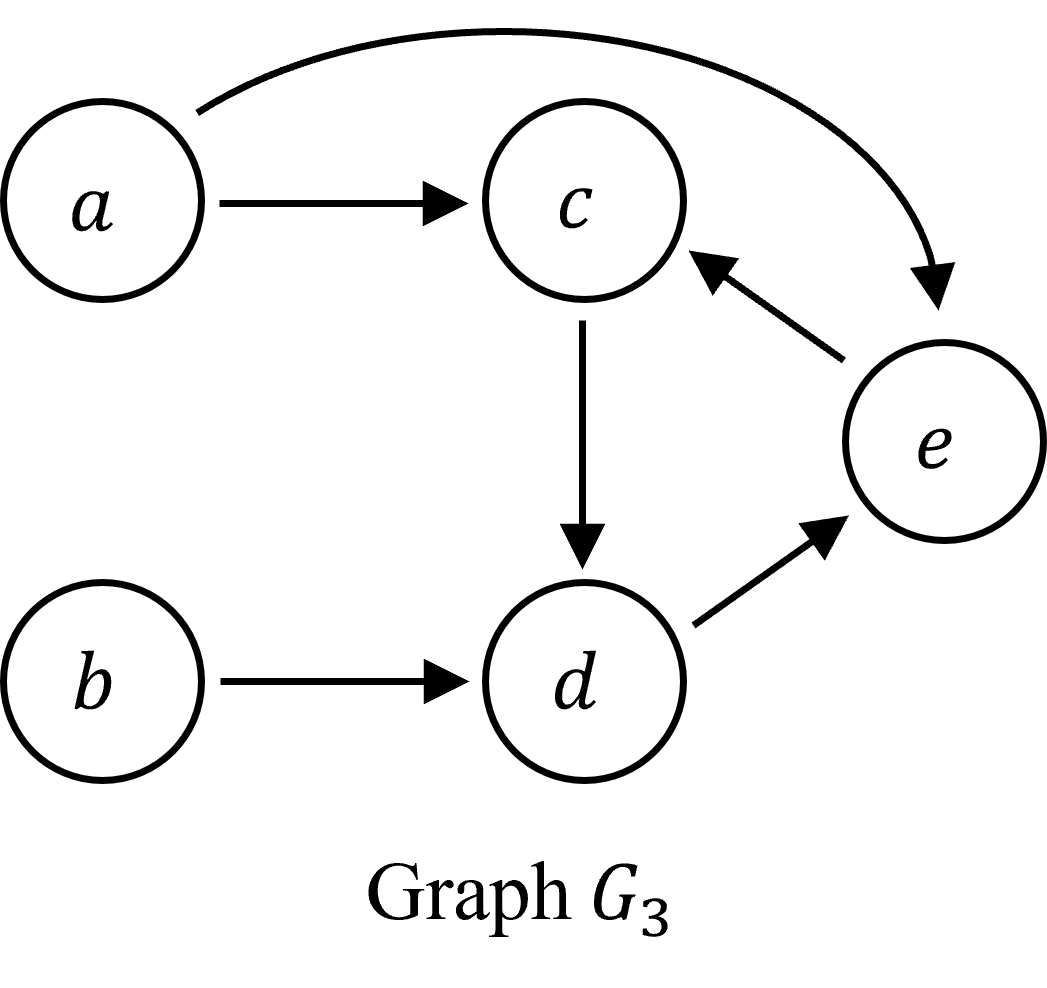}
  \caption{A graph that is not essentially separable and maximal.}
  \label{fig:non-weakly-separable-graph}
\end{figure}

Next we show that graph $G_3$ is maximal. No edge $\edge{a}{b}$ can be added due to the fact that $a$ and $b$ are separable due to (e.g.) separation fact (1).
We have shown above that no edge $\edge{b}{c}$ can be added so the only other edges that could be added is the edge $\edge{a}{d}$.
Consider adding an edge $\edge{a}{d}$ to the graph. 
If $\edge{a}{d}$ has a tail at $d$ or $\edge{b}{d}$ has a tail at $d$ in $G'_3$ then fact (1) would not be true. However, if both edges have arrowheads at $d$ then fact (2) would not hold. Thus, no such separable graph $G'_3$ equivalent to $G_3$ exists. Thus, no such edge can be added to the graph and $G_3$ is maximal.
}
{


\begin{figure}[htbp]
  \centering
  \includegraphics[scale=0.6]{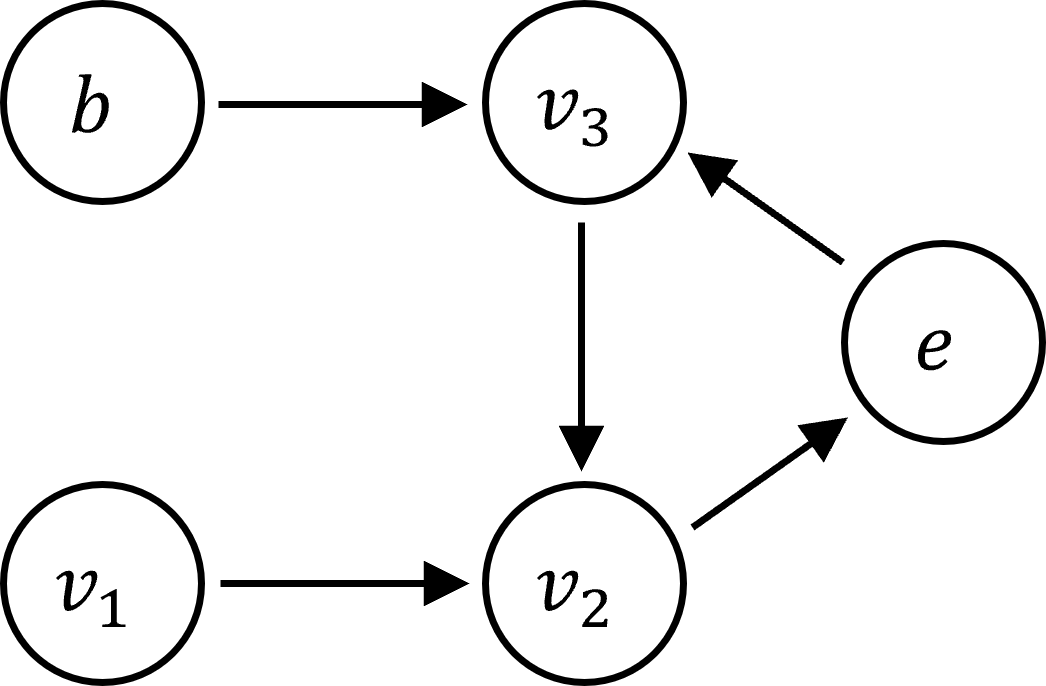}
  \caption{A graph that is not essentially separable and maximal.}
  \label{fig:alt-non-weakly-separable-graph}
\end{figure}

Next we show that the graph in Figure~\ref{fig:alt-non-weakly-separable-graph} is maximal. As shown above, the graph is not separable or even weakly-separable. 
The only edges that can possibly be added are $\edge{v_1}{v_3}$ and  $\edge{b}{e}$ as all other pairs of vertices are separable in the graph. We showed above that no edge $\edge{v_1}{v_3}$ can be added. Next we consider adding an edge $\edge{b}{e}$ to the graph. Two additional separation facts that hold in the graph are fact (3), $\indep{b}{v_2}{v_3}{G}$ and fact (4), $\indep{b}{v_2}{v_2e}{G}$. If $\edge{b}{e}$ with a tail at $e$ is added then fact (3) would not be true. If $\edge{b}{e}$ with an arrowhead at $e$ is added then fact (4) would not be true. Thus, no such edge can be added to the graph and $G$ is maximal.
}

In the remainder of the paper we provide results about separable and essentially separable families of graphs and the graphical models defined by these graph families.

\section{Graphical characterization of essentially separable graphs}
\label{sec:essentially-acyclic}

In this section we provide a graphical characterization of essentially separable graphs. Furthermore, we show that any graphical models defined in terms of an essentially separable model has an equivalent graphical model defined in terms of a simple acyclic graphs called an anterial graph.

A graph $G$ is \emph{essentially acyclic} if there is an acyclic graph $G'$ that is separation equivalent to $G$. 
Thus the family of essentially acyclic is the essential graph family $\EssentialFamily{\AcycFamily}$ generated by acyclic graphs.
A graph that is not essentially acyclic is \emph{essentially cyclic}.
Consider graphs $G_4$ and $G_5$ from Figure~\ref{fig:essentially-acylic-graphs}. The two graphs are separation equivalent as 
$a$ is separated from $d$ given $bc$ in both graphs, $b$ is separated from $c$ given $ad$ in both graphs, and these are the only separation facts that hold in these graphs. 

\begin{figure}[htbp]
  \centering
  \includegraphics[scale=0.6]{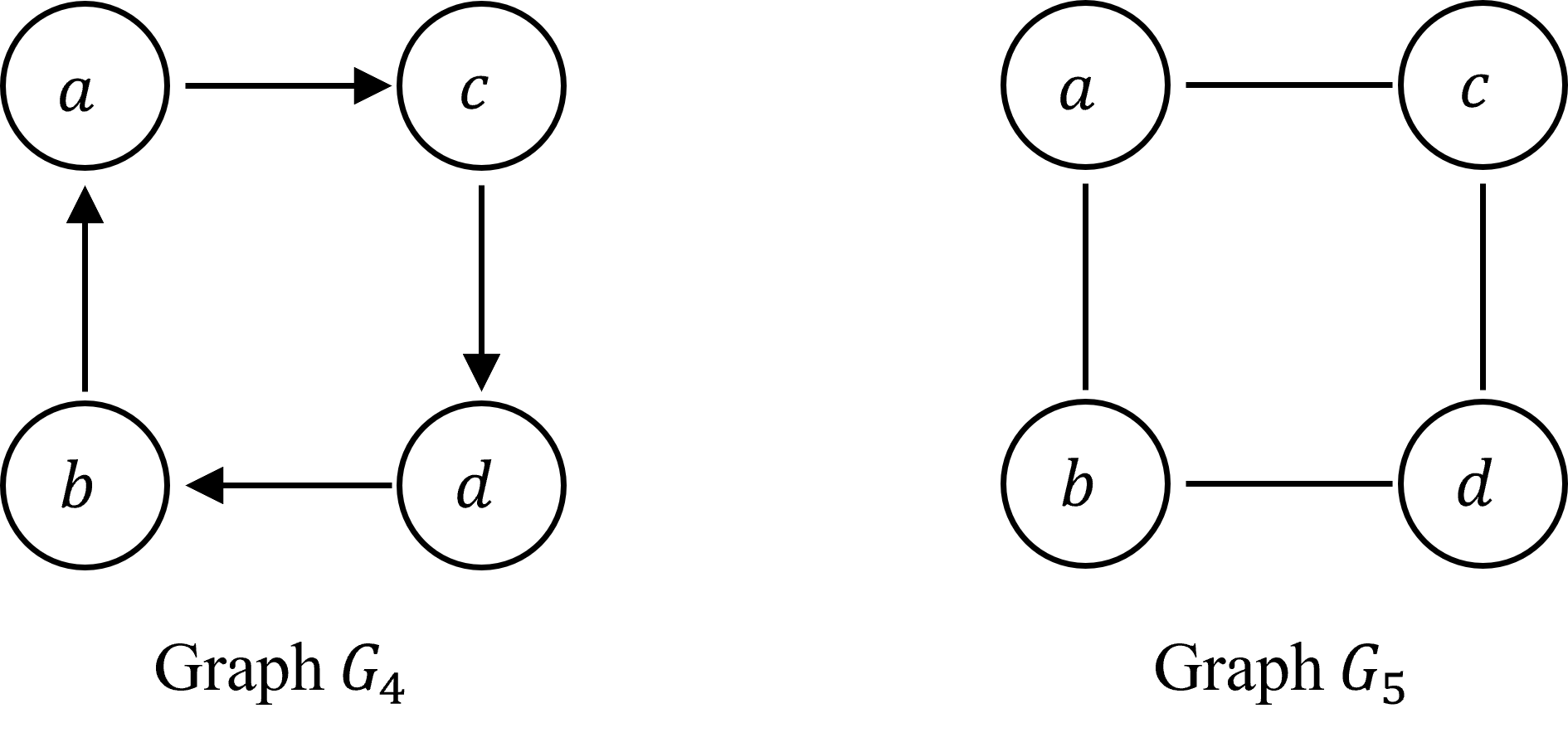}
  \caption{A cyclic graph $G_4$ and an equivalent acyclic graph $G_5$.}
  \label{fig:essentially-acylic-graphs}
\end{figure}

The next theorem provides a characterization of essentially separable graphs as those graphs that are essentially acyclic.

\reptheorem{thm:weakly-separable::essentially-acyclic}{
A graph is essentially separable if and only if the graph is essentially acyclic.}

The graph $G_3$ from Figure~\ref{fig:non-weakly-separable-graph} is essentially cyclic. This follows from Theorem~\ref{thm:weakly-separable::essentially-acyclic} and the fact that $G_3$ is not essentially separable.

The following corollary of Theorem~\ref{thm:weakly-separable::essentially-acyclic} shows that, separable graphical models $\gmfWSep$ and acyclic graphical models $\gmfAcyc$ are modeling equivalent.

\repcorollary{cor:acyclic-power}{$\gmfWSep \equiv \gmfAcyc$.}

Next we define $\simplify$ algorithm that converts a graph to an anterial graph.
An \emph{anterial graph} is a simple mixed graph $G$ such that (1) $G$ has no semi-directed circuits, and (2) if $i\arrowarrow j$ then $i \not \in \ant(j,G)$ and $j \not \in \ant(i,G)$. We denote the family of anterial graphs by $\AnterialFamily \subseteq \SimpleFamily \cap \AcycFamily$. 
The graph $G_1$ in Figure~\ref{fig:separable-and-weakly-separable-graphs} is an example of a  separable graph that is not an anterial graph but is simple and acyclic. The graph $G_1$ is not anterial due to the existence of the edge $c \arrowarrow d$ such that $d\in \ant(c,G)$.

The $\simplify$ algorithm is shown in Algorithm~\ref{alg:simplify}.
We denote the graph obtained by applying Algorithm~\ref{alg:simplify} to graph $G$ by $\simplify(G)$.
The graph $\simplify(G)$ is a simple graph with the same adjacencies as $G$ such that an edge $\edge{i}{j}$ in  $\simplify(G)$ has a tail at $i$ if and only if there is an anterior walk from $i$ to $j$ in $G$.
Consider applying the $\simplify$ algorithm to graph $G_4$ of Figure~\ref{fig:essentially-acylic-graphs}. Consider the directed edge $a \tailarrow c$. Due the fact that there is an anterior walk from $c$ to $a$ (i.e., $c\in \ant(a,G)$) and an anterior walk from $a$ to $c$ (the edge it self), the edge $\edge{a}{c}$ in $\simplify(G_4)$ must be $a \tailtail c$. By a similar argument, all of the edges must be undirected and the result $\simplify(G_4)$ is the graph $G_5$ also from Figure~\ref{fig:essentially-acylic-graphs}.

\begin{minipage}{.8\linewidth}
\begin{algorithm}[H]
\caption{The \simplify\ algorithm}\label{alg:simplify}
\begin{algorithmic}
\vspace{0.03in}
\State {\bf Input:} A graph $G=\seq{V,E}$.
\State {\bf Output:} An anterial graph.
\vspace{-0.05in}
\\ \!\!\!\!\!\!\!\! \hrulefill
\vspace{0.05in}
\State $E'=\emptyset$

\For {each edge $\edge{i}{j} \in E$}
    \If {$i \in \ant(j,G)$ and $j \in \ant(i,G)$}
        \State {$E'=E'\cup\set{i \tailtail j}$}
      \ElsIf {$i\in\ant(j,G)$}
            \State {$E'=E'\cup \set{i \tailarrow j}$}
      \ElsIf {$j\in\ant(i,G)$}
        \State {$E'=E'\cup \set{i \arrowtail j}$}
      \Else
          \State {$E'=E'\cup \set{i \arrowarrow j}$}
      \EndIf
\EndFor
\State \Return $\seq{V,E'}$

\end{algorithmic}
\end{algorithm}
\end{minipage}
\vspace{0.2in}

The following theorem proves that that every separable graph has an equivalent anterial graph and that such a graph can be obtained using the \simplify\ algorithm. 

\reptheorem{thm:simplify:separable-graph:equivalent-anterial-graph}{
For every graph $G$, the graph $\simplify(G)$ is an anterial graph; if $G$ is a separable graph then $\simplify(G)$ is a separable graph and $G \equiv \simplify(G)$.}

Theorem~\ref{thm:simplify:separable-graph:equivalent-anterial-graph} has significant implications about the statistical expressivity of graphical models defined using the family of separable anterial graphs $\SepAnterialFamily$ that is captured in the following corollary.

\repcorollary{cor:anterial-power}{$\gmfWSep \equiv \gmfSepAnterial$.}

Corollary~\ref{cor:anterial-power} shows that the family of essentially separable graphical models is modeling equivalent to the family of separable anterial graphical models.
This is notable as anterial graphs are simple acyclic graphs which, from a computational perspective, allows structure identification algorithm that seek to identify the equivalence class of an essentially separable graph to focus on simple graphs rather than multigraphs.

\section{Separational characterizations of essential graph families}
\label{sec:inducing-vertices}

In this section we give separational characterizations of three essential graph families (defined in Section~\ref{sec:graph-families}): essentially separable graphs (Section~\ref{sec:essentially-separable}), essentially acyclic graphs (Section~\ref{sec:essentially-acyclic}), and \emph{essentially undirected graphs}, the essential family $\EssentialFamily{\UndirectedFamily}$ generated by the family of undirected graphs.

We use separational characterizations because they are invariant under separation equivalence (see Section~\ref{sec:separational}).  This ensures that each characterization applies uniformly to all graphs in an essential family and clarifies the separation structure shared by the graphs in these families.
Many familiar graph properties are not invariant under separation equivalence.
For example, in Section~\ref{sec:essentially-separable} we observed that adjacency is not invariant, and in Section~\ref{sec:essentially-acyclic} we observed that acyclicity and undirectedness are also not invariant.  Indeed, each of the essential graph families considered here is generated from a graph family defined by properties that fail to be invariant under separation equivalence.

We begin by characterizing essentially undirected graphs.
A fundamental distinction between undirected graphs and, for instance, directed graphs, is that adding a vertex to a separating set for two vertices in a directed graph can induce a dependence between the two vertices.
The simplest example of this phenomena is the graph $i \tailarrow k \arrowtail j$ with three vertices and two edges in which $\indep{i}{j}{\emptyset}{G}$ and $\dep{i}{j}{k}{G}$. In this example, the vertex $k$ induces dependence between $i$ and $j$ given the empty set. We call vertex $k$ an \emph{inducing vertex}.

It is useful to generalize the concept of an inducing vertex by defining the concept of an inducing object in an independence model and by considering sets of objects that induce dependence.
An object $d\in \objectDomain$ is an \emph{inducing object for} $\testcaseABC\in \testsForDomain$ in independence model $I$ for $\objectDomain$ if both $\indep{A}{B}{C}{I}$ and $\dep{A}{B}{Cd}{I}$. The inducing objects for $\testcaseABC\in \testsForDomain$ in \indepmodel\ $I$ is the set $\IV(A,B,C,I)=\setcomprehensionLSE{d}{\objectDomain}{\indep{A}{B}{C}{I} \wedge \dep{A}{B}{Cd}{I}}$.
We say that $d\in \objectDomain$ is an \emph{inducing object} in independence model $I$ for $\objectDomain$ if there exists $\testcaseABC$ such that $d$ is an inducing object for $\testcaseABC$ in $I$.

Next we generalize to consider sets of objects that induce dependence.
A set of objects $D \subseteq \objectDomain$ is an \emph{inducing set for} $\testcaseABC\in \testsForDomain$ in independence model $I$ for $\objectDomain$ if both $\indep{A}{B}{C}{I}$ and $\dep{A}{B}{CD}{I}$. 
The inducing sets for $\testcaseABC\in \testsForDomain$ in \indepmodel\ $I$ is the set of sets $\IS(A,B,C,I)=\setcomprehensionLSE{D}{\powerset{\objectDomain}}{\indep{A}{B}{C}{I} \wedge \dep{A}{B}{CD}{I}}$.

The existence of an inducing set in an independence model allows one to infer the existence of additional dependence facts.
Consider the \emph{induced dependence property}: 
$$\indep{A}{B}{C}{I} \ \wedge \ \dep{A}{B}{CD}{I} \implies \dep{A}{D}{CB}{I} \ \wedge \ \dep{D}{B}{CA}{I}.$$

As described in the next proposition, the induced dependence property is guaranteed to hold in any semi-graphoid \indepmodel.

\repproposition{prop:induced-dependence}{
    The induced dependence property holds in any semi-graphoid independence model.}

Due to Propositions~\ref{prop:compprob} and \ref{prop:compgraph}, the induced dependence property holds for both statistical independence in probability distributions and vertex separation in mixed graphs.
Furthermore, if we consider, inducing sets of size one (i.e., $|D|=1$) then the induced dependence property is a statement about the properties of inducing objects (e.g., inducing vertices or inducing variables). For instance, in the context of mixed graphs, the induced dependence property can be used to guarantee the existence of connecting walks from the existence of an inducing vertex. Vertex $d$ being an inducing vertex for sets $A$ and $B$ given set $C$ implies the existence of two connecting walks, one starting at a vertex in $A$ and the other starting at a vertex $B$ with both terminating at $d$. 

The following theorem provides a characterization of essentially undirected graphs in terms of the  existence of inducing vertices in the induced independence model of the graph.

\reptheorem{thm:stat-undirected::no-inducing-vertex}{A graph $G$ is essentially undirected if and only if $\imodel{G}$ contains no inducing vertex.}

We illustrate this proposition with two examples.
Consider graphs $G_4$ and $G_5$ from Figure~\ref{fig:essentially-acylic-graphs}. As described in Section~\ref{sec:essentially-acyclic}, the two graphs are separation equivalent. This implies that $G_4$ is essentially undirected and, by Theorem~\ref{thm:stat-undirected::no-inducing-vertex} graph $G_4$ has no inducing vertices.
Next consider graph $G_6$ and $G_7$ from Figure~\ref{fig:essentially-undirected-graphs}. The two graphs are separation equivalent and thus, Theorem~\ref{thm:stat-undirected::no-inducing-vertex} implies that $G_6$ has no inducing vertices. 
This example illustrates that the graphical condition given by \cite{wermuth1990substantive} for a chain graph to have no equivalent undirected graph is not appropriate for mixed graphs.

\begin{figure}[htbp]
  \centering
  \includegraphics[scale=0.6]{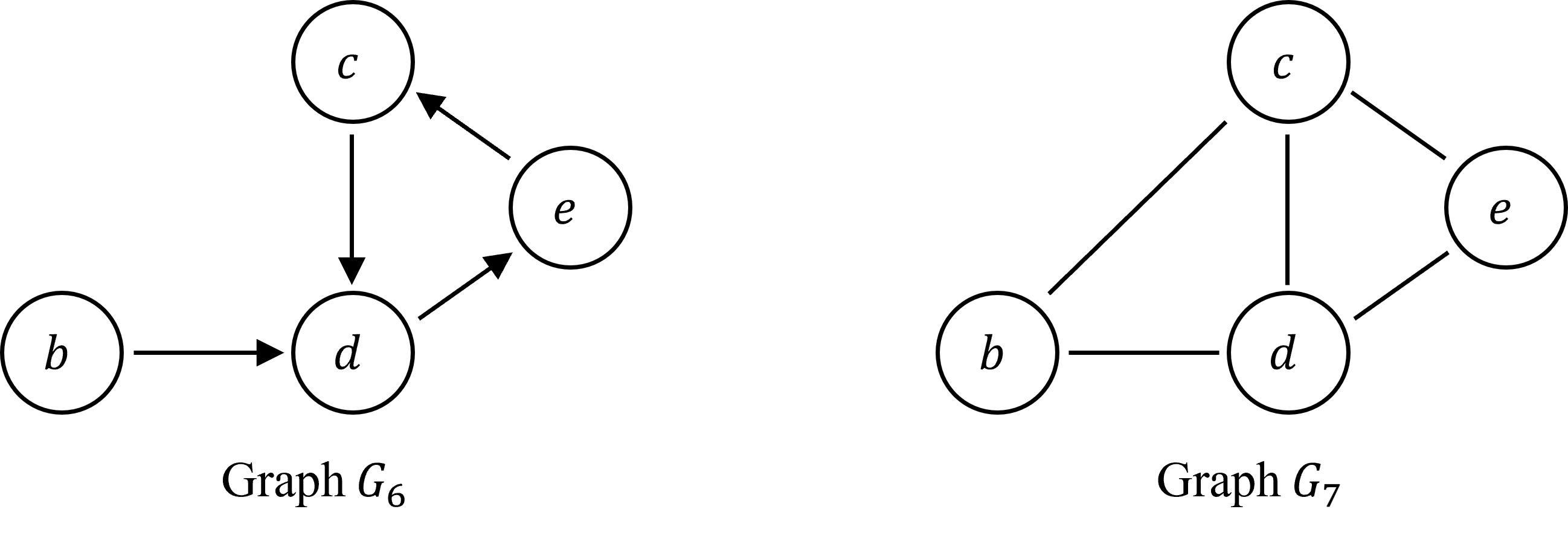}
  \caption{An essentially undirected graph $G_6$ and an equivalent undirected graph $G_7$.}
  \label{fig:essentially-undirected-graphs}
\end{figure}

\newcommand{\oleq}{\mathrel{\lessapprox}}
\newcommand{\ogeq}{\mathrel{\gtrapprox}}
\newcommand{\onotleq}{\mathrel{\not\lessapprox}}
\newcommand{\oeq}{\mathrel{\approx}}
\newcommand{\oneq}{\mathrel{\not \approx}}
\newcommand{\oless}{\mathrel{<}}
\newcommand{\ogtr}{\mathrel{>}}
\newcommand{\onless}{\mathrel{\not <}}
\newcommand{\ongtr}{\mathrel{\not >}}

\newcommand{\oincomp}{\mathrel{\parallel}}
\newcommand{\ocomp} {\mathrel{\curlyvee}}

\newcommand{\oselectleqDV}[2]{#1_{\oleq #2}}
\newcommand{\oselectlessDV}[2]{#1_{\oless #2}}
\newcommand{\oselectgeqDV}[2]{#1_{\ogeq #2}}
\newcommand{\oselectgtrDV}[2]{#1_{\ogtr #2}}

\newcommand{\oequivclass}[1]{[#1]_\approx}

\newcommand{\oleqG}{\mathrel{\lessapprox}_G}
\newcommand{\onotleqG}{\mathrel{\not\lessapprox_G}}
\newcommand{\oeqG}{\mathrel{\approx_G}}
\newcommand{\oneqG}{\mathrel{\not \approx_G}}
\newcommand{\ogeqG}{\mathrel{\gtrapprox_G}}
\newcommand{\olessG}{\mathrel{<_G}}
\newcommand{\ogtrG}{\mathrel{>_G}}
\newcommand{\onlessG}{\mathrel{\not <_G}}
\newcommand{\ongtrG}{\mathrel{\not >_G}}
\newcommand{\oincompG}{\mathrel{\oincomp_G}}
\newcommand{\ocompG}{\mathrel{\ocomp_G}}

\newcommand{\vant}[1] {\oselectgeqDV{V}{#1}} 
\newcommand{\vpost}[1] {\oselectleqDV{V}{#1}} 
\newcommand{\oequivclassG}[1]{[#1]_\approx^G}

\newcommand{\eqIG}{\mathrel{\sim_G}}
\newcommand{\neqIG}{\mathrel{\not\sim_G}}

\newcommand{\ucomponents}{\set{U_1,\ldots,U_c}}
\newcommand{\defUcomponents}{U(G)=\ucomponents}
\newcommand{\acomponents}{\set{A_1,\ldots,A_a}}
\newcommand{\defAcomponents}{A(G)=\acomponents}
\newcommand{\lesssimU}{\mathrel{\lesssim_{U(G)}}}

Next we consider ordering information in  graphs and how it relates to inducing vertices in graphs.
We use $\oleq$ to denote a preorder over a set; a preorder is a binary relation that is transitive and reflexive. We use the following notation for the derived relations (i) $a \oeq b$ if $a\oleq b$ and $b \oleq a$, (ii) $a < b$ if $a \oleq b$ and $b \onotleq a$, (iii) $a \oincomp b$ if $a \onotleq b$ and $b\onotleq a$, and (iv) $a \ocomp b$ if $a\oleq b$ or $b \oleq a$.
We use a preorder $\oleq$ over $V$ to define functions from subsets of $V$ to subsets of $V$.
In particular,
$\vpost{A}= \setcomprehensionLSE{b}{V}{\exists a \in A, b \oleq a}$, $\vant{A}= \setcomprehensionLSE{b}{V}{\exists a \in A, b \ogeq a}$, and 
$\oequivclass{A} = \setcomprehensionLSE{b}{V}{\exists b\in A, b \oeq k}$.

Preorders arise naturally when considering graphs. The \emph{induced preorder} of a graph $G$ is denote $\oleqG$ and defined as $a \oleqG b$ if $a = b$ or there is an anterior walk from $b$ to $a$ in $G$, that is $b\in\ant(a,G)$. The relation $\oleqG$ for graph $\defgraphG$ is a preorder on $V$ as reflexivity is guaranteed by the definition and anterior walks can be appended to show transitivity. Note that the relation $a \oeqG b$ implies that either $a=b$ or the existence of anterial walks from $a$ to $b$ and from $b$ to $a$ in graph $G$.

We begin by noting that the existence of an inducing set in a graph provides information about anterior relationship between the inducing set and other vertices in the graph.

\repproposition{prop:inducing-set:not-all-anterior-to-separated-vertices-or-separating-set}{If $D\in \IS(A,B,C,G)$ then $D\cap \antsetCij \neq \emptyset$ (or equivalently $D \not\subseteq \vant{Cij}$).}

It often simpler to consider inducing vertices. The following property is the specialization of the Proposition~\ref{prop:inducing-set:not-all-anterior-to-separated-vertices-or-separating-set} to inducing vertices.

\repproposition{prop:inducing-vertex:not-anterior-to-separated-vertices-or-separating-set}{If $b\in \IV(i,j,C,G)$ then $b\not \in \antsetCij$ (or equivalently $b\notin \vant{Cij}$).}

Next we consider the relationship $\oeqG$ induced by a graph $G$ and show that any pair of vertices $a \oeqG b$ is also equivalent with respect to membership in sets of inducing vertices in the induced independence model of $G$.
An \indepmodel\ $I$ satisfies the \emph{inducing vertex equivalence property} with respect to $\oleq$ if for all vertices $a,b\in V$, if $a\oeq b$ then for all $i,j\in V$ and $C\subseteq V$, it holds that $a\in \IV(i,j,C,I)$ if and only if $b\in \IV(i,j,C,I)$.

\repproposition{prop:inducing-equivalent:one-is-inducing-vertex:each-is-inducing-vertex}{For any graph $G$, the inducing vertex equivalence property holds in $\imodel{G}$ with respect to $\oleqG$.}

Consider the following property that generalizes the inducing vertex equivalence property to inducing sets.

An \indepmodel\ $I$ satisfies the \emph{inducing set equivalence property} with respect to $\oleq$ if for all vertices $d\in V$ and  all non-empty sets $D,E\subset V$, such that $E,D\subseteq \oequivclassG{d}$, and all $i,j\in V$ and $C\subseteq V$, it holds that $D\in \IS(i,j,C,I)$ if and only if $E\in \IS(i,j,C,I)$.

The inducing set equivalence property, unlike the inducing vertex equivalence property, is not guaranteed to hold in all graphs.
For instance, consider Graph $G_3$ of Figure~\ref{fig:non-weakly-separable-graph}. In this graph, $\oequivclass{c}=\set{c,d,e}=\IV(a,b,\emptyset,G_3)$ which illustrates Proposition~\ref{prop:inducing-equivalent:one-is-inducing-vertex:each-is-inducing-vertex}. The fact that $\indep{a}{b}{cde}{G_3}$ shows that the inducing set property does not hold.

\reptheorem{thm:essentially-separable::inducing-equivalence}{
    $G$ is essentially separable if and only if $\imodel{G}$ satisfies the inducing set equivalence property with respect to $\oleqG$.}

The following corollary of Theorem~\ref{thm:essentially-separable::inducing-equivalence} follows from Theorem~\ref{thm:weakly-separable::essentially-acyclic}.

\begin{corollary}
    $G$ is essentially acyclic if and only if $\imodel{G}$ satisfies the inducing set equivalence property with respect to $\oleqG$.
\end{corollary}

As shown in subsequent sections, inducing vertices and derived concepts like induced arrowheads are also useful for developing separational characterizations of the separation equivalence of separable graphs (Section~\ref{sec:iarrowhead-characterization-of-equivalence}), representing the equivalence classes of separable graphs (Section~\ref{sec:canonical-graph-projections}), and identifying essentially separable graphs (Section~\ref{sec:identifying}).

\section{Graphical characterizations of separation equivalence}
\label{sec:miw-characterization-of-equivalence}

In this section we provide two graphical characterizations of the separation equivalence of two graphs. First
we show that two separable graph are separation equivalent if and only if they have the same adjacencies and same \miw s.
This characterization is a direct generalization of the characterizations of separation equivalence for directed acyclic graphs \cite{VP1990}, chain graphs \cite{Frydenberg1990}, and ancestral graphs \cite{Zhao2005} to separable graphs.

We begin by defining the property of two graphs having the same adjacencies.
Two graphs have the \emph{same adjacencies} if for every edge in one graph there is a corresponding edge with the same endpoints in the other graph and vice versa. 
Note that the definition of same adjacencies does not require the corresponding edges to be of the same type. In other words, only the endpoints of the edges matter and not their endmarks.

Next we define the property of two graphs having the same \miw s.
We define the property based on an equivalence relation over collider walks rather than walk equality.\footnote{We extend this equivalence relation to all walks in Appendix~\ref{sec:walk-decomposition-and-equivalence}.}
To motivate the need to introduce an equivalence relation, consider the graph $a\arrowarrow b \arrowarrow c$ and the graph $a \tailarrow b \arrowtail c$. These two graphs are separation equivalent but their sets of minimal inducing walks are disjoint.
Thus, in order to characterize separation equivalence we need to introduce an equivalence relation. In particular, two collider walks $\gamma$ and $\omega$ are equivalent $\gamma\equiv \omega$ if $\vertexseq{\gamma} = \vertexseq{\omega}$ and $\sectionseq{\gamma}=\sectionseq{\omega}$ or $\vertexseq{\reverse(\gamma)} = \vertexseq{\omega}$ and $\sectionseq{\reverse(\gamma)}=\sectionseq{\omega}$ where $\reverse(\gamma)$ is the walk with the same edges as $\gamma$ but the opposite traversal order. Two graphs have the \emph{same \miw s} if for every \miw\ $\omega$ in one graph there is a corresponding \miw\ $\omega'$ such that $\omega \equiv \omega'$ and vice versa.

\reptheorem{thm:separable-graphs:equivalent::same-adj-and-same-miw}{
    Two separable graphs are equivalent if and only if they have the same adjacencies and same \miw s.}

We illustrate this theorem with two examples. Let $\miwrep(G)$ be the set of \miw s in a graph $G$.
Consider graphs $G_4$ and $G_5$, from Figure~\ref{fig:essentially-acylic-graphs}. The two graphs have the same adjacencies and $\miwrep(G_4)=\miwrep(G_5)=\emptyset$. Thus, by Theorem~\ref{thm:separable-graphs:equivalent::same-adj-and-same-miw}, $G_4 \equiv G_5$.
Next consider the graph $G_1$ from Figure~\ref{fig:separable-and-weakly-separable-graphs} and graph $G_8$ from Figure~\ref{fig:induced-arrowheads}. These two graphs have the same adjacencies and $\miwrep(G_1)=\{a \tailarrow c \arrowarrow b, a\tailarrow c \arrowarrow d, a\tailarrow c \arrowtail e\}\equiv \{a \tailarrow c \arrowtail b, a \tailarrow c \arrowtail d, a \tailarrow c \arrowtail e\} = \miwrep(G_8)$. Thus, by Theorem~\ref{thm:separable-graphs:equivalent::same-adj-and-same-miw}, $G_1 \equiv G_8$.

\begin{figure}[htbp]
  \centering
  \includegraphics[scale=0.6]{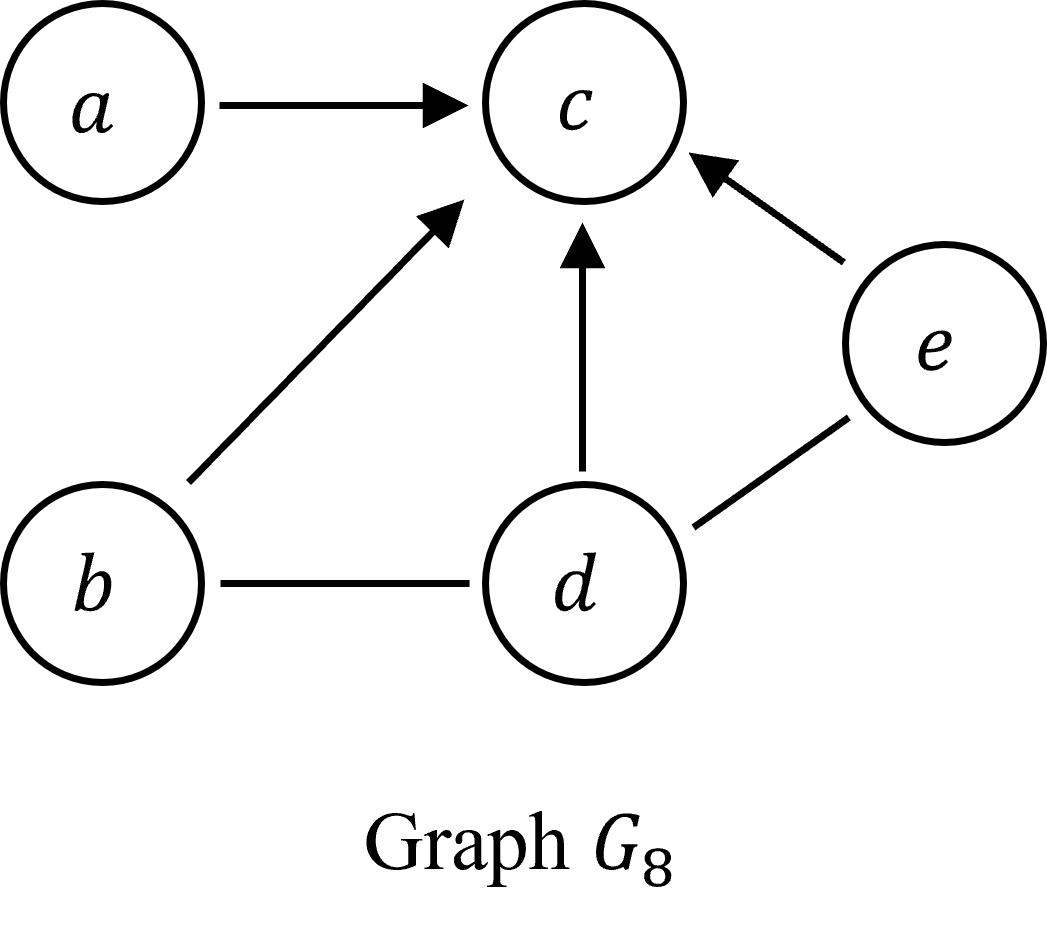}
  \caption{Separable graph $G_8$ is equivalent to graph $G_1$ of Figure~\ref{fig:separable-and-weakly-separable-graphs}.}
  \label{fig:induced-arrowheads}
\end{figure}

Next we present a useful characterization of separation equivalence and describe a useful property related to this characterization. This characterization is refinement of the first characterization where we compare \mdiw s rather than \miw s.
In order to define \mdiw s, we need a few additional definitions.
A section on an inducing walk is \emph{discriminated} if the section is not anterior to either endpoint of the inducing walk. 
An inducing walk \emph{discriminates} a collider section if the section is discriminated by the walk.
An inducing walk is a \emph{(simple) \diw} if it contains exactly one discriminated section and a \emph{compound \diw} if it contains more than one discriminated section.
Note that an inducing walk containing no discriminated sections is a \emph{\siw}.

Two graphs have the \emph{same \mdiw s} if for every \mdiw\ $\omega$ in one graph there is a corresponding \mdiw\ $\omega'$ such that $\omega \equiv \omega'$ and vice versa.

\reptheorem{thm:separable-graphs:equivalent::same-adj-same-mdiw}{
    Two separable graphs are equivalent if and only if they have the same adjacencies and the same \mdiw s.}

For the graphs $G_8$ from Figure~\ref{fig:induced-arrowheads} and $G_1$ from Figure~\ref{fig:separable-and-weakly-separable-graphs} the set of \miw s and the set of \mdiw s are identical, that is, every \miw\ in these graphs is also a \mdiw.

The next proposition shows that discriminated sections of \mdiw s  contain inducing vertices and that all other vertices are not inducing vertices. 

\repproposition{prop:diw-for-section:section-ivs}{If walk $\omega$ between $a$ and $b$ is a \diw\ that discriminates collider section $\sigma$ in graph $G$ and $C$ separates $a$ and $b$ in $G$ then every vertex on $\sigma$ is in $\IV(a,b,C,G)$ and every other vertex on $\omega$ is not in $\IV(a,b,C,G)$.}

This proposition along with Proposition~\ref{prop:inducing-vertex:not-anterior-to-separated-vertices-or-separating-set} play an essential role in proving that our \sia\ identifies all required arrowheads. 

\TextVersion{}{
Our next characterization is refinement of the first characterization where we compare \mdiw s rather than \miw s.
In order to define \mdiw s, we need a few additional definitions.
A section on an inducing walk is \emph{discriminated} if the section is not anterior to either endpoint of the inducing walk. 
An inducing walk \emph{discriminates} a collider section if the section is discriminated by the walk.
An inducing walk is a \emph{(simple) \diw} if it contains exactly one discriminated section and a \emph{compound \diw} if it contains more than one discriminated section.
Note that an inducing walk containing no discriminated sections is a \emph{\siw}.
Two graphs $G$ and $G'$ have the \emph{same \mdiw s} if and only if $G \equiv_\mdiwrep G'$ where $\mdiwrep$ is the \cwp\ of a graph the the set of \mdiw s in the graph.

\reptheorem{thm:separable-graphs:equivalent::same-adj-same-mdiw}{
    Two separable graphs are equivalent if and only if they have the same adjacencies and the same \mdiw s.}

This characterization and is useful for testing the equivalence of two separable graphs
as compared to using of Theorem~\ref{thm:separable-graphs:equivalent::same-adj-and-same-miw}. In order to use Theorem~\ref{thm:separable-graphs:equivalent::same-adj-and-same-miw} to test equivalence, one need to enumerate \miw s. Using Theorem~\ref{thm:separable-graphs:equivalent::same-adj-same-mdiw}, however, one is only required to enumerate \miw s. Because \mdiw s are a subset of \miw s, the use of this characterization for testing equivalence can reduce the computational cost of testing.

Our final characterization compares \iarrowhead s rather than \miw s or \mdiw s.}

\section{Separational characterization of separation equivalence}
\label{sec:iarrowhead-characterization-of-equivalence}

In this section we show that two separable graphs are separation equivalent if and only if they have the same vertex separability and same \iarrowhead s.

We begin by defining the property of two graphs having the same vertex separability.
Two graphs have the same vertex separability if, for every pair of vertices, the pair is separable in one graph if and only if it is separable in the other.
This property is useful because it is a graph separation property that, when the graphs are assumed separable, is equivalent to the graphical condition that the graphs have the same adjacencies.
Consequently, the separational property of same vertex separability can replace the graphical property of same adjacency in Theorems~\ref{thm:separable-graphs:equivalent::same-adj-and-same-miw} and \ref{thm:separable-graphs:equivalent::same-adj-same-mdiw}.

Next we define the property of two graphs having the same \iarrowhead s.
The \iarrowhead s of a graph are defined in terms of the vertex property of being an inducing vertex defined in Section~\ref{sec:inducing-vertices}, and two properties of pairs of vertices: vertex separability and the arrowhead inducing property.
A pair of vertices $\seq{b,d}$ is \emph{arrowhead inducing} at $b$ if there exists $i,j,C$ such that $b\in \IV(i,j,C,G)$ and $d\not\in \IV(i,j,C,G)$.
A pair of vertices $\seq{a,b}$ has an \emph{induced arrowhead} at $b$ if the pair is arrowhead inducing at $b$ and the pair is not vertex separable.
The \iarrowhead\ property of a pair of vertices is a graph separation property as the property of begin an inducing vertex and vertex separability are both graph separation properties.

Note that the definition of an \iarrowhead\, due to the fact it is a graph separation property, makes no mention that an edge has an arrowhead endmark. The following proposition, shows that, despite this property, the existence of an induced arrowhead on an edge in a graph implies the existence of an exclusive arrowhead on that edge.

%
\repproposition{prop:induced-arrowhead:exclusive-arrowhead}{
If an edge $\edgedb$ in a graph $G$ has an induced arrowhead at $b$ then the edge has an exclusive arrowhead at $b$ in $G$.}

Two graphs $G$ and $G'$ have the \emph{same \iarrowhead s} if  for every pair of vertices $\seq{a,b}$ in $G$, the pair has an induced arrowhead at $b$ in $\imodel{G}$ if and only if the pair has an induced arrowhead at $b$ in $\imodel{G'}$. Note that the property is separational, that is, it is defined in terms of graph separation properties of the two graphs.

\reptheorem{thm:separable-graphs:equivalent::same-adj-same-inducing-arrowheads}{Two separable graphs are equivalent if and only if they have the same vertex separability and same \iarrowhead s.}

We illustrate this theorem with two examples.
Consider graphs $G_4$ and $G_5$, from Figure~\ref{fig:essentially-acylic-graphs}. The two graphs have the same adjacencies and neither graph has any induced arrowheads. Thus, by Theorem~\ref{thm:separable-graphs:equivalent::same-adj-same-inducing-arrowheads}, $G_4 \equiv G_5$.
Next consider graph $G_1$ from Figure~\ref{fig:separable-and-weakly-separable-graphs} and graph $G_8$ from Figure~\ref{fig:induced-arrowheads}. These two graphs have the same adjacencies and
both graphs have the same four induced arrowheads: an induced arrowhead at $c$ on edges $\edge{a}{c}, \edge{b}{c}, \edge{d}{c},$ and $\edge{d}{c}$.
Thus, by Theorem~\ref{thm:separable-graphs:equivalent::same-adj-same-inducing-arrowheads}, $G_1 \equiv G_8$.

Finally, we note that, unlike the characterizations in Section~\ref{sec:miw-characterization-of-equivalence}, this characterization is not graphical, but rather separational, in the sense that it can be defined purely in terms of the induced independence models of the two graphs. The fact that this characterization, and the definitions of same vertex separability and same \iarrowhead s are separational is useful for developing \sia s that use statistical tests of independence. In particular, this allows for the identification of adjacencies and arrowheads sufficient to identify the equivalence class of a graph from statistical tests under suitable assumptions.

\section{Representing equivalence classes of separable graphs}
\label{sec:canonical-graph-projections}

In Algorithm~\ref{alg:InducedArrowheadsCanonical} we define the \inducedarrowheads\ algorithm that takes a mixed graph and returns a mixed graph with the same adjacencies while preserving arrowheads on edges only if they have \iarrowhead s.

\begin{minipage}{.8\linewidth}
\begin{algorithm}[H]
\caption{\inducedarrowheads}\label{alg:InducedArrowheadsCanonical}
\begin{algorithmic}
\vspace{0.03in}
\State {\bf Input:} A graph $G=\seq{V,E}$.
\State {\bf Output:} A simple graph.
\vspace{-0.05in}
\\ \!\!\!\!\!\!\!\! \hrulefill
\vspace{0.05in}
\State $E'=\emptyset$
\For {each $\edge{i}{j} \in E$}
    \If {edge $\edge{i}{j}$ has induced arrowheads at $i$ and $j$}
        \State $E'=E' \cup \set{i \arrowarrow j}$
    \ElsIf { edge $\edge{i}{j}$ has an induced arrowhead at $i$}
        \State $E'=E' \cup \set{i \arrowtail j}$
    \ElsIf { edge $\edge{i}{j}$ has an induced arrowhead at $j$}
        \State $E'=E' \cup \set{i \tailarrow j}$
    \Else
        \State $E'=E' \cup \set{i \tailtail j}$
    \EndIf
\EndFor
\State \Return $\seq{V,E'}$

\end{algorithmic}
\end{algorithm}
\end{minipage}
\vspace{0.2in}

We illustrate the \inducedarrowheads\ algorithm by applying it to Graph $G_1$ of Figure~\ref{fig:separable-and-weakly-separable-graphs}. The result is Graph $G_8$ shown in Figure~\ref{fig:induced-arrowheads}. The graph has the same adjacencies at $G_1$. Each of the arrowhead endmarks on an edge is an induced arrowhead on the edge. For instance, the edge $a \tailarrow c$ has an induced arrowhead at $c$ due to the facts that $c\in \IV(a,b,\emptyset,G)$ and $a\not\in \IV(a,b,\emptyset,G)$.


Next we consider properties of the \inducedarrowheads\ algorithm.
First we establish that the \inducedarrowheads\ algorithm preserves separation equivalence if applied to a separable graph.

\reptheorem{thm:separable:inducedarrowheads-equivalent}{If $G$ is separable then $G$ and $\inducedarrowheads(G)$ are equivalent.}


Next, we consider different types of representations of equivalence classes of graph.
The equivalence class of $a\in A$ given equivalence relations $\equiv$ is denoted $\equivclassSRL{A}{\equiv}{a}$ but we use $\equivclassSL{A}{a}$ when the equivalence relation is clear from context.
A function $f\in \function{A}{B}$ is a \emph{canonical representation function for equivalence classes of $A$ with respect to $\equiv$} if it is the case that $a\equiv a'$ if and only if $f(a) = f(a')$ for all $a,a'\in A$. A function $f\in \function{A}{A}$ is a \emph{projection} if $f(f(a))=f(a)$ for $a\in A$. A function $f\in \function{A}{A}$ is a \emph{canonical projection} for $A$ with respect to equivalence relation $\equiv$ if $f$ is a projection and, for $a\in A$, it holds that $f(a)\equiv a$. Thus, a canonical projection is a type of canonical representation. Furthermore, showing that a function is a canonical representation function does not require showing $f(a) \equiv a$.
The following theorem shows that the \inducedarrowheads\ function is a canonical projection.

\reptheorem{thm:induced-arrowheads:sound:identifying}{The 
\inducedarrowheads\ algorithm is a canonical projection for separable graphs with respect to separation equivalence.
}


We relate the \inducedarrowheads\ representation of an equivalence class of separable graph to a generalization of the concept of the largest graph to mixed graphs. The concept of the largest chain graph was introduced in \cite{Frydenberg1990}.
A graph $G'$ is a \emph{largest equivalent simple graph} of $G$ if $G'\equiv G$ and $G'$ is simple and $G'$ has the largest number of tail endmarks of any simple separable graph equivalent to $G$. Let $\largest(G)$ be the set of largest equivalent graphs to $G$.
Frydenberg showed that for chain graphs, every equivalence class contains a unique largest graph that is equivalent to the members of the equivalence class.  The next corollary follows directly from Theorem~\ref{thm:induced-arrowheads:sound:identifying} --- the fact that \inducedarrowheads\ is a canonical projection --- and generalizes Frydenberg's uniqueness result to separable graphs.

\repcorollary{cor:largest}{
    If $G$ is a separable graph then $\largest(G)$ contains a unique graph which is the graph $\inducedarrowheads(G).$}

Finally, we related the \inducedarrowheads\ representation of an equivalence class to anterial graphs. 

\repcorollary{cor:separable:induced-arrowhead-anterial}{
    If $G$ is a separable graph then $\inducedarrowheads(G)$ is an anterial graph.}
\section{Identifying essentially separable graphs}
\label{sec:identifying}

\newcommand{\iadj}{\DefineConstant{iAdj}}
\newcommand{\aiadj}{\DefineConstant{aiAdj}}
\newcommand{\ideg}{\DefineConstant{iDeg}}
\newcommand{\sepset}{\DefineConstant{SepSet}}
\newcommand{\sepsetbound}{\DefineConstant{SepSetBound}}

In this section, we develop the \sgi\ algorithm, a constraint-based \gsia\ that uses only \indeptest s.
We show that \sgi\ identifies the separation‑equivalence class of essentially separable graphs under perfect testing and provide a bound on the number of \indeptest s required.

The \sgi\ algorithm is shown in Algorithm~\ref{alg:SGI}. 
The algorithm outputs a simple graph $G$ as a representation for an equivalence class of essentially separable graphs. The use of a simple graph is justified by Corollary~\ref{cor:anterial-power}.

\begin{minipage}{.8\linewidth}
\begin{algorithm}[H]
\caption{The \sgi\ Algorithm}\label{alg:SGI}
\begin{algorithmic}
\vspace{0.03in}
\State {\bf Input:} Independence model $I$ for variables $V$.
\State {\bf Output:} A simple mixed graph over $V$.
\vspace{-0.05in}
\\ \!\!\!\!\!\!\!\! \hrulefill
\vspace{0.05in}
\State{$E \gets \set{v_i \tailtail v_j \ | \ \set{v_i,v_j}\subseteq V \wedge v_i \neq v_j}$}
\State $G \gets \seq{V,E}$ \Comment{Initialize to be a simple undirected complete graph.}
\State socs $\gets 0$ \Comment{socs is current size of conditioning sets.}
\While {socs $ < \sepsetbound(G)$}
  \For {pair $\pair{i}{j} \in V \times V$}
   \If {$i\in \adj(j,G)$}
    \State candidates $\gets \ant(i,G) \cup \ant(j,G)$.
    \For {$C\subseteq$ candidates such that $|C|= socs$}
       \If {$\indep{i}{j}{C}{I}$} \Comment{Check if $C$ is a separating set}
         \State {Remove edge between $i$ and $j$ in $G$.}
         \State {IV $= \emptyset$}
         \For {$k \in V \setminus (C \cup \set{i,j})$} \Comment{Find inducing vertices for $i,j,C$}
            \If {$\dep{i}{j}{Ck}{I}$} \Comment{Check if $k$ is an inducing vertex}
                \State {IV $=$ IV $\cup \set {k}$}
            \EndIf
         \EndFor
         \For {$k\in $ IV} \Comment{Use Orientation Rule 1}
        \For {$l\in \adj(k,G)$}
            \If {$l \not\in $IV}   
                \State {Change endmark at $k$ on $\edge{l}{k}$ to an arrowhead.} 
            \EndIf
        \EndFor
    \EndFor
       \EndIf
    \EndFor
   \EndIf
  \EndFor
  socs $\gets \text{socs} + 1$ \Comment{Try larger conditioning sets.}
\EndWhile
\State \Return $G$

\end{algorithmic}
\end{algorithm}
\end{minipage}
\vspace{0.2in}

We begin with an overview of the algorithm.
The \sgi\ algorithm starts with a simple complete undirected graph $G$, that is, a graph in which all edges are undirected and there is an undirected edge between every pair of vertices.
The algorithm systematically searches for separating sets $C$ of increasing size for pairs of adjacent variables $i$ and $j$ and uses \indepmodel\  $I=\imodel{\gendist}$ to test whether $\indep{i}{j}{C}{I}$. If the test determines that $i$ and $j$ are separated by $C$ then the edge between $i$ and $j$ is removed.
The algorithm then identifies the inducing vertices $\IV(i,j,C,I)$ and uses them to add arrowhead endmarks to existing edges.
The algorithm continues to try to remove edges until it has considered separator sets of sufficient size. 

Unlike many \indeptest\ based \gsia s, the \sgi\ algorithm adds orientation information and determines the adjacency structure concurrently and uses this orientation information to restrict the search for separating sets.
In particular, the \sgi\ algorithm implicitly uses the following induced arrowhead identification rule to add arrowheads to the current representation $G$.

\begin{orule}[Induced arrowhead identification rule]
If there is an edge $\edge{k}{l}$ in simple graph $G$ such that there exists vertices $i$ and $j$ and vertex set $C$ such that $k\in \IV(i,j,C,I)$ and $l\not\in \IV(i,j,C,I)$, change the endmark at $k$ on $\edge{k}{l}$ to be an arrowhead.
\end{orule}

The induced arrowhead identification rule is a sound orientation rule under the assumption that $I=\imodel{\gendist}$ provides perfect testing for an essentially separable graph $\genstruct$. This is due to the fact that induced arrowheads are exclusive arrowheads (Proposition~\ref{prop:induced-arrowhead:exclusive-arrowhead}).
Note that the algorithm uses previously identified orientation information to restrict attention to sets of vertices that are potentially anterior separating sets.
This restriction is justified by the following theorem and the soundness of the added arrowheads under the assumption of perfect testing.


\repproposition{prop:minimal:anterior}{
{\rm (Minimal separator implies anterior separator){\bf .}}
If $C$ is a minimal separating set for a pair of vertices then $C$ is an anterior separating set for the pair of vertices.}

In order to effectively terminate the algorithm, we need to an upper bound the size the minimal separating sets for a pair of vertices that are adjacent in the current graph $G$.
Proposition~\ref{prop:minimal:anterior} and the correctness of Orientation Rule~1 and edge removal provides a simple tighter bound on the size of the minimal separating set in a graph. In particular, if $C$ is a minimal separator for $a$ and $b$ in $\genstruct$, then $|C| \leq |\ant(ab,\genstruct)| \leq | \ant(ab,G)|$. This, however is not a tight bound and leads to excess computation. For instance, if $\genstruct$ is undirected and connected the bound is vacuous. In graphs without bidirected edges, the degree of the vertices --- the cardinality of the set of adjacent vertices --- can be used to bound the size of the minimal separating set. This, however, is not possible for mixed graphs. For instance, consider graph $G_{11}$ shown in Figure~\ref{fig:non-local-separator}. In order to separate $a$ and $e$ in $G_{11}$ the minimal separating set is the set $\set{b,c,d}$ that includes vertex $c$ that is adjacent to neither $a$ nor $b$. 

This motivates the definition of the induced adjacencies of a vertex. The \emph{induced adjacencies} of a vertex $a$ given vertex $b$ in graph $G$ is the set $\iadj(a,b,G)$ that contains the vertices connected to $a$ by collider walk in which every internal vertex is in $\ant(ab,G)$. In graph $G_{11}$, $\iadj(a,b,G)=\iadj(b,a,G)=\set{b,c,d}$.\footnote{The set of induced adjacencies is a generalization of the set $D$-$SEP$ defined by \cite{sgs} and Proposition~\ref{prop:induced-adj-separation} is related to Theorem~6.2 and Lemma~6.2.4 proved in \cite{sgs}.} The \emph{anterior induced adjacencies} of a vertex $a$ given vertex $b$ in graph $G$ is the set $\aiadj(a,b,G)=\iadj(a,b,G) \cap \ant(ab,G).$ The next proposition shows that, if $a$ and $b$ are separable then the anterior induced adjacencies of $a$ given $b$ are sufficient to separate $a$ and $b$.

\repproposition{prop:induced-adj-separation}{If $a$ and $b$ are separable in graph $G$ then $\indep{a}{b}{C}{G}$ where $C = \aiadj(a,b,G)$.}

\begin{figure}[htbp]
  \centering
  \includegraphics[scale=0.6]{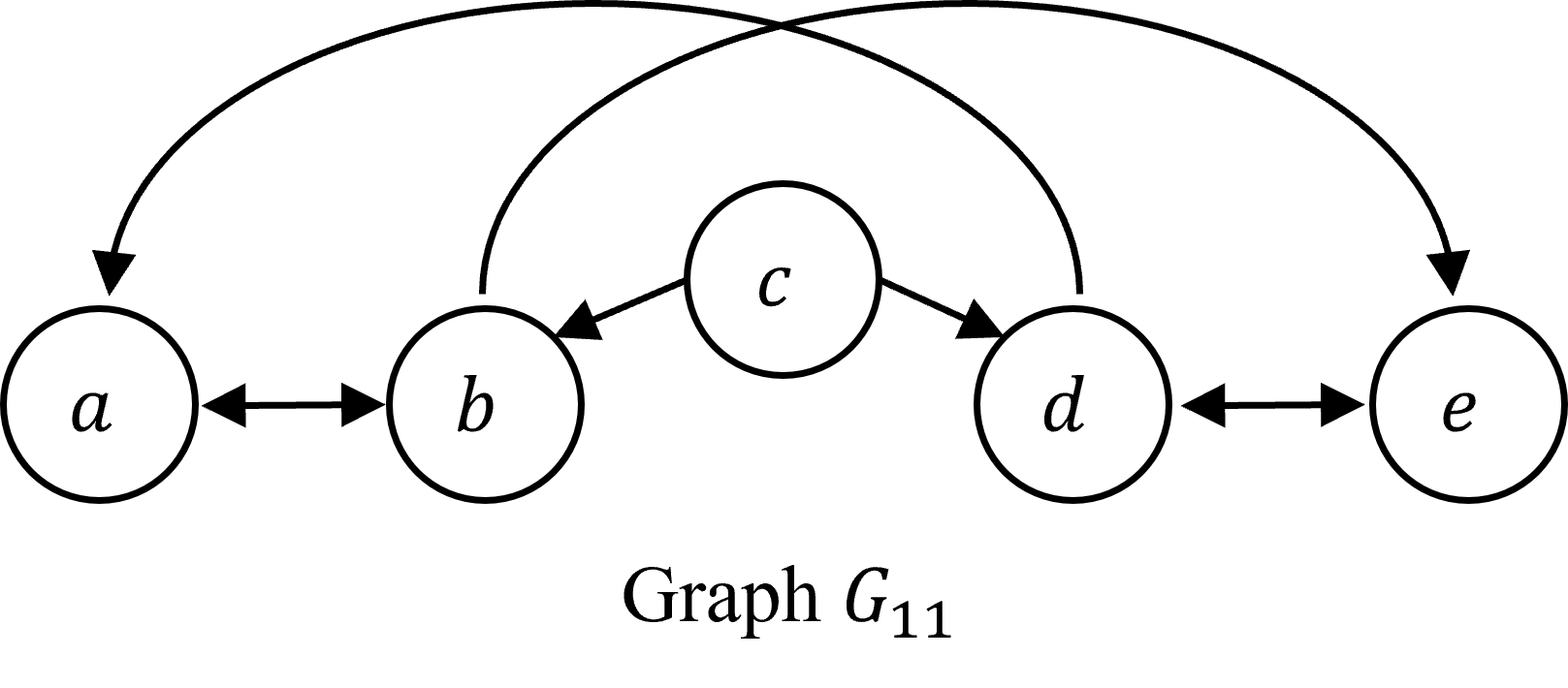}
  \caption{Graph $G_{11}$ illustrates that vertex separation is non-local.}
  \label{fig:non-local-separator}
\end{figure}

The quantity $\sepsetbound(a,b,G)$ is the \emph{separation set bound} for a pair of vertices $a$ and $b$ in graph $G$ and defined to be 0 if $a\not\in \adj(b,G)$ and to be $\sepsetbound(a,b,G)=\max(|\aiadj(a,b,G)|,|\aiadj(b,a,G)|)$ otherwise. 
The following proposition captures the fact that the separation set bound for $G$ bound the size of the minimal separation set in $\genstruct$.

\repproposition{prop:sep-set-bound}{If $a\not\in \adj(b,\genstruct)$ and $|C|$ is a minimal separating set for $a$ and $b$ in $\genstruct$ then under perfect testing, $|C| \leq \sepsetbound(a,b,G)$.}

We then use this bound to define an overall bound for the graph $G$. The \emph{separation set bound} of graph $G$ is defined to be $\sepsetbound(G)=\max_{a,b \in V} \sepsetbound(a,b,G)$ and use it to determine when to terminate the \sgi\ algorithm.

The following theorem describes the correctness of the \sgi\ algorithm.
The correctness result assumes that the $\gendist$ provides \emph{perfect testing} with respect to $\genstruct$ for the \sgi\ algorithm.

\reptheorem{thm:SGI:identifies-graph}{If $\genstruct$ is essentially separable then
$\sgi(\imodel{\genstruct})\equiv \genstruct$ and 
if $G \equiv \genstruct$ is separable then $\sgi(\imodel{\genstruct})= \inducedarrowheads(G)$.}

From Corollaries~\ref{cor:largest} and \ref{cor:separable:induced-arrowhead-anterial}, the graph $\sgi(\imodel{\genstruct})$ is both the largest graph equivalent to $\genstruct$ and an anterial graph equivalent to $\genstruct$.

\repcorollary{cor:SGI:identifies}{If $\genstruct$ is essentially separable and $\gendist$ provides perfect testing for $\genstruct$ with respect to \sgi\ then the \sgi\ algorithm identifies the equivalence class of $\genstruct$.}

We illustrate the \sgi\ algorithm applied to $\imodel{G_2}$ from Figure~\ref{fig:separable-and-weakly-separable-graphs}. The initial graph is the simple complete undirected graph (not shown). The first separation fact that is found is $\indep{a}{b}{\emptyset}{G}$. The algorithm removes the edge between $a$ and $b$ and then computes the set $\IV(a,b,\emptyset,G)=\set{c}.$ It then uses Orientation Rule 1 to add arrowhead endmarks at $c$ on the edges $e(a,c), e(b,c), e(d,c), e(e,c).$ The results at this point is shown as Graph $G_9$ in Figure~\ref{fig:sgi-intermediate-steps}. The algorithm then continues to search for separation facts. It next finds that $\indep{a}{d}{\emptyset}{G}$, removes the edge $e(a,d)$, computes the set $\IV(a,d,\emptyset,G)=\set{c}$. In this case the application of Orientation Rule 1 adds no new arrowhead endmarks. It next finds $\indep{a}{e}{\emptyset}{G}$, removing the edge, and again no new arrowhead endmarks are added.  The result at this stage is shown at Graph $G_{10}$ in Figure~\ref{fig:sgi-intermediate-steps}. The algorithm then finds $\indep{b}{e}{\emptyset}{G}$, and removes the edge $e(b,e)$ and adds no new arrowhead endmarks. No additional independencies are found and the result is the graph $G_8$ in Figure~\ref{fig:induced-arrowheads}. Note that the adjacencies of $G_6$ and $G_2$ are not the same as $G_2$ is essentially separable. Nonetheless, it does have the same adjacencies as the graph $G_1$ and the same induced arrowheads.

\begin{figure}[htbp]
  \centering
  \includegraphics[scale=0.6]{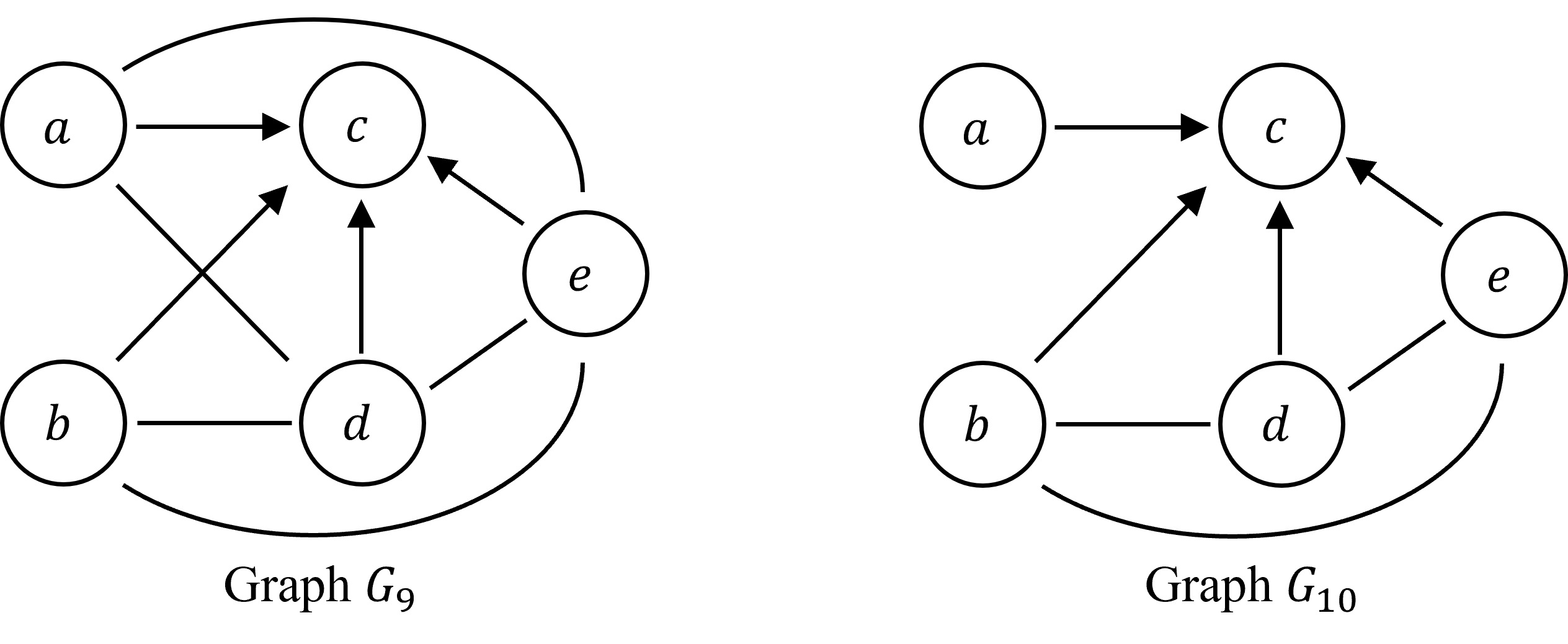}
  \caption{Graph $G_9$ and $G_{10}$ are intermediate graphs obtained while applying the \sgi\ algorithm.}
  \label{fig:sgi-intermediate-steps}
\end{figure}

The next theorem provides a simple bound on number of independence tests required by the \sgi\ algorithm in terms the properties of a graph equivalent to $\genstruct$.

\reptheorem{thm:polynomial-bound}{If $\genstruct$ is essentially separable and $\gendist$ provides perfect testing for $\genstruct$  then the \sgi\ algorithm requires $O(|V|^{\sepsetbound(G)+1})$ independence tests where $G=\largest(\genstruct)$ is the representation of the equivalence class of $\genstruct$.}

\section{Discussion}
\label{sec:discussion}

This paper contributes to a long series of work investigating the properties of families of graphical models and algorithms identifying equivalence classes of graphical models using independence tests. In this section, we discuss aspects of our work as it relates to previous work that has not been discussed in the previous sections.

\subsection{Essentially cyclic graphical models}
\label{sec:beyond-weakly-separable}

Cyclic directed graphs have been used to represent the equilibrium of systems with feedback. For instance, \cite{Richardson1996} considers simple graphs with directed edges and \cite{Koster1996} considers reciprocal graphs; graphs with directed and undirected edges with additional restrictions. Both of these families are subfamilies of the mixed graphs considered in this paper.

A characterization of separation equivalence for directed graphs was given by \cite{Richardson1997}. 
A \sia\ that identifies the equivalence class of a directed graphical models under perfect testing was given in \cite{Richardson1996}. More recently an alternative characterization of separation equivalence for directed graphs was given by \cite{ClaassenMooij_UAI_23}. 
The family of directed graphs includes some essentially cyclic graphs (e.g., graph $G_3$ from Figure~\ref{fig:non-weakly-separable-graph}).

The characterizations of the separation equivalence of separable graphs given in Sections~\ref{sec:miw-characterization-of-equivalence} and \ref{sec:iarrowhead-characterization-of-equivalence} include some cyclic graphs but does not include essentially cyclic graphs. A natural generalization of our work and previous work on directed graphs is to provide a characterization of separation equivalence for all mixed graphs, to develop a canonical representation function for the equivalence class of mixed graphs, and to develop \sia s that identify the equivalence class of mixed graphs.

\subsection{Sound and complete representation of equivalence classes}
\label{sec:sound-and-complete}

In this paper, the goal of our \sgi\ \sia\ is to obtain a representation of the equivalence class of the generative graphical model under perfect testing. In fact, the output of SGI provides additional information about the generative structure of the graphical model as all identified arrowheads are exclusive arrowheads and shared by all equivalent separable graphs. A natural goal for a \sia\ it to provide correct information about endmarks (soundness) and as much information about endmarks as possible (completeness). This is typically accomplished through the application of orientation rules. A set of sound and complete orientation rules for directed acyclic graphs was presented by \cite{Meek95causal} and for ancestral graphs by \cite{zhang2007}. 

From Theorem~\ref{thm:induced-arrowheads:sound:identifying}, we know that if $G$ is a separable graph then $\inducedarrowheads(G)\equiv G$. Corollary~\ref{cor:largest}, implies that it is not possible to add any additional arrowheads. This implies that the \sgi\ \sia\ is arrowhead complete in the sense that no additional arrowheads can be added.
Developing a set of sound and complete orientation rules for essentially separable graphs (e.g., applied to the output of the \sgi\ algorithm) is a natural next step. In order to accomplish this one needs to extend the class of graphs used to represent the output as in \cite{zhang2007}. This is due to the fact that for some separation equivalence classes some but not all of the tail endmarks are exclusive.

The arrowhead completeness of the SGI algorithm enables one to explore the impact of the ontological assumptions about the generative process on the identifiability of features of the equivalence class of the generative model. For instance, if we assume that $G$ is a directed acyclic graph, then one can use the PC algorithm and the orientation rules R1-R4 of \cite{Meek95causal} to obtain a sound and complete set of endmarks under perfect testing. However, an equivalent graph, under the weaker assumption that the generative model is essentially separable might not have the same arrowheads. For instance, consider the graph with edges $a \tailarrow c$, $b \tailarrow c$, $c \tailarrow d$. The result of the PC algorithm and the orientation rules under perfect testing adds of an arrowhead at $d$ on the edge  $\edge{c}{d}$ which is not added by the SGI algorithm. This is justified by the existence of the equivalent graph in which $\edge{c}{d}$ is undirected.
Such additional arrowheads are sound when making the ontological assumption that the family generating processes are directed acyclic graphs rather than essentially separable graphs. This situation is similar when assuming $G$ is an ancestral graph and applying the orientation rules R1-R4 of \cite{zhang2007}. Again, one obtains an arrowhead at $d$ on $\edge{c}{d}$ as the orientation rules are sound for ancestral graphs but not for essentially separable graphs.

\subsection{Assumption for identification and their plausibility}
\label{sec:identification-assumptions}

The correctness results for the \sgi\ algorithm (Corollary~\ref{cor:SGI:identifies}) assume perfect testing.
In this section we consider the tenability and plausibility of this identification assumption.

A distribution $P$ is \emph{perfect} with respect to $G$ if $\indep{A}{B}{C}{G} \iff \indep{A}{B}{C}{P}$.\footnote{This assumption is sometimes referred to as the distribution and the graph being faithful; see, for example, \cite{Sadeghi2017}. Faithfulness, however, often refers only one direction. In particular, the direction opposite of $\gendist$ being Markov with respect to $\genstruct$; see \cite{sgs}.} Clearly, if a distribution $P$ is perfect with respect to $G$ then $\imodel{P}$ provides perfect testing for $G$. We establish the tenability of the perfect testing assumption by establishing the existence of perfect distributions for all essentially separable graphs. This follow from the paper \cite{Sadeghi2017} that shows the existence of a perfect distribution for every chain mixed graphs and 
Theorem~\ref{thm:simplify:separable-graph:equivalent-anterial-graph}, every separable graph has an equivalent anterial graph.
Other related results for more restricted graph families and distribution families are those of \cite{sgs, Meek1995faithfulness, Pena2009}.

Next we consider the plausibility of the identification assumption. 
Previous work has justified the assumption of the generative distribution being a perfect distribution for particular defining distribution families and graph families by showing that, in a specific measure theoretic sense, most distributions that can be represented by a graph are perfect. \cite{sgs, Meek1995faithfulness, Pena2009} One can hope to extend these results to richer families of graphs and families of distributions to show the plausibility of the assumption of perfect testing.

\subsection{Elucidating the assumptions for identification}

The perfect testing assumption used in Corollary~\ref{cor:SGI:identifies} does not provide insight into the nature of the assumptions required the \sgi\ algorithm correctness. The conjunction of the following alternative compatibility conditions are sufficient for \sgi\ to identify structure and provide some insight into the requirements for correctness.

1) \emph{Bounded vertex separability perfectness}: for all $i$ and $j$, the vertices $i$ and $j$ are not vertex separable in $\genstruct$ if and only if for all  $C\subseteq V \setminus \set{i,j}$ such that $|C|< \sepsetbound(i,j,\genstruct)$, $\dep{i}{j}{C}{\gendist}$.

2) \emph{Bounded induced arrowhead perfectness}: if vertices $k$ and $l$ are not vertex separable in $\genstruct$, then for all $i$, $j$, and $C\subseteq V\setminus \set{i,j}$ such that $|C|<\sepsetbound(i,j,\genstruct)$, 
\begin{align*}
k &\in \IV(i,j,C,\imodel{\genstruct}) \wedge l\notin \IV(i,j,C,\imodel{\genstruct}) \\
&\iff\;
k \in \IV(i,j,C,\imodel{\gendist}) \wedge l\notin \IV(i,j,C,\imodel{\gendist})
\end{align*}

Informally, the sufficiency of these conditions follows from Theorem~\ref{thm:separable-graphs:equivalent::same-adj-same-inducing-arrowheads} and the following informal arguments.

The bounded vertex separability perfectness condition guarantees adjacency correctness in the result of the \sgi\ algorithm. The forward direction of this condition guarantees that no required edge is removed and the reverse direction of the condition guarantees that missing edges are removed. The forward direction of bounded vertex separability is essentially a bounded version of adjacency faithfulness of \cite{ZhangSpirtes08} with the other direction being a bounded local Markov condition (see also \cite{TehSadeghiSoo2025}).

The bounded induced arrowhead perfectness condition guarantees that any edge in the result of the \sgi\ algorithm will be oriented with an arrowhead if and only if it has an induced arrowhead in $\genstruct$. The induced arrowhead is stronger than perfect testing as it requires (e.g.) correctness for all $i,j$ and all separating sets $C$ with $|C| < \sepsetbound(i,j,\genstruct)$ but the algorithm uses (e.g.) only minimal separating sets.

Finally, we note that the assumption of perfect testing is also sufficient but not minimal due to the fact that one can recover from (1) failing to remove an edge due to an incorrect test if there is another test later that removes the edge, and 
(2) failing to orient an induced arrowhead on an edge in $\genstruct$ if another edge removed later leads to the orientation of the required induced arrowhead.

\subsection{More efficient structure identification algorithms}
\label{sec:more-efficient}


Most \indeptest\ based \sia s, such as the the PC algorithm \cite{sgs}, the FCI algorithm \cite{sgs} and the FCI+ algorithm \cite{ClaassenMooijHeskes2013}, use some sort of locality in their search for separatings sets for pairs of variables. For instance, the PC algorithm, when trying to separate vertices $a$ and $b$ only consider vertices that are adjacent to $a$ and $b$. As discussed in Section~\ref{sec:identifying}, this is sufficient in if there are no bidirected edges. If, however, the generative structure includes bidirected edges one needs to augments the candidate set for potential separating sets to include additional vertices. For instance, the FCI \cite{sgs}, RFCI \cite{ColomboEtAl2012}, and FCI+ \cite{ClaassenMooijHeskes2013} each consider vertices beyond those adjacent to the endpoints of an edge they are trying to remove. In this paper, we also consider extended sets of vertices that we call the set of induced adjacencies. Unlike these other approaches, we do not restrict separating sets to to induced adjacencies but rather allow for a non-local search among vertices that are potentially anterior to the endpoints of the edge that we are trying to remove.

One can, however, consider a more localized version of the \sgi\ algorithm in which we choose potential separating sets for $i$ and $j$ using the current graph $G$.  In particular, Proposition~\ref{prop:induced-adj-separation}  shows that the anterior inducing adjacencies provide a more localized separation criterion as compared to search among vertices that are anterior to either $i$ or $j$. This, one can choose $C$ to be a subset of $\aiadj(i,j,G)$ or to be a subset of $\aiadj(j,i,G)$. One complication in doing so, is that, unlike the sets $\ant(i,G)$ and $\ant(j,G)$, there is no guarantee that $\aiadj(i,j,G) \subseteq \aiadj(i,j,\genstruct)$. Thus if searching from among possible separating sets (e.g.) $C\subseteq \aiadj(i,j,G)$ we need to check all subsets of size $\leq socs$. It would be interesting to compare alternative \sia s that use both locality and orientation information to improve computational efficiency. 

In addition to computational efficiency, it is important to consider the statistical reliability of \sia s. There are a number of different approaches to improving reliability including using redundant tests and removing unnecessary tests. The \sgi\ algorithm, after identifying a separating set, determines the complete set of inducing variables for the separated vertices. It is not necessary to identify this entire set as many of the induced arrowheads can be identified on the basis of other separating sets. It would be interesting to explore the reliability of the \sgi\ algorithm as compared with other \indeptest\ based \sia s and to explore if reliability or efficiency can be improved by using locality and by being lazy about computing complete sets of inducing variable for separated vertices especially when future tests will establish the existence of an arrowhead.

Finally, we note that Classen et al. \cite{ClaassenMooijHeskes2013} also provide a polynomial bound for identifying ancestral. The polynomial bound we provide in Theorem~\ref{thm:polynomial-bound} is not, however, directly comparable.

\subsection{Maximal graphs}

In Section~\ref{sec:essentially-separable}, we showed that there are maximal graphs that are not separable. Previous work such as \cite{LauritzenSadeghi2018} and \cite{RichardsonSpirtes2002} have provided maximization algorithm for subfamilies of mixed graphs. It would be interesting to develop algorithms that transform a graph into a maximal equivalent graph and study their properties. The properties of maximal graphs are likely to prove useful in extending the results on characterizing separation equivalence to all mixed graphs and in developing identification algorithms for the separation equivalence classes of mixed graphical models.

\subsection{Other separation criteria}
\newcommand{\sseparation}{$\sigma$-separation}
\newcommand{\sseparable}{$\sigma$-separable}
\newcommand{\seseparable}{$\sigma$-essentially separable}
\newcommand{\seacyclic}{$\sigma$-essentially acyclic}
The focus of this paper has been on the properties of mixed graphs with respect to d-separation. It would be interesting to consider separable, essentially separable and essentially acyclic graphs defined with respect to other separation criteria. In particular, the \sseparation\ criteria has been defined for mixed graphs with directed and bidirected edge. It would be interesting to extend \sseparation\ to mixed graphs and explore the properties of \sseparable, \seseparable, and \seacyclic\ graphs.
In particular, one might expect that all mixed graphs in such an extension, are \seacyclic, given the work by \cite{Spirtes1995cyclic} and \cite{ForreMooij_UAI_19} related to acyclification. Furthermore, one would then expect that all mixed graphs are \seseparable.



\bibliographystyle{imsart-number} 
\bibliography{main}       

\newpage

\appendix
\section{Proofs of properties of independence models}
\label{appendix:useful-properties-of-independence-models}

\TextVersion{
In this section we prove Proposition~\ref{prop:sep-equiv::pairwise-sep-equiv} from Section~\ref{sec:indepmodels}.
\vspace{1em}

\refproposition{prop:sep-equiv::pairwise-sep-equiv}

\begin{proof}
    The forward direction follows from the fact that the \pairwise\ projection of an \indepmodel\ contains all and only the pairwise \indepstatement s that hold in a graph. For the reverse direction
    we assume that $\pairwise(I)=\pairwise(I')$. Assume that the claim does not hold. In that case there must be an $\testcaseABC$ such that either (i) $\indep{A}{B}{C}{I}$ and $\dep{A}{B}{C}{I'}$ or (ii)  $\dep{A}{B}{C}{I}$ and $\indep{A}{B}{C}{I'}$. \Wlog, assume (i).
    From decomposition we have $\indep{a}{b}{C}{I}$ for $a\in A$ and $b\in B$. From pairwise equivalence, we have $\indep{a}{b}{C}{I'}$ for $a\in A$ and $b\in B$. Finally, from composition we have $\indep{A}{B}{C}{I'}$. Thus we have a contradiction and the claim holds.
\end{proof}
}{
In this section we prove Propositions~\ref{prop:sep-equiv::pairwise-sep-equiv} and \ref{prop:cep-properties} from Section~\ref{sec:indepmodels}.
\vspace{1em}

\refproposition{prop:sep-equiv::pairwise-sep-equiv}

\begin{proof}
    The forward direction follows from the fact that the \pairwise\ projection of an \indepmodel\ contains all and only the pairwise \indepstatement s that hold in a graph. The reverse direction follows from the fact that every non-pairwise dependence statement $\indepABC$ that holds in a \indepmodel\ $I$ is a deterministic function of $\pairwise(I)$ and the trivial independence function and the contrapositive versions of the symmetry and decomposition properties applied to the set of pairwise dependent statements that hold in $I$.
\end{proof}

\refproposition{prop:cep-properties}

\begin{proof}
Let $\defIndepModelIForObjectDomain$ and  $L_1, L_2,S_1,S_2$ be disjoint sets of $\objectDomain$.

\newcommand{\Cond}{\DefineConstant{C}}
\newcommand{\Marg}{\DefineConstant{M}}
\newcommand{\Remove}{\DefineConstant{R}}

We define useful projections for sets of tests $T \subseteq \testsForDomain$ for object domain $\objectDomain$.

$$\Cond_S(T)= \setcomprehensionLSE{\testcaseABC}{\tests_{\objectDomain\setminus S}}
{\seq{A,B,C\cup S} \in T \wedge (A\cup B \cup C)\cap S = \emptyset}$$
$$\Marg_L(T)= \setcomprehensionLSE{\seq{A,B,D}}{\tests_{\objectDomain\setminus L}}{\seq{A,B,D} \in T \wedge (A\cup B \cup D) \cap L = \emptyset}$$


These projections allow us to provide an alternative definition of the \cep\ projection:
$$\indepmodelConditionalExpectation{L}{S}(I)=
\seq{\objectDomain \setminus (L \cup S), \Cond_S(T) \cap \Marg_L(T)}.$$

The first property in the proposition then follows from the fact that $\Cond_\emptyset(T)=\Marg_\emptyset(T)=T$ and basic set manipulation.

Next we define $\cap$ for \indepmodel s by lifting $\cap$ on sets to tuples. In particular,
$$\seq{\objectDomain,T}\cap \seq{\objectDomain',T'}=\seq{\objectDomain \cap \objectDomain', T \cap T'}.$$
This allows us to relate the composition of \cep s to intersection for disjoint subsets $L_1,L_2,S_1,S_2$ of $\objectDomain$ as follows:
$$\indepmodelConditionalExpectation{L_1}{S_1}\circ \indepmodelConditionalExpectation{L_2}{S_2}=\indepmodelConditionalExpectation{L_1}{S_1}\cap \indepmodelConditionalExpectation{L_2}{S_2}.$$

Properties (2) and (3) from the proposition then follows from basic properties of sets.

\end{proof}

Next we provide an alternative proof that the independence property weak transitivity holds in mixed graphs from Section~\ref{sec:separation}.
\vspace{0.1in}

\refproposition{prop:weaktrans}

\begin{proof}
    Follows from Lemma~\ref{lem:open-walks-to-common-vertex:exist-open-walk-between-endpoints}, and decomposition.
\end{proof}
}

\section{Proofs of useful properties of walks and vertex sets}
\label{appendix:useful-properties-proofs}

\subsection{Manipulating walks}

We can use the subsequence operator to define a \emph{subwalk}. For instance, if $\omega=\seq{e_1,\ldots, e_n}$ then $\gamma=\omega(e_i,e_j)=\seq{e_i,\ldots,e_j}$ is the subwalk containing all of the edges from $e_i$ to $e_j$.
We extend the definition of the subwalk operator to allow for the specification of a subwalk via a pair of vertices on the walk or a subsequence of the vertex sequence of the walk.
Consider the walk $\omega=\seq{e_1,\ldots,e_{n-1}}$ with vertex sequence $\defwalkomega$.
We use $\omega(v_i,v_j)$ to refer to the subwalk of $\omega$ containing the edges corresponding to adjacent pairs of vertices on the vertex subsequence $\vertexseq{\omega}(v_i,v_j)$. 
We also allow the subwalk operator to be applied to a vertex subsequence of the vertex sequence of the walk. For instance, $\omega(\seq{v_i,\ldots,v_j})=\omega(v_i,v_j)$.
Note that the subwalk operator can be used to reverse the sequences of edges in a walk. For instance, $\omega(v_n,v_1)=\seq{e_{n-1},\ldots,e_1}$ is the walk whose first vertex is the last vertex of the walk $\omega$.  Note that $\omega(v_i,v_i)=\omega(\seq{})=\omega(\seq{v_i})=\seq{}$ is an empty walk. We use the binary operator $\append$ to append two sequences (e.g., $\omega(e_1,e_n)= \omega(e_1,e_i) \append \omega(e_i,e_n)$)

\subsection{Termination properties of walks}

A walk $\omega$ with $\defwalkomega$ is \emph{into} $v_1$ if the edge $\edge{v_1}{v_2}$ on $\omega$ has an arrowhead at $v_1$, otherwise it is \emph{out of} $v_1$.
For instance, the walk $a \tailarrow b \arrowarrow c$ is out of $a$ and into $c$.
A walk $\omega$ with $\defwalkomega$ and $\defsectionomega$ is \emph{weakly into} $v_1$ if either (1) $\omega$ is into $v_1$ or (2) the edge $\edge{\last(\sigma_1)}{\first(\sigma_2)}$ has an arrowhead at $\last(\sigma_1)$, otherwise it is \emph{strongly out of} $v_1$. For instance, the walk $a \tailtail b \tailarrow c \tailarrow d \tailtail e$ is strongly out of $a$ and weakly into $e$.
Note that a walk that is into an endpoint is necessarily weakly into that endpoints and that a walk that is strongly out of an endpoint is necessarily out of that endpoint.

\subsection{Open walks and connecting walks}
\label{sec:open-walks}

Next we define a generalization of connecting walks that yields an equivalent definition of vertex separation that is useful for describing and proving results about mixed graphs.

A walk $\omega$ is \emph{open} given $C$ if every collider  section of $\omega$ is anterior to $\Cij$ and every internal vertex on a non-collider section is not in $C$. If a walk is not open given $C$ then it is \emph{closed} given $C$. 
Note that the walk $i \arrowarrow a \arrowarrow i \tailtail b \arrowtail j$ is open given $C=\set{i}.$
A collider section of a walk is \emph{open} given $C$ if every vertex is anterior to $C$ and is \emph{closed} otherwise.
A non-collider section of a walk is \emph{open} given $C$ if every vertex on the section that is not an endpoint of the walk is not in $C.$
Thus, a walk is open given $C$ if and only if every section of the walk is open given $C.$

The concept of shortest open walk often plays an important role in proving separation properties of graphs (including the following lemma). A walk $\omega$ is a \emph{shortest open walk} given $C$ if $\omega$ is open given $C$ and there is no shorter walk that is open given $C$.

The following lemma captures the property that if is there is an open walk between $i$ and $j$ given $C$ there must be another open walk given the $\Cmij.$

\begin{lemma}\label{lem:open-walk-given-C-and-possible-endpoints::open-walk-given-C-without-endpoints}
    If walk $\omega$ between $i$ and $j$ is open given $C$ in graph $G$ then there is an open walk $\omega'$ that is open given $\Cmij$ in $G$.
\end{lemma}

\begin{proof}
    Let $\omega$ with $\defwalkomega$ and $\defsectionomega$ be the shortest walk between $i$ and $j$ that is open given $C$ such that there is no subwalk including $\omega$ itself between $i$ and $j$ that is open given $\Cmij.$ We consider two cases.
    
    \proofcase{1} $C\cap \set{i,j} = \emptyset.$ In this case, the walk $\omega$ is open given $\Cmij$. This contradiction the assumption that there is no subwalk open given $\Cmij$ and we have a contradiction.

    \proofcase{2} $C \cap \set{i,j} \neq \emptyset.$ 
    Note that the only occurrence of $i$ on $\sigma_1$ can be as vertex $v_1$ and the only occurrence of $j$ on $\sigma_m$ can be as vertex $v_n.$ 
    If the only section containing $i$ is on $\sigma_1$ and the only occurrence of $j$ is on $\sigma_m$ then $\omega$ open given $\Cmij,$ a contradiction.
    If $i\in C\cap \set{i,j}$ then every occurrence of $i$ must be in a non-collider section otherwise the walk is not open given $C$. Similarly for $j.$
    If every occurrence of $i$ or $j$ is on collider section that contains a member of $\Cmij$ then  $\omega$ is open given $\Cmij,$ a contradiction.
    Thus there must be a collider section containing a vertex in $\set{i,j}$ and no other vertex in $\Cmij$. \Wlog, assume $\sigma_k$ is a collider section containing $v_l=i$ and no vertex in $\Cmij.$ In this case, the walk $\omega(v_l,v_n)$ is a shorter walk open given $C,$ a contradiction.
    
    In all cases we have a contradiction so the lemma must hold.    
\end{proof}

The following lemma proves that one can transform an open walk given $C$ into another open walk given $C$ in which all collider sections are anterial to a vertex in $C$.

\begin{lemma}\label{lem:open-give-c:all-colliders-anterior-to-C}
    If there is a walk $\omega$ between $i$ and $j$ that is open given $C$ then there is a walk $\omega'$ between $i$ and $j$ open given $\Cmij$ in which all collider sections on $\omega'$ are anterior to $\Cmij$.
\end{lemma}

\begin{proof}
    Let $\omega$ be an open walk given $C$ with $\defwalkomega$ and $\defsectionomega$. By Lemma~\ref{lem:open-walk-given-C-and-possible-endpoints::open-walk-given-C-without-endpoints}, we can assume that $C\cap \set{i,j}=\emptyset.$    
    We construct the required walk $\omega'$ as follows:

    If there is a collider section on $\omega$ that is anterior to $v_1$ but not anterior to $C$ then choose $i$ to be as large as possible such that $\sigma_i$ is such a collider section and choose $i=1$ otherwise.
    If $i > 1$ then let $\gamma_i$ be the shortest anterior walk from $\first(\sigma_i)$ to $v_1$ and let $\gamma_1=\seq{}$ otherwise. Due to the choice of $\sigma_i$, the walk $\gamma_i$ contains no vertices in $C$.

    If there is a collider section on $\omega$ that is anterior to $v_n$ but not anterior to $C$ that is not before $\sigma_i$ then choose $j$ to be as small as possible but with $j\geq i$ such that $\sigma_j$ is such a section and let $j=m$ otherwise.    
    If $j\neq m$ then let 
    $\gamma_j$ be the shortest walk from a vertex $v_b$ in $\sigma_j$ to $v_n$ and let $\gamma_j=\seq{}$ otherwise.
    Again, due to the choice of $\sigma_j$, the walk $\gamma_j$ contains no vertices in $C$.

    Finally, let $\omega'=\gamma_i \append \omega(\first(\sigma_i),\last(\sigma_j)) \append \gamma_j$. The walk $\omega'$ is an open walk given $C$ such that all collider sections are anterior to $C$.
\end{proof}

The following proposition demonstrates that open walks and connecting walks yield the equivalent definitions of vertex separation in a mixed graph.

\begin{lemma}\label{lem:open-walk::connecting-walk}
    There is a connecting walk between $i$ and $j$ given $C\setminus \set{i,j}$ if and only if there is an open walk between $i$ and $j$ given $C$.
\end{lemma}

\begin{proof}
    For the forward direction, assume there is a walk $\omega$ that is a connecting walk given $C$. The walk $\omega$ is also an open walk given $C$ as every non-collider section contains no vertex in $C$ and   
    every collider section contains vertex in $C$ and thus is anterior to $\antsetij$.

    For the backward direction, assume there is a walk between $i$ and $j$ that is an open walk given $C$. From Lemma~\ref{lem:open-walk-given-C-and-possible-endpoints::open-walk-given-C-without-endpoints}, we can assume that $C$ does not contain either $i$ or $j$.
    From Lemma~\ref{lem:open-give-c:all-colliders-anterior-to-C}, we assume that the walk $\omega$ between $i$ and $j$ is open given $C$ and that every collider section on $\omega$ is anterior to $C$.
    We transform the walk $\omega$ into walk $\omega'$ by replacing every collider section $\tau$ on $\omega$
    that does not contains a vertex in $C$ 
    with a walk $\tau'$ between $\first(\tau)$ and $\last(\tau)$ that is connecting given $C$.
    The walk $\tau'$ is defined as follows:
    Let $\defwalktau$ and let $\gamma_c$ be a shortest anterior walk from a vertex on $\tau$ to a vertex $c\in C$. Let the first vertex of $\gamma'$ be $v_k$. We define $\tau'=\tau(v_1,v_k) \append \gamma_c \append \gamma_c(c,v_k) \append \tau(v_k,v_n)$.

    Next we show that $\omega'$ constructed in this way is a connecting walk given $C$.
    Consider any collider section $\tau$ on $\omega$ that does not contain a vertex in $C$.
    Let $\tau'=\tau(v_1,v_k) \append \gamma_c \append \gamma_c(c,v_k) \\ \append \tau(v_k,v_n)$ with $\gamma_c$ a shortest anterior walk from $\tau$ to a vertex $c\in C$.
    By choosing the shortest anterior walk the only vertex on $\gamma'$ in $C$ is its endpoint $c$. Thus the only vertex in $C$ on $\tau'$ is $c$ and it appears only once on the walk.
    If $\gamma_c$ is a semi-directed walk then $c$ is on a collider section of $\gamma_c\append \gamma_c(c,v_k)$ and all other sections of $\tau'$ are non-collider section on $\omega'$ due to the fact that the walk $\tau'$ is strongly out of $v_1$ and $v_n$. Furthermore, none of these sections contain a vertex in $C$. If $\gamma_c$ is an undirected walk then $\tau'$ is an undirected walk and every vertex on $\tau'$ is on a collider section that contains a vertex $c\in C$.
    
    Thus walk $\omega'$ constructed in this way is a connecting walk as every collider section contains a vertex in $C$ and no non-collider section contains a vertex in $C$.
\end{proof}

\subsection{Invertible walk decomposition functions}

The decomposition of walks into subwalks will play an important role in proving various results. In this section, we define an invertible walk decomposition. In later section, we define specific walk decompositions.

Let $\walksetgraph$ be the set of all possible walks in graph $G$.
A \emph{walk decomposition} of a walk $\omega \in \walksetgraph$ with $\defwalkomega$ is a set of walks $\decomp(\omega) = \set{\gamma_1, \ldots, \gamma_n}$ such that 
(1) $\gamma_i$ is a subwalk of $\omega$ for $1 \leq i \leq n$, and
(2) every edge $\walkedgev{i}$ on $\omega$ for $1 \leq i < n$ is on some subwalk $\gamma_j$.

A \emph{\wdf} $f$ for a set of walks $\walkset\subseteq \walksetgraph$ is a function $f(\omega)$ that maps a walk $\omega \in \walkset$ to a walk decomposition of $\omega$. If $f$ is a walk decomposition function for a set of walks $\walkset$ we denote the set of possible walk decomposition by $f(\walkset) = \bigcup_{\omega \in \walkset}{f(\omega)}$.

A walk decomposition function $f$ for $\walkset$ is \emph{invertible} if there is a function $\inverse{f}$ from $f(\walkset)$ to $\walkset$ such that for all $\omega\in \walkset$ it is the case that $\inverse{f}(f(\omega))=\omega$.
The function $\inverse{f}$ is a \emph{\wcf} for a walk function $f$.

\subsection{The maximal walk decomposition and walk equivalence}
\label{sec:walk-decomposition-and-equivalence}

The \emph{\mwd} of a walk is a walk decomposition into \mcw s and \mt s. A \emph{trek} is a walk in which all vertices are on non-collider sections.
A collider walk $\gamma$ is a \emph{\mcw} on $\omega$ if $\gamma$ is a subwalk of $\omega$ and there is no longer subwalk of $\omega$ that is a collider walk and contains $\gamma$.
A trek $\tau$ is a \emph{\mt} on $\omega$ if $\tau$ is a subwalk of $\omega$ and there is no other trek that is a subwalk of $\omega$ and contains $\tau$.
The \emph{\mwd} of $\omega$
is the walk decomposition $\maxdecompose(\omega) = \set{\tau_1, \gamma_1, \tau_2,\ldots, \gamma_{n-1},\tau_n}$ such that (i) $\gamma_i$ for $1 \leq i < n$ is a non-trivial maximal collider walk of $\omega$, and
(ii) $\tau_i$ for $1 \leq i \leq n$ is a maximal trek of $\omega$.
The \wdf\ $\maxdecompose$ is invertible due to the fact that a non-collider section of a walk must be a section of a \mt\ and a collider section of a walk must be an internal section of a \mcw. In other words, there is a function $\maxcompose$ such that $\maxcompose(\maxdecompose(\omega))=\omega$.
If a pair of walks in a \mwd\ of a walk overlaps then one of the walks is a \mt\ and the other is a \mcw\ and they share exactly one edge. 
While the first and last edges of a walk must be part of a \mt\ in a \mwd\ of a walk, they can be part of the same maximal trek if there are no collider sections in the walk.
The first and last treks of a maximal walk decomposition are called \emph{terminal} treks while  treks between two collider walks are called \emph{internal} treks.
An internal trek must be into both of its endpoints, a terminal maximal trek with an adjacent maximal collider walk is required to be into one of its endpoints, and the maximal trek of a walk consisting of a single maximal trek is not required to be into either endpoint. For example if walk $\omega$ is the walk $a \tailtail b \tailarrow c \tailtail d \arrowarrow e \arrowtail f$ then $\maxdecompose(\omega)=\set{a\tailtail b \tailarrow c, b \tailarrow c \tailtail d \arrowarrow e \arrowtail f, e \arrowtail f}$.

Next we define the equivalence maximal walk decompositions. 
Two maximal walk decompositions $\decomp(\gamma) \equiv \seq{\gamma_1,\ldots,\gamma_n}$ and $\wdecomp(\omega)=\seq{\omega_1,\ldots,\omega_m}$ are \emph{equal} 
if (1) each walk decomposition contains the same number of subwalks (i.e., $n=m$) and,
(2) for $1\leq i \leq n$, if $i$ is odd the vertex sequence of $\gamma_i$ and $\omega_i$ are the same (i.e., $\vertexseq{\gamma_i}=\vertexseq{\omega_i}$)
and if $i$ is even then $\gamma_i\equiv \omega_i$.
Two walks $\omega$ and $\gamma$ are equivalent (denoted $\omega\equiv \gamma$) if $\wdecomp(\omega)\equiv \wdecomp(\gamma)$ or $\wdecomp(\reverse(\omega))\equiv \wdecomp(\gamma)$ where $\reverse(\omega)$ is the walk with the same edges as $\omega$ but the opposite traversal order.
While two equivalent walks are guaranteed to have the same vertex sequence they need not consist of edges with the same endmarks or have the same section sequence.
Consider walks $a\tailarrow b \tailtail c$ and $a \arrowarrow b \tailarrow c$. These two walks are equivalent due to the fact that each walk decomposition consists of a single maximal trek and these maximal treks have the same vertex sequence.
The two walks, however, have different section sequences. Now consider the walks $a \tailarrow b \arrowarrow c$ and $a \arrowarrow b \arrowtail c$. These two walks are also equivalent. Their walk decompositions are of length three, their maximal treks consist of single edges with the same vertex sequence and their maximal collider walks have the same section sequence.

\subsection{Trisections}

The concept of trisections is used in many of the proofs that follow.
A walk $\omega$ is \emph{trisection} if its section sequence $\sectionseq{\omega}=\seq{\sigma_1,\sigma_2,\sigma_3}$ has three section and the first and last are of length one.
We often denote a trisection by $\seq{i,\sigma,j}$.
Trisection subwalks of walks are useful.
If $\sigma$ is an internal section on a walk $\omega$ then there is a unique trisection on $\omega$ that contains $\sigma$. This allows us to select a trisection on a walk using an internal section on that walk. A trisection $\seq{i, \sigma, j}$ is \emph{unshielded} in graph $G$ if there is no edge $\edge{i}{j}$ in $G$ and is \emph{shielded} otherwise. An important fact that is often used in proofs is that a chordless unshielded trisection $\seq{i,\sigma,j}$ whose section $\sigma$ is a collider section is a minimal inducing walk.

\subsection{Decomposing and Composing open walks}

In this section, we describe walk decompositions of open walks into open walks and the openness of the composition of open walks.
For some decompositions of an open walk into open subwalks we can guarantee termination properties of the subwalks.\footnote{The definition of the termination properties of walks is given in Section~\ref{sec:walks}.} 
The openness of a walk composed of two open walks also depends on the termination properties of the subwalks.

\begin{lemma}\label{lem:open-walk:strongly-out-of-non-collider-vertices}
    If $\omega$ is an open walk given $C$ and $b$ is an internal vertex on a non-collider section of $\omega$ then the walks $\omega(v_1,b)$ and $\omega(v_n,b)$ are open given $C$. Furthermore, at least one of the walk is strongly out of $b$.
\end{lemma}

\begin{proof}
    Let $\omega$ is a open walk given $C$ in graph $G$ and $b$ is an internal vertex on a collider section of $\omega$. It must be the case that $b\not \in C$ otherwise $\omega$ would not be an open walk given $C$. The walks $\omega(v_1,b)$ and $\omega(v_n,b)$ are both open given $C$ as every collider section on these walks is a collider section on $\omega$ and must contain a vertex in $\ant(C,G)$. Furthermore non-collider sections on either of these walk cannot contain a vertex in $C$ otherwise the walk $\omega$ would not be open given $C$. Finally, one of the walks must be strongly out of $b$ otherwise $b$ would be on a collider section of $\omega$.
\end{proof}

Next we define minimal open walks. 
The concept of a minimal open walk generalizes the concept of a minimal inducing walk and is essential to describe the compositional properties of open walks.
A walk $\gamma$ is a \emph{minimal open walk} between $i$ and $j$ given $C$ if $\gamma$ is an open walk given $C$ and there is no shorter open walk between $i$ and $j$ given $C$ that uses only vertices on $\gamma$.
One of the important properties of a minimal open walk is that the endpoints of the walk do not appear on internal non-collider sections. This property is proved in the following lemma.

\begin{lemma}\label{lem:minimal-walk:endpoints-not-internal-vertex-on-non-collider-section}
    If $\omega$ is a minimal open walk between $i$ and $j$ given $C$ then neither endpoint of $\omega$ appears as an internal vertex on a non-collider section of $\omega$.
\end{lemma}

\begin{proof}
    Let $\omega$ be a minimal open walk given $C$ with $\defwalkomega$. Suppose the lemma is not true. In this case at least one of the endpoints appears on a non-collider section of $\omega$. Without loss of generality, assume that $v_1=v_j$ for $1 < j < n$. Consider the walk $\omega'=\omega(v_j,v_n)$. The walk $\omega'$ is an open walk between $v_1$ and $v_n$ by Lemma~\ref{lem:open-walk:strongly-out-of-non-collider-vertices}. Furthermore, it is strictly shorter than $\omega$ and only uses vertices on $\omega$ as it is a subwalk on $\omega$. Thus $\omega$ is not a minimal open walk given $C$ and we have a contradiction. Thus, neither endpoint can appear as an internal vertex on a non-collider section of a minimal open walk.
\end{proof}

\begin{lemma}\label{lem:open-walk:weakly-into-collider-section}
    If $\omega$ is a minimal open walk given $C$ with $\defwalkomega$ and $\defsectionomega$ and vertex $b$ is on a collider section $\sigma_i$ then 
    
    (i) the walk $\omega(v_1,\first(\sigma_i))$ is into $\first(\sigma_i)$ and is a minimal open walk given $Cv_n\setminus \first(\sigma_i)$, and 
    
    (ii) the walk $\omega(v_n,\last(\sigma_i))$ is into $\last(\sigma_i)$ and is a minimal open walk given $Cv_1\setminus last(\sigma_i)$.
\end{lemma}

\begin{proof}
    Let $\omega$ be a minimal open walk given $C$ with $\defwalkomega$ and $\defsectionomega$ such that $b$ is on a collider section $\sigma_i$. Let $f_i=\first(\sigma_i)$.

    Proof of (i): Consider the walk $\omega'=\omega(v_1,f_i)$. The walk $\omega'$ must be into $f_i$ due to the fact that $\sigma_i$ is a collider section. 
    Next we show that $\omega'$ is open given $Cv_n \setminus f_i$. Every collider section on $\omega'$ is a collider section on $\omega$, and thus, from $\omega$ being open given $C$, we have that every collider section on $\omega'$ is anterior to $Cv_n\setminus f_i$.
    Every non-collider section on $\omega'$ is a non-collider section on $\omega$ except the last section of $\omega'$ that consists only of the vertex $b$.
    From Lemma~\ref{lem:minimal-walk:endpoints-not-internal-vertex-on-non-collider-section} and the assumption that $\omega$ is a minimal open walk, we know that $v_n$ is not on an internal non-collider section of $\omega'$. Combining this fact and the fact that $\omega$ is open given $C$, it must be the case that no non-collider section $\omega'$ contains a vertex in $Cv_n\setminus f_i$. Thus $\omega'$ is an open walk given $Cv_n\setminus f_i$.

    Proof of (ii): the proof is analogous to the proof of (i).
\end{proof}

Next we consider the openness of walks obtained by composing open subwalks.

\begin{lemma}\label{lem:compose:weakly-into}
    If walk $\omega$ between $i$ and $j$ and walk $\gamma$ between $j$ and $k$ are both minimal open walks given $C$ then $\omega \append \gamma$ is open given $Cj$ if and only if both $\omega$ and $\gamma$ are weakly into $j$.
\end{lemma}

\begin{proof}
    Let  walk $\omega$ between $i$ and $j$ and walk $\gamma$ between $j$ and $k$ be minimal open walks given $C$ that are weakly into $j$. Let $\omega'=\omega \append \gamma$.

    For the forward direction, assume $\omega \append \gamma$ is open given $Cj$.  Suppose the lemma does not hold and at least one of $\omega$ and $\gamma$ are strongly out of $j$. This implies that $j$ must be on a non-collider section on $\omega \append \gamma$ which implies that the walk is not open given $Cj$ which is a contradiction.

    For the backward direction, assume both $\omega$ and $\gamma$ are both weakly into $j$. First from the fact that both walks are open given $C$, it must be the case that $j\not\in C$.
    Furthermore, from Lemma~\ref{lem:minimal-walk:endpoints-not-internal-vertex-on-non-collider-section}, neither $\omega$ nor $\gamma$ contain $j$ on a non-collider section.
    Next, due to both walks being weakly into $j$, $j$ must be on a collider section.
    All other collider sections are anterial to some vertex in $C$ and no non-collider section can contain $C$ or $j$. Thus the walk is an open walk given $Cj$.
\end{proof}

\begin{lemma}\label{lem:compose:strongly-outof}
    If walk $\omega$ between $i$ and $j$ and walk $\gamma$ between $j$ and $k$ are both open given $C$ and at least one of $\omega$ or $\gamma$ is strongly out of $j$ then $\omega \append \gamma$ is open given $C$.
\end{lemma}

\begin{proof}
    Let $\omega$ be a walk between $i$ and $j$ open given $C$ and $\gamma$ be a walk between $j$ and $k$ given $C$ and that at least one of $\omega$ or $\gamma$ is strongly out of $j$.
    Let $\omega'=\omega \append \gamma$.
 
    Suppose that $\omega'$ is not open given $C$. 

    First, all sections of $\omega'$ are sections on $\omega$ or $\gamma$ except possibly the last section of $\omega$ and the first section of $\gamma$. Furthermore every such section is a collider section on $\omega'$ if and only if it is a collider section on $\omega$ or $\gamma$. From the fact that $\omega$ and $\gamma$ are open given $C$, all such sections are open on $\omega'$ given $C$.
    Next  the fact that $\omega$ and $\gamma$ are open given $C$ implies that no vertex in the last section of $\omega$ is in $C$ and no vertex in the first section of $\gamma$ is in $C$.
    Thus, the only way that $\omega'$ is not open given $C$ is if the section containing $j$ on $\omega'$ is a collider section. This can only happen if both $\omega$ and $\gamma$ are weakly into $j$ which is a contradiction. Thus the lemma must be true.
\end{proof}

\subsection{Properties of anterior and posterior sets}

The following proposition of anterior and posterior sets are often used without explicitly referencing the following lemmas.

The vertices \emph{posterior} to a set of vertices $A$ in a graph $G$ is the set $\post(A,G)$ containing all vertices in $A$ and all vertices $j\in V$ such that there is an anterior walk from a vertex $a\in A$ to $j$.

\begin{proposition}\label{lem:ant-property}
\[
\begin{array}{ll}
    \ant(A\cup B,G) \supseteq \ant(A,G) & (1)\\
     \ant(A\cup B,G)=\ant(A,G)  \text{\quad if\quad} B\subset \ant(A,G)& (2) \\
\end{array}
\]
\end{proposition}

\begin{proof}
    Let $G$ be a graph.

    Property~(1): For every vertex $c\in \ant(A,G)$ there must be an anterior walk from $c$ to a vertex in $A$ which implies the existence of an anterior walk from $c$ to a vertex in $A\cup B$ and thus $c\in \ant(A\cup C,G)$.

    Property~(2): Property~(1) proves one direction so we just need to show $ant(A\cup B,G) \subseteq \ant(A,G)$. For every vertex $c\in \ant(A\cup B, G)$ then there is an anterior walk $\gamma$ from $c$ to a vertex $d$ in $A\cup B$. If $d\in A$ we are done. Otherwise $d\in B$. Because $d\in B\subseteq \ant(A,G)$ there is a walk $\gamma'$ from $d$ to a vertex in $A$. The walk $\gamma \append \gamma'$ is an anterior walk from $b$ to a vertex in $A$ which proves the claim.
\end{proof}

\begin{proposition}\label{lem:post-property}
\[
\begin{array}{ll}
    \post(A\cup B,G) \supseteq \post(A,G) & (1)\\
     \post(A\cup B,G) = \post(A,G)  \text{\quad if\quad} B\subset \post(A,G)& (2) \\
\end{array}
\]
\end{proposition}

\begin{proof}    
    Let $G$ be a graph.

    Property~(1): For every vertex $c\in \post(A,G)$ there must be an anterior walk from a vertex in $A$ to $c$ which implies the existence of an anterior walk from a vertex in $A\cup B$ to $c$ and thus $c\in \post(A\cup C,G)$.

    Property~(2): Property~(1) proves one direction of the property so we just need to show $\post(A\cup B,G) \subseteq \post(A,G)$. For every vertex $c\in \post(A\cup B)$ then there is an anterior walk $\gamma$ from a vertex $d$ in $A\cup B$ to $c$. If $d\in A$ we are done. Otherwise $d\in B$. Because $d\in B\subseteq \post(a,G)$ there is a walk $\gamma'$ from a vertex in $A$ to $d$. The walk $\gamma \append \gamma'$ is an anterior walk from a vertex in $A$ to $d$ which proves the claim.
\end{proof}

The \emph{proper anterior set} of a set of vertices $C$ in graph $G$ is the set $\propant(C,G)=\ant(C,G)\setminus C$.

\begin{proposition}\label{lem:propant-property}
If $B\subseteq \propant(A,G)$ then 
    $$B \cup \propant(A\cup B,G) = \propant(A,G).$$
\end{proposition}

\begin{proof}
    Assume $B\subseteq \propant(A,G)$.
    \[
    \begin{array}{lll}
    B \cup \propant(A\cup B,G) & = B\cup \ant(A\cup B,G)\setminus (A\cup B) & \text{Definition of \ } \propant. \\  
    & = \ant(A\cup B, G)\setminus A & \text{Basic set operations.}\\
    & = \ant(A,G) \setminus A & \text{Lemma~\ref{lem:ant-property}}\\
    & = \propant(A,G) & \text{Definition of \ } \propant. \\
    \end{array}
    \]
\end{proof}

\section{Proofs related to the characterization of separable graphs and vertex separation}

In this section, we provide a proof of theorems appearing in Section~\ref{sec:separable-graphs}. In particular, we provide a proof of Theorem~\ref{thm:vertex-separable:pairwise-anterial-separable}, the characterization of vertex separation,  and a proof of Theorem~\ref{thm:separable-graph::no-self-inducing-walks}, the characterization of separable graphs.

The following lemma is the contrapositive of Proposition~\ref{prop:separable:not-adjacent}.

\begin{lemma}\label{lem:adjacent:not-vertex-separable}
    If two vertices are adjacent in a graph then they are not separable.
\end{lemma}

\begin{proof}
    Let $a\in \ant(b,G)$. This implies there is an edge $\edge{a}{b}$ in $G$ and the walk consisting of this single edge is a walk between $a$ and $b$. There is no separating set $C \subseteq V\setminus \set{a,b}$ and thus the two vertices are not be separable.
\end{proof}

\begin{lemma}\label{lem:siw:endpoints-not-separable}
    The endpoints of a self-inducing walk are not separable.
\end{lemma}

\begin{proof}
    A self-inducing walk $\omega$ with $\defwalkomega$ is an open walk given any vertex set $C\subseteq V \setminus \set{v_1,v_n}$ and thus the claim then follows from Lemma~\ref{lem:open-walk::connecting-walk}.
\end{proof}

An important anterial property of treks is that every internal vertex of treks is anterior to one of its endpoints.

\begin{lemma} \label{lem:trek:internal-vertices:anterial-walk}
    For any internal vertex on a trek, there is an anterial subwalk of the trek from the internal vertex to one of the endpoints of the trek.
\end{lemma}

\begin{proof}
    We prove the claim by induction on the number of vertices. 
    
    For the base case we consider treks of length two. In this case there are no internal vertices and the claim holds vacuously. 
    
    For the induction case, assume it is true for all treks of length $n-1$ and show for treks of length $n$. Let $\omega=\seq{e_1,\ldots,e_n}$ be a trek of length $n$ with vertex sequence $\defwalkWLIJ{\omega}{v}{1}{n+1}$. The subwalk $\omega'=\seq{e_1,\ldots,e_{n-1}}$ is a trek of length $n-1$ and thus, by the induction hypothesis, there is an anterial subwalk from every internal vertex to either $v_1$ or $v_n$. 
    If edge $e_n$ has a tail at $v_{n}$ then $v_{n}$ must be anterior to $v_{n+1}$ and the claim holds.
    If edge $e_n$ has an arrowhead at $v_n$ then no vertex $v_i$ of $\omega$ $i < n$ can be connected to $v_n$ by a semi-directed walk otherwise $\omega$ would have a collider section. This implies that either $v_n$ is connected to $v_1$ by and undirected walk or there is a semi-directed walk from $v_n$ to $v_1$. In either case, the claim holds.
    As the claim holds whether the last edge has an arrowhead or tail at the last internal vertex the claim holds of any trek of length $n$.
\end{proof}

The following lemma proves that every vertex on an open walk given $C$ is anterial to the vertex set $C$ or an endpoint of the walk.

\begin{lemma}
\label{lem:open-walk-given-C:every-vertex-in-anterior}
    Every vertex on an open walk between $i$ and $j$ given $C$ is in $\antsetCij$.
\end{lemma}

\begin{proof}
    Let $\omega$ be an open walk between $i$ and $j$ given $C$ in a graph $G$. Every vertex $b$ on $\omega$ is on a collider section or a non-collider of $\omega$.
    
    \proofcase{1} $b$ is on a collider section. From the fact that $\omega$ is open given $C$, every collider section is anterior to $\Cij$ and thus the every vertex on a collider section must be in $\antsetCij$.

    \proofcase{2} $b$ is on a non-collider section.  Every non-collider section is on a maximal trek and, from Lemma~\ref{lem:trek:internal-vertices:anterial-walk}, every vertex on a maximal trek is anterior to one of its endpoints. Furthermore, every endpoint of a maximal trek on $\omega$ is either an endpoints of the walk or a vertex on a collider section. Thus every vertex on a maximal trek must be in the set $\antsetCij$.

    In either case the vertex must be in the set $\antsetCij$ which implies the lemma is true.
\end{proof}

\reftheorem{thm:vertex-separable:pairwise-anterial-separable}

\begin{proof}
    For the backward direction, assume that $\indep{i}{j}{\propantij}{G}.$ In this case, the vertices $i$ and $j$ are separated by $\propantij$ and thus separable.

    For the forward direction, we prove the contrapositive. We assume that $\dep{i}{j}{C}{G}$ where $C=\propantij$ and prove that $i$ and $j$ are not separable.    
    From, Lemma~\ref{lem:open-walk::connecting-walk} there must be a walk $\omega$ between $i$ and $j$ that is open given $C$. We prove by cases.

    \proofcase{1} $\omega$ contains no collider sections. In this case, all internal vertices are on non-collider sections. If $\omega$ contains an internal vertex it must be on a non-collider section and in $C$ which would imply that the walk is not be open given $C$, so $\omega$ has no internal vertices. Thus $\omega$ must contain a single edge $\edge{i}{j}$ which, by Lemma~\ref{lem:adjacent:not-vertex-separable}, implies that $i$ and $j$ are not separable.

    \proofcase{2} $\omega$ contains a collider sections.
    From the fact that $\omega$ is open given $C$, every collider section must be  anterior to $\Cij.$
    Thus, if $\omega$ had an internal vertex on a non-collider section, this vertex, by Lemma~\ref{lem:open-walk-given-C:every-vertex-in-anterior}, would be in $C$ and the walk would be closed given $C.$ This implies that every internal vertex on $\omega$ is on a collider section. These facts imply that $\omega$ is a self-inducing walk which, by Lemma~\ref{lem:siw:endpoints-not-separable}, implies that $i$ and $j$ cannot be separated.

    The lemma holds in each of the exhaustive cases and thus the lemma holds.
\end{proof}

\begin{lemma}\label{lem:open-given-endpoints:either-no-edge-or-no-siw}
    If $\omega$ is a non-trivial walk between $i$ and $j$ in graph $G$ that is not a self-inducing walk then $\omega$ is closed given $C=\propantij$.
\end{lemma}

\begin{proof}
    Let graph $G$ contain a non-trivial walk $\omega$ with $\defwalkomega$ and $\defsectionomega$ such that $\omega$ is not a self-inducing walk.

    The maximal walk decomposition of $\omega$ must contain a non-trivial maximal trek because $\omega$ is not a \siw. Let $\omega(v_i,v_k)$ with $i < k$ be a non-trivial maximal trek. Note that there must be an internal vertex $v_j$ with $i < j < k$ on this maximal trek and this vertex is on a non-collider section of $\omega.$
    We consider three cases.

    \proofcase{1} Neither $v_i$ nor $v_k$ are on collider section of $\omega.$ In this case $i=1$ and $k=n$ and $\omega$ is a trek. From the fact that $\omega$ is non-trivial it contains an internal vertex on a non-collider section. This vertex must be in $C$ and thus the walk is not open given $C.$

    \proofcase{2} Either $v_i$ or $v_k$ are on collider sections of $\omega$ but not both. \Wlog, assume $v_k$ is on a collider section and thus $i=1.$ If $v_k\not \in C$ then there is a collider section on $\omega$ that is not anterior to $C$ and thus the walk is not open given $C.$ On the other hand, if $v_j\in C$ then the internal vertex $v_j$ is in $C$. This implies that there is a non-collider section of $\omega$ that contains a vertex in $C$ and the walk is not open given $C.$

    \proofcase{3} Both $v_i$ and $v_k$ are on collider sections of $\omega.$ If either $v_k\not \in C$ or $v_i\not\in C$ there is a collider section on $\omega$ that is not anterior to $C$ and thus the walk is not open given $C.$ On the other hand, if $\set{v_i,v_k} \subseteq C$ then the internal vertex $v_j$ is in $C$. This implies that there is a non-collider section of $\omega$ that contains a vertex in $C$ and the walk is not open given $C.$
\end{proof}

\begin{lemma}\label{lem:open-given-propantij::edge-or-siw}
    A walk $\omega$ between $i$ and $j$ is open given $\propantij$ in graph $G$ if and only if either (i) there is an edge $\edge{i}{j}$ in $G$, or (ii) there is a self-inducing walk between $i$ and $j$.
\end{lemma}

\begin{proof}
    This follows from Lemmas~\ref{lem:adjacent:not-vertex-separable}, \ref{lem:siw:endpoints-not-separable}, and \ref{lem:open-given-endpoints:either-no-edge-or-no-siw}.
\end{proof}

\reftheorem{thm:separable-graph::no-self-inducing-walks}

\begin{proof}
    Let $G$ be a graph.
    For the forward direction, assume that $G$ is separable. Aiming for a contradiction, suppose that $G$ contains a self-inducing walk between $i$ and $j$. This implies $i\not\in \adj(j,G)$ and, from Lemma~\ref{lem:siw:endpoints-not-separable}, that $i$ and $j$ are not separable. Thus we have a contradiction. Thus $G$ must not contain a self-inducing walk.

    For the reverse direction, assume that $G$ contains no self-inducing walk. Aiming for a contradiction, assume that $G$ is not separable. This implies that there is a pair of vertices $i\not\in \adj(j,G)$ such that $i$ and $j$ are not separable. This implies that $\dep{i}{j}{\propantij}{G}$ which implies that there is a walk $\omega$ between $i$ and $j$ that is open given $\propantij.$ From Lemma~\ref{lem:open-given-propantij::edge-or-siw} and $i\not\in \adj(j,G)$, however, this walk $\omega$ must be a self-inducing walk. This is a contradiction as $G$ does not contain any self-inducing walks and the lemma must hold.
\end{proof}

\section{Proofs related to essentially acyclic and anterial graphs}
In this section we prove results from Section~\ref{sec:essentially-acyclic}. In addition, we prove a set of useful properties of \miw s and shortest open walks that are used in this section and in later sections of the appendix.

\subsection{Characterization of \miw s}

\begin{lemma}\label{lem:minimal-inducing-walk:characterization}
    An inducing walk in a graph is minimal if and only if any chord on the walk is (i) between two internal vertices and is directed, or (ii)  between an internal vertex and an endpoint of the walk and is either undirected or directed with a tail at the internal vertex.
\end{lemma}

\begin{proof}
    Let $\omega$ be an inducing walk in graph $G$ with vertex sequence $\defwalkomega$ and let $\edge{v_i}{v_j}$ for $j > i+1$ be a chord of the inducing walk.
    
    Assume $\omega$ is minimal but that the chord does not satisfy condition (i) or (ii). We consider the possible endpoint for $\edge{v_i}{v_j}$. First note that there can be no edge between the endpoints $v_1$ and $v_n$ as in this case $\omega$ would not be an inducing walk. If $v_i$ and $v_j$ are both two internal vertices of $\omega$ then we are in case (i). If there is either an undirected edge or a bidirected edge, however, then the walk $\omega(v_1,v_i)\append \seq{\edge{v_i}{v_j}} \append \omega(v_j,v_n)$ would be an inducing walk that uses a subset of the vertices in $\omega$ and thus $\omega$ is not minimal. This is a contradiction and thus there is no chord violating condition (i).
    If one of $v_i$ and $v_j$ is an endpoint and the other is internal then we are in case (ii). If $v_i=v_1$ and there is an arrowhead at $v_j$ on chord $\edge{v_i}{v_j}$ then
    $\seq{\edge{v_1=v_i}{v_j}} \append \omega(v_j,v_n)$
    would be an inducing walk that uses a subset of the vertices in $\omega$ and thus $\omega$ is not minimal. If $v_j=v_n$ and there is an arrowhead at $v_i$ on chord $\edge{v_i}{v_j}$ then $\omega(v_1,v_i)\append \seq{\edge{v_i}{v_n=v_j}}$ would be an inducing walk that uses a subset of the vertices in $\omega$. In either case, $\omega$ is not minimal. This is a contradiction and there is no chord violating condition (ii).

    For the other direction, assume $\omega$ is an inducing walk such that all of its chords satisfy conditions (i) and (ii). Assume that there is an inducing walk $\omega'$ between $v_1$ and $v_n$ that uses a subset of vertices in $\omega$. In this case, there must be a chord of $\omega$ that is on $\omega'$. Let $\edge{v_i}{v_j}$ with $i < j$ be the first chord, an edge on $\omega'$ not on $\omega$. If the first edge of $\omega'$ is a chord it must be $v_1=v_i$ and $v_j \neq v_n$ as such a chord does not satisfy either (i) or (ii). In this case, the edge has a tail at $v_j$ and cannot be on a collider section on $\omega'$. This implies that $\omega'$ is not an \iw. If the edge is not the first edge then the vertex $v_i$ must be an internal vertex on $\omega'$. If $v_j=v_n$ then there must be a tail at $v_i$ on the edge and $v_i$ is not on a collider section of $\omega'$. Again $\omega'$ is not an \iw. Finally, if both $v_i$ and $v_j$ are internal vertices on $\omega'$ then the chord must be directed and thus both $v_i$ and $v_j$ cannot be a collider section of $\omega'$ and thus it is not an \iw. In each of these cases, we have a contradiction as $\omega'$ is assumped to be an \iw\ and thus no such $\omega'$ exists and $\omega$ must be a \miw.
\end{proof}

\subsection{The \miw\ decomposition}

In this section, we decompose \miw s that are not \siw s into \mdiw s.
In order to define \mdiw s, we need a few additional definitions.
A section on an inducing walk is \emph{discriminated} if the section is not anterior to either endpoint of the inducing walk. 
An inducing walk \emph{discriminates} a collider section if the section is discriminated by the walk.
An inducing walk is a \emph{(simple) \diw} if it contains exactly one discriminated section and a \emph{compound \diw} if it contains more than one discriminated section.
Note that an inducing walk containing no discriminated sections is a \emph{\siw}.

A walk decomposition is a \emph{\miw\ decomposition} if every walk in the decomposition is a \mdiw.
Finally, a \emph{\miw\ decomposition function} is a walk decomposition function that decomposes a \miw\ that is not a \siw\ into a set of \mdiw s.

A recursive \miw\ decomposition function $\iwdecompose$ is given in Algorithm~\ref{alg:iwdecomp}. We next prove essential properties and the correctness of the algorithm.

\begin{minipage}{.8\linewidth}
\begin{algorithm}[H]
\caption{The \iwdecompose\ algorithm}\label{alg:iwdecomp}
\begin{algorithmic}
\vspace{0.03in}
\State {\bf Input:} A \miw\ $\omega$  with $\defwalkomega$ and $\defsectionomega$ and a graph $G=\seq{V,E}$ containing $\omega$.
\State {\bf Output:} A minimal inducing walk decomposition of $\omega$.
\vspace{-0.05in}
\\ \!\!\!\!\!\!\!\! \hrulefill
\vspace{0.05in}

\If {$\omega$ is not a \miw\ or is a self-inducing walk}
    \State {\Return $\emptyset$}
\ElsIf {$\omega$ contains only one discriminated section $\sigma_i$}
    \State $walks = \set{\omega}$
    \If {$\omega(v_1,\first(\sigma_i))$ is an inducing walk}
        \State $walks = walks \cup \iwdecompose(\omega(v_1, \first(\sigma_i)))$
    \EndIf
    \If {$\omega(\last(\sigma_i),v_n)$ is an inducing walk}
        \State $walks = walks \cup \iwdecompose(\omega(\last(\sigma_i),v_n))$
    \EndIf
    \State {\Return $walks$}
\Else
    \State {Let $\sigma_i$ and $\sigma_j$ with $i < j$ be two discriminated sections on $\omega$}.    
    \State \Return $\iwdecompose(\omega(v_1, \first(\sigma_j)) \cup \iwdecompose(\omega(\last(\sigma_i), v_n))$
\EndIf

\end{algorithmic}
\end{algorithm}
\end{minipage}
\vspace{0.2in}

The first lemma provides sufficient conditions under which a \miw\ contains another \miw\ as a subwalk.

\begin{lemma}\label{lem:decomposing-minimal-inducing-walks}
    If $\omega$ is a minimal inducing walk with $\defwalkomega$ and $\defsectionomega$ in graph $G$ then 
    
    (i) if $\first(\sigma_i,G)\not\in \adj(v_1,G)$ then $\omega(v_1,\first(\sigma_i,G))$ is a minimal inducing walk, and

    (ii) if $\last(\sigma_i,G)\not\in \adj(v_n,G)$ then $\omega(\last(\sigma_i,G), v_n)$ is a minimal inducing walk.
\end{lemma}

\begin{proof}
    Let $\omega$ be a minimal inducing walk with $\defwalkomega$ and $\defsectionomega$ in graph $G$.
    First we prove (i). Assume  $\first(\sigma_i,G) \not\in \adj(v_1,G)$.
    This implies that $\omega'= \omega(v_1,\first(\sigma_i,G))$ is an inducing walk. The walk $\omega'$ is a minimal inducing walk, by Lemma~\ref{lem:minimal-inducing-walk:characterization} and the fact that $\omega$ is a minimal inducing walk.
    The proof for (ii) is symmetric to the proof of (i). 
\end{proof}

If a \miw\ contains a discriminated section it must obey certain adjacency properties that are captured in the next lemma.

\begin{lemma}\label{lem:miw:no-chord-between-endpoints-and-discriminated-section}
    If a \miw\ in graph discriminates a collider section then there is no chord between the endpoints of the \miw\ and a vertex on the discriminated section.
\end{lemma}

\begin{proof}
    Let $G$ be a graph containing a \miw\ $\omega$ with $\defwalkomega$ and $\defsectionomega$ that discriminates section $\sigma_i$.
    Assume the lemma does not hold. Then there must be a chord $\edgevivj$ such $i=1$ or $i=n-1$ and $v_j$ is on $\sigma_i$. \Wlog\ assume $i = 1$.
    
    \proofcase{1} $\edgevivj$ has an arrowhead at $v_j$. 
    In this case, the walk $\seq{\edgevivj}\append \omega(v_j,v_n)$ is an inducing walk that uses a subset of the vertices on $\omega$ and thus $\omega$ is not a \miw. This is a contradiction.

    \proofcase{2} $\edgevivj$ has a tail at $v_j$.
    In this case, the section is not discriminated by $\omega$ which is a contradiction.

    In either case we have a contradiction and the lemma must hold.
\end{proof}

\cut{
\begin{lemma}\label{lem:miw:vertices-on-discriminated-section-not-adjacent-to-endpoints-old}
    If a \miw\ $\omega$ in graph $G$ with $\defwalkomega$ and $\defsectionomega$ discriminates section $\sigma_a=\seq{v_i,\ldots,v_k}$ with $i < k$ then for all $v_j$ on $\sigma_a$ it is the case that
    (1) if $j > 2$ then $v_j \not \in \adj(v_1,G)$, and
    (2) if $j < n-1$ then $v_j \not \in \adj(v_n,G)$.
\end{lemma}

\begin{proof}
    Let $\omega$ be a \miw\ in graph $G$ with $\defwalkomega$ and $\defsectionomega$ that discriminates section $\sigma_a=\seq{v_i, \ldots, v_k}$ with $i < k$.
    
    Suppose that the lemma is not true. In this case, either (1) or (2) does not hold.
    
    \proofcase{1} Property (1) does not hold. There is an edge $\edge{v_j}{v_1}$ with $j > 2$. If the edge has a tail at $v_j$ then the section is not discriminated and if the edge has an arrowhead at $v_j$ then the walk is not a \miw. Thus in either case, we have a contradiction.
    
    \proofcase{2} Property (2) does not hold.  There is an edge $\edge{v_j}{v_n}$ with $j > 2$. If the edge has a tail at $v_j$ then the section is not discriminated and if the edge has an arrowhead at $v_j$ then the walk is not a \miw. Thus in either case, we have a contradiction.

    Thus in either case we have a contradiction and thus the lemma must be true.
\end{proof}
}

The following lemma allows us to decompose a \miw\ that contains two or more discriminated sections into shorter \miw s.

\begin{lemma}\label{lem:decomposition-of-miw-with-two-discriminated-sections}
    If $\omega$ with $\defwalkomega$ and $\defsectionomega$ discriminates two collider sections $\sigma_i$ and $\sigma_j$ with $1 < i < j < n$ then the subwalks $\omega(v_1, \first(\sigma_j))$ and $\omega(\last(\sigma_i), v_n)$ are \miw s containing discriminated sections.
\end{lemma}

\begin{proof}
    Let $\omega$ be a minimal inducing walk in graph $G$ with $\defwalkomega$ and $\defsectionomega$ that discriminates two collider sections $\sigma_i$ and $\sigma_j$ with $1 < i < j < n$. 
    Let $\omega_1 = \omega(v_1, \first(\sigma_j))$ and
    $\omega_n = \omega(\last(\sigma_i), v_n)$.
    Both $\omega_1$ and $\omega_n$ are collider walks. From Lemma~\ref{lem:miw:no-chord-between-endpoints-and-discriminated-section} both are inducing walks. Finally, by Lemma~\ref{lem:minimal-inducing-walk:characterization} and the fact that $\omega$ is a minimal inducing walk, both are minimal inducing walks. 
\end{proof}

\begin{lemma}\label{lem:decompmin-invertible}
The function $\iwdecompose$ is an invertible \miw\ decomposition function.
\end{lemma}

\begin{proof}
    If $\omega$ is a \miw\ that is not a \siw\ then it must have a discriminated section. The correctness of the walk decomposition function follows from the fact that each added walk is a \diw\ and Lemma~\ref{lem:decomposition-of-miw-with-two-discriminated-sections}, which allows us to decompose a walk into minimal inducing walks with fewer discriminated sections. The walk decomposition function $\iwdecompose$ is invertible due to the fact that no internal section on a \miw\ can be repeated.
\end{proof}

\begin{lemma}\label{lem:separable:miw:every-internal-section-discriminated}
If $\omega$ is a \miw\ in a separable graph $G$ then every collider section on $\omega$ is discriminated by some \diw\ in $\iwdecompose(\omega, G)$.
\end{lemma}

\begin{proof}
We prove by induction on the number of collider section on the inducing walk.
Let $\omega$ be a \miw\ that is not a \siw\ in a graph $G$. It must have a discriminated section $\sigma$.

For the base case, $\sigma$ is the only collider section of $\omega$.
In this case, the collider section $\sigma$ must be discriminated by $\omega$ and we are done.

For the induction case, we assume it is true for walks with $n$ collider sections and show for $n+1$ collider sections.
First, the collider section $\sigma$ is discriminated by $\omega$. 
If there is a collider section $\sigma'$ on $\omega$ closer to $v_1$ than $\sigma$ then the walk $\omega(v_1, \first(\sigma))$ is a collider walk. By Lemma~\ref{lem:miw:no-chord-between-endpoints-and-discriminated-section} and the fact that $\omega$ is a \miw, the subwalk is also a \miw. The walk has at most $n$ collider sections and thus, by the induction hypothesis, the walk discriminates $\sigma'$.
If there is a collider section $\sigma'$ on $\omega$ closer to $v_n$ than $\sigma$ then the walk $\omega(\last(\sigma), v_n)$ is a collider walk. By Lemma~\ref{lem:miw:no-chord-between-endpoints-and-discriminated-section} and the fact that $\omega$ is a \miw, the subwalk is a \miw. The walk has at most $n$ collider sections and thus, by the induction hypothesis, the walk discriminate $\sigma'$.

Thus, every collider section of $\omega$ is discriminated by some \diw\ that is a subwalk of $\omega$.
\end{proof}

\subsection{Minimal inducing walks in separable graphs}

\begin{lemma}\label{lem:non-discriminated-section-of-miw:not-anterior-to-discriminated}
    If walk $\omega$ with $\defwalkomega$ and $\defsectionomega$ is a \diw\ in graph $G$ that discriminates $\sigma_i$ then $\sigma_i$ is not anterior to $\sigma_j$ for $j\neq i$.
\end{lemma}

\begin{proof}
    Let walk $\omega$ with $\defwalkomega$ and $\defsectionomega$ be a \diw\ in graph $G$ that discriminates $\sigma_i$. Every collider section other than $\sigma_i$ is anterior to $v_1$ or $v_n$. Suppose the lemma does not hold and $\sigma_i$ is anterior to $\sigma_j$. In this case $\sigma_i$ is anterior to $v_1$ or $v_n$ and thus $\sigma_i$ is not discriminated by $\omega$. This is a contradiction so the lemma must hold.
\end{proof}

\begin{lemma}\label{lem:separable-graph:minimal-inducing-walk:internal-not-anterial}
    In a separable graph, no internal section on a minimal inducing walk can be anterior to an adjacent section.
\end{lemma}

\begin{proof}
    Let $G$ be a separable graph. Assume the lemma is not true and that $\omega$ is a \miw\ that contains a section $\sigma_i$ anterior to an adjacent section. From Lemma~\ref{lem:separable:miw:every-internal-section-discriminated}, there is a \diw\ $\omega'$ that discriminates $\sigma$. Result then follows from Lemma~\ref{lem:non-discriminated-section-of-miw:not-anterior-to-discriminated}.
\end{proof}

\cut{
\begin{proof}
    Suppose that the lemma is not true. 
    Let $G$ be a separable graph and walk $\omega$ with $\defwalkomega$ and $\defsectionomega$ be a shortest inducing walk in $G$ such that there is an internal section $\sigma_j$ that is anterior to an adjacent section. Without loss of generality assume that $\sigma_j \in \subset \ant(\sigma_{j-1}, G)$ and $j < m$. 
    
    There must be an internal section $\sigma_i$ on $\omega$ such that $\sigma_i \not \subseteq \antomegaendpoints$, otherwise the walk is a self-inducing walk which, by Theorem~\ref{thm:separable-graph::no-self-inducing-walks}, contradicts the assumption that $G$ is separable.

    Again, without loss of generality assume that $i > j$.
    Due to the minimality of $\omega$ and 
    $\sigma_i \not \subseteq \antomegaendpoints$, it must be the case
    that $\first(\sigma_i) \not \in \adj(v_1,G)$.
    In this case, by Lemma~\ref{lem:decomposing-minimal-inducing-walks}, the walk $\omega(v_1, \first(\sigma))$ is a shorter minimal inducing walk which contradicts the selection criterion for $\omega$.

    Thus the lemma must be true.
\end{proof}
}

\begin{lemma}\label{lem:separable-graph:minimal-inducing-walk:no-chord-between-adjacent-sections}
    In a separable graph, a minimal inducing walk does not have a chord between adjacent sections. 
\end{lemma}

\begin{proof}
    Let $G$ be a separable graph and $\omega$ be a minimal inducing walk with $\defwalkomega$ and $\defsectionomega$ with a chord between vertex $v_i$ on section $\sigma_a$ and vertex $v_j$ on $\sigma_{a+1}$.

    \proofcase{1} $1 < a < m - 1$. 
    In this case, the chord is between two internal sections. By Lemma~\ref{lem:minimal-inducing-walk:characterization} and minimality, the chord must be directed. Without loss of generality,
    assume $v_i \tailarrow v_j$. This, however, contradicts Lemma~\ref{lem:separable-graph:minimal-inducing-walk:internal-not-anterial}.

    \proofcase{2} $a=1 \vee a = m - 1$.
    In this case, the chord is between an internal section and its adjacent terminal section. By Lemma~\ref{lem:minimal-inducing-walk:characterization} and minimality, the chord must have a tail at the internal vertex. This, however,  contradicts Lemma~\ref{lem:separable-graph:minimal-inducing-walk:internal-not-anterial}.

    In either case we have a contradiction and the lemma must hold.
\end{proof}

\begin{lemma} \label{lem:separable-graph:minimal-inducing-walk:no-chord-in-collider-section}
    In a separable graph, there is no chord on a collider section of a minimal inducing walk.
\end{lemma}

\begin{proof}
    Suppose the lemma is not true. Let $G$ be a separable graph contain a walk $\omega$ with $\defwalkomega$ and $\defsectionomega$ such that there is a chord on collider section $\sigma_a=\seqvIJ{i}{l}$. Let chord $\edgevIJ{j}{k}$ on $\sigma_a$ be the chord that maximizes $|k-j|$ and $i \leq j < k \leq l$.
    
    \proofcase{1} $\edgevIJ{j}{k}$ is undirected or bidirected. In this case the walk $\omega(v_1,v_j)\append \edgewalk{v_j}{v_k} \append \\ \omega(v_k,v_n)$ is an inducing walk that uses a subset of the vertices in $\omega$ and thus $\omega$ is not a minimal inducing walk. This is a contradiction.

    \proofcase{2} $\edgevIJ{j}{k}$ is directed. Without loss of generality assume that $\edgevIJ{j}{k}$ has an arrowhead at $v_k$. In this case, the walk $\edgewalk{v_j}{v_k}\append \omega(v_k, \first(\sigma_{a+1}))$ is a collider trisection. Furthermore, by Lemma~\ref{lem:separable-graph:minimal-inducing-walk:no-chord-between-adjacent-sections} the walk is unshielded and thus is a minimal inducing walk. Due to the fact that $\sigma_a$ is an undirected walk, the collider trisection is also a self-inducing walk which, by Theorem~\ref{thm:separable-graph::no-self-inducing-walks}, contradicts the assumption that $G$ is a separable graph.

    Both cases lead to contradictions so the lemma must hold. 
\end{proof}

\begin{lemma}\label{lem:separable-graph:min-ind-walk:no-chord-on-semi-directed-cycle}
    In a separable graph, a \miw\ has no chord between internal sections that are part of a semi-directed circuit.
\end{lemma}

\begin{proof}
    Assume the lemma is not true and let $G$ be a separable graph containing
    a \miw\ with a chord between internal sections that are part of a semi-directed circuit in $G$.
    Let $\omega$ with $\defwalkomega$ and $\defsectionomega$ be a shortest \miw\ in $G$ with a chord between internal sections that are part of a semi-directed circuit in $G$.
    The chord must be directed otherwise $\omega$ would not be a \miw.
    Let chord $\edgevivj$ of $\omega$ be the chord between internal sections and on a semi-directed circuit with the largest jump, that is, such that $|j-i|$ is as large as possible.
    \Wlog, assume that $i<j$ and that edge $\edgevivj$ has an arrowhead at $v_j$. Let $v_i$ be on $\sigma_s$ and $v_j$ be on $\sigma_u$. It must be the case that $1 < s < u < m$.
    Let $l_i = \last(\sigma_i)$ and $f_i=\first(\sigma_i)$.
    In order for $v_1$ and $v_n$ to be separable there must be a discriminated section $\sigma_t$ on $\omega$ otherwise $\omega$ is a self-inducing walk. 

    \proofcase {1} Section $\sigma_t$ is not after $\sigma_s$ on $\omega$,  that is, $1< t \leq s$.
    We consider two subcases. First, $l_t\not\in \adj(v_n,G)$. In this case, then the walk $\omega(l_t, v_n)$ is a \miw\ containing a chord between internal sections that is on a semi-directed circuit that is shorter than $\omega$. This is a contradiction. Second, $l_t\in \adj(v_n,G)$. In this case, by Lemma~\ref{lem:minimal-inducing-walk:characterization}, the edge $\edge{l_t}{v_n}$ must have a tail which implies that $\sigma_t$ is not discriminated. Again we have a contradiction.

    \proofcase{2} Section $\sigma_t$ is not before $\sigma_u$ on $\omega$, that is, $u \leq t < m$. We again consider two subcases.
    First, $f_t\not\in \adj(v_1,G)$. In this case, the walk $\omega(v_1, f_t)$ is a \miw\ containing a chord between internal sections that is on a semi-directed circuit that is shorter than $\omega$. This is a contradiction. Second, $f_t\in \adj(v_1,G)$. In this case, by Lemma~\ref{lem:minimal-inducing-walk:characterization}, the edge $\edge{f_t}{v_1}$ must have a tail at $f_t$ which implies $\sigma_t$ is not a discriminated section. Again we have a contradiction.

    \proofcase{3} Section $\sigma_t$ is between $\sigma_s$ and $\sigma_u$ on $\omega$. If there are multiple discriminated section between $\sigma_s$ and $\sigma_u$ then choose $\sigma_t$ to be the closest to $\sigma_u$. We again consider two subcases.
    First, $l_t \in \adj(v_n,G)$. In this case, the edge
    $\edge{l_t}{v_n}$ must have a tail at $l_t$ by Lemma~\ref{lem:minimal-inducing-walk:characterization} which implies $\sigma_t$ is not a discriminated section. This is a contradiction. Second, $l_t \not\in \adj(v_n,G))$. In this case, the walk $\omega(l_t,v_n)$ must be a self-inducing walk as every collider section between $\sigma_t$ and $\sigma_n$ is not discriminated. This, by Theorem~\ref{thm:separable-graph::no-self-inducing-walks}, contradicts the assumption that $G$ is separable.

    In each case we have a contradiction. Thus, a minimal inducing walk in a separable graph can have no chord between internal sections that is part of a semi-directed circuit.
\end{proof}

\begin{lemma}\label{lem:separable-graph:MIW-edge:exclusive-endmark-on-internal-endpoint}
    In a separable graph, any edge on a minimal inducing walk has an exclusive endmark at any endpoint that is internal to the minimal inducing walk.
\end{lemma}

\begin{proof}
    Suppose the lemma is not true.
    Let $G$ be a separable graph that contains $\omega$ minimal inducing walk $\omega$ with $\defwalkomega$ such that not all of the internal endmarks on $\omega$ are exclusive. Without loss of generality assume that the edge is $\edgevIJ{i-1}{i}$ for $1 < i \leq n$ and the edge does not have an exclusive endpoint at $v_i$.

    \proofcase{1} $i=2$. In this case, the there must be an $\edgeprime{v_1}{v_2}$ with an tail at $v_2$. This implies that $\sigma_2 \subseteq \ant(\sigma_1,G)$ which contradicts Lemma~\ref{lem:separable-graph:minimal-inducing-walk:internal-not-anterial}.

    \proofcase{2} $i > 2$ and $\edgevIJ{i-1}{i}$ is bidirected.    
    In this case, the there must be an $\edgeprime{v_{i-1}}{v_i}$ with a tail at $v_i$.
    Because the edge is bidirected, the two endpoints are on adjacent sections of $\omega$. These sections are both collider sections on $\omega$. Thus there is an internal section anterial to an adjacent section which  
    contradicts Lemma~\ref{lem:separable-graph:minimal-inducing-walk:internal-not-anterial}.

    \proofcase{3} $i > 2$ and $\edgevIJ{i-1}{i}$ is undirected. In this case, the there must be an $\edgeprime{v_{i-1}}{v_i}$ with an arrowhead at $v_i$.
    Because the edge is undirected, the two endpoints are on the same section $\sigma_a$ of $\omega$. This implies that the walk $\seq{\edgeprime{v_{i-1}}{v_i}}\append \omega(v_i,\first(\sigma_{a+1}))$ is a collider trisection. Furthermore, by Lemma~\ref{lem:separable-graph:minimal-inducing-walk:no-chord-between-adjacent-sections}, the trisection is unshielded an thus minimal inducing walk. This implies, however, that the trisection is a self-inducing walk which contradicts Theorem~\ref{thm:separable-graph::no-self-inducing-walks}.

    In all three cases we have a contradiction and thus the lemma must hold.

\end{proof}

\subsection{Properties of shortest open walks }

Recall that a walk $\omega$ is a \emph{shortest open walk} given $C$ if $\omega$ is open given $C$ and there is no shorter walk that is open given $C$.

\begin{lemma}\label{lem:shortest:no-chord-between-distinct-maximal-treks}
Every shortest open walk has no chords between vertices of non-collider sections on distinct maximal treks.
\end{lemma}

\begin{proof}
    Let $\omega$ be a shortest open walk between $a$ and $b$ given conditioning set $C$ and
    the walk decomposition of $\omega$ be
    $\seq{\tau_1,\gamma_1,\tau_2,\ldots, \gamma_{n-1},\tau_n}$  with maximal treks $\tau_i$ open given $C$ and maximal collider walks $\gamma_i$ open given $C$.
    
    Assume that the lemma is not true. Let maximal treks 
    $\tau_k$ with $\vertexseq{\tau_k}=\seq{a_1,\ldots, a_m}$ and $\tau_l$ with $\vertexseq{\tau_l}=\seq{b_1,\ldots, b_n}$ be such that there is a chord $\edge{a_i}{b_j}$. Let walk $\omega'=\omega(a_1,a_i) \append \edge{a_i}{b_j} \append \omega(b_j,j)$ be the walk obtained by replacing the subwalk from $a_i$ to $b_j$ in $\omega$ by the chord $\edge{a_i}{b_j}$. 
    We show that for any chord between $a_i$ and $b_j$ the walk $\omega'$ is a shorter open walk than $\omega$ which is a  contradiction.

    If the edge $\edge{a_i}{b_j}$ is undirected then there is a section $\sigma$ on $\omega'$ that contains both $a_i$ and $b_j$. Furthermore, $\sigma$ must be a non-collider on $\omega'$ given $C$ and cannot contain a member of $C$ as all vertices on $\sigma$ are on non-collider sections of $\omega$. In this case, $\omega'$ is a shorter open walk than $\omega$ which is a contradiction.

    If the edge $\edge{a_i}{b_j}$ is directed, we assume, without loss of generality that it is oriented $a_i \tailarrow b_j$. In this, case the section $\sigma$ of $\omega'$ containing $a_i$ must be an open non-collider section given $C$. The section $\sigma'$ of $\omega'$ containing $b_j$ need not be a non-collider section. In the case that it is a non-collider section it is open given $C$ and $\omega'$ is a shorter open walk than $\omega$ which is a contradiction. Next we consider the case in which $\sigma'$ is a collider section on $\omega'$. In this case, the section containing $b_j$ on $\omega$ is a non-collider section which, by Lemma~\ref{lem:open-walk-given-C:every-vertex-in-anterior}, implies that there is an anterior walk $\gamma=\seq{a_1=b_j,\ldots, a_n}$ from $b_j$ to some vertex $a_n\in C \cup \set{i,j}$. In this case, $\omega'$ must be a shorter open walk than $\omega$, which is a contradiction.

    If the edge $\edge{a_i}{b_j}$ is bidirected, $a_i$ and $b_j$ are in separate sections on $\omega'$. As in the case above, if the sections are non-colliders they must be open and if they are colliders, they must be anterior to $C\cup\set{i,j}$. Thus, in any combination, $\omega'$ is a shorter open walk than $\omega$, which is a contradiction.

    Hence, there can be no chord  between vertices of non-collider sections on distinct maximal treks.
\end{proof}

\begin{lemma}\label{lem:shortest:mcw:min-inducing-walk}
    Every maximal collider walk in the walk decomposition of a shortest open walk is a minimal inducing walk.
\end{lemma}

\begin{proof}
    Let $\gamma$ be a maximal collider walk on a shortest open walk $\omega$. Let $\tau$ and $\tau'$ be the adjacent maximal treks adjacent to $\gamma$ on $\omega$. The endpoints $\gamma$ are both on non-collider section of distinct maximal treks of $\omega$. From Lemma~\ref{lem:shortest:no-chord-between-distinct-maximal-treks}, there can be no chord between the endpoints and thus $\gamma$ must be an inducing walk. Furthermore, the walk $\gamma$ must be minimal or there would be a shorter open walk than $\omega$. 
\end{proof}

\begin{lemma}\label{lem:shortest:vertex-on-non-collider-section:only-on-one-non-collider-section}
    A vertex on a non-collider section of a shortest open walk given $C$ appears only in one non-collider section on the walk.
\end{lemma}

\begin{proof}
    Suppose the lemma is not true. Let $G$ be a graph that contains a walk $\omega$ with $\defwalkomega$ and $\defsectionomega$ that is open given $C$  such that there are vertices $v_i$ on non-collider section $\sigma_h$ and $v_k$ on non-collider section $\sigma_j$ on $\omega$ such that $v_i = v_k$. \Wlog\ assume $i < k$ which implies $h \leq j$.
    The walk $\omega(v_1,v_i) \append \omega(v_j,v_n)$ is a shorter open walk given $C$ and thus we have a contradiction and the lemma must be true.
\end{proof}

\subsection{Proofs of Theorems~\ref{thm:weakly-separable::essentially-acyclic} and \ref{thm:simplify:separable-graph:equivalent-anterial-graph}}

In order to prove the theorems we prove a series of lemmas relating to the \simplify\ algorithm.

\begin{lemma}\label{lem:simplify:simple}
    The graph $\simplify(G)$ is a simple graph.
\end{lemma}

\begin{proof}
    Let $G$ be a graph and let $H=\simplify(G)$. If there is no edge between $a$ and $b$ in $G$ then there is no edge between $a$ and $b$ in $H$. If there are one or more edges between $i$ and $j$ in $G$ then only one edge is added to $H$ because the choice of edge endmarks is uniquely determined by the existence of anterior walks between $i$ and $j$ in $G$.
\end{proof}

A graph $G$ satisfies the \emph{anterial endmark property} if the endmarks of every edge $\edgeab$ in $G$ is such that there is a tail at $a$ on $\edgeab$ if and only if there is an anterior walk from $a$ to $b$ in $G$.

\begin{lemma}\label{lem:simple:anterial-graph::anterial-endmark}
    A simple graph $G$ is an anterial graph if and only if the graph satisfies the anterial endmark property. 
\end{lemma}

\begin{proof}
    For the forward direction, suppose $G$ is an anterial graph.
    Assume that the anterial endmark property does not hold. In this case there is an edge $\edgeab$ with an arrowhead at $b$ and there is an anterior walk $\omega$ from $b$ to $a$. We consider two cases. First, $\edgeab$ has an arrowhead at $a$. In this case, $a \arrowarrow b$ and $b \in \ant(a,G)$ which imply that $G$ is not anterial. Second, $\edgeab$ has a tail at $a$. In this case, the walk $\omega \append \seq{\edgeab}$ is a semi-directed circuit and again $G$ is not anterial. In either case we have a contradiction so the forward direction holds.
    
    For the backward direction, let $G$ be a simple graph and assume that the graph satisfies the anterial endmark property but is not anterial. We consider two cases. First there is an edge $a \arrowarrow b$ and an anterior walk $\omega$ from $b$ to $a$ in $G$. In this case, the graph $G$ does not satisfy the anterial endmark property. Second there is an edge $a \tailarrow b$ and an anterior walk $\omega$ from $b$ to $a$ in $G$. In this case, the graph $G$ does not satisfy the anterial endmark property. In either case we have a contradiction and the backward direction holds.    
\end{proof}

\begin{lemma}\label{lem:simplify:anterial-graph}
    If $G$ is a graph then $\simplify(G)$ is an anterial graph.
\end{lemma}

\begin{proof}
    From Lemma~\ref{lem:simplify:simple}, the graph $\simplify(G)$ is simple.
    Graph $\simplify(G)$ satisfies the anterial endmark property because every edge $\edgeab$ added to the the graph $\simplify(G)$ has a tail at $a$ if and only if there is an anterior walk from $a$ to $b$.
    Thus, by Lemma~\ref{lem:simple:anterial-graph::anterial-endmark}, the graph $\simplify(G)$ is an anterial graph.
\end{proof}

\begin{lemma}\label{lem:simplify:same-adjacencies}
    The graph $G$ and $\simplify(G)$ have the same adjacencies.
\end{lemma}

\begin{proof}
    Let $G$ be a graph and let $H=\simplify(G)$. 
    The \simplify\ algorithm results in a graph $\simplify(G)$ with the same adjacencies as the input graph $G$ due to the fact that for every edge $\edgeab$ in $G$ some edge $\edgeprime{a}{b}$ is added to $H$ and if there is no edge $\edgeab$ in $G$ then no edge is added between $a$ and $b$ in $H$.
\end{proof}

\begin{lemma}\label{lem:simplify:tail:no-arrowhead}
    If there is an edge $\edgeab$ in $G$ that has a tail at $a$ then
    there is no edge $\edgeprime{a}{b}$ in graph $\simplify(G)$ with an arrowhead at $a$.
\end{lemma}

\begin{proof}
    Let $G$ be a graph containing an $\edgeab$ with a tail at $a$. In this case $a\in \ant(b,G)$. This implies that the $\edgeprime{a}{b}$ in $\simplify(G)$ must be either $a \tailtail b$ or $a \tailarrow b$.
\end{proof}

\begin{lemma}\label{lem:anterial-walk-in-graph::anterial-walk-in-anterialize}
    There is an anterior walk from $v_1$ to $v_n$ in graph $G$ if and only if there is an anterior walk from $v_1$ to $v_n$ in $\simplify(G)$.
\end{lemma}

\begin{proof}
For the first direction, let $\gamma$ be an anterior walk from $v_1$ to $v_n$ in $G$. From Lemma~\ref{lem:simplify:same-adjacencies}, we know that there is a walk $\omega$ in $\simplify(G)$ with the same vertex sequence. Furthermore, from Lemma~\ref{lem:simplify:tail:no-arrowhead} the walk $\omega$ must also be an anterior walk as every undirected edge in $G$ remains an undirected edge in $\simplify(G)$ and every directed edge in $G$ is either an undirected or directed edge in $\simplify(G)$.

    For the other direction, let $\omega$ be an anterior walk in $\simplify(G)$ with $\defwalkomega$. Associate with each edge $\edge{v_i}{v_{i+1}}$ on $\omega$ an anterior walk on $\gamma_i$; one must exist due to the fact that an edge $\edgevivj$ in $\simplify(G)$ has a tail at $v_i$ if and only if there is an anterior walk from $v_i$ to $v_j$ in $G$. The walk $\gamma=\gamma_1 \append \cdots \append \gamma_{n-1}$ is an anterior walk from $v_1$ to $v_n$ in $G$.
\end{proof}

\begin{lemma}\label{lem:separable-graph:miw:equivalent-walk-in-simplify}
    If separable graph $G$ contains \miw\ $\omega$ then there is a walk $\omega'$ in $\simplify(G)$ such that $\omega' \equiv \omega$.
\end{lemma}

\begin{proof}
    Let graph $G$ contains a \miw\ $\omega$ with $\defwalkomega$ and let $G'=\simplify(G)$.

    From Lemma~\ref{lem:simplify:same-adjacencies} there is a walk $\omega'$ in $G'$ such that $\defwalkomegaprime$. We consider the possible edge types for each edge $\walkedgev{i}$ for $1 \leq i < n$ on $\omega$ and the corresponding type of edge in $\omega'$.
    
    \proofcase{1} $v_i \tailtail v_{i+1}$. In this case, the edge must be an internal edge (i.e.,$1 < i < n-1$) and the corresponding edge on $\omega'$ must be undirected as $v_i\in \ant(v_j,G)$ and $v_j\in \ant(v_i,G)$.
    
    \proofcase{2} $v_i \tailarrow v_j$. In this case, the edge must be a terminal edge of $\omega$ and $i=1$ or $i=n$.    
    By Lemma~\ref{lem:separable-graph:minimal-inducing-walk:internal-not-anterial}, 
    the arrowhead at $v_j$ must be an arrowhead in the corresponding edge on $\omega'$.

    \proofcase{3} $v_i \arrowarrow v_j$. In this case, either $v_i$ or $v_j$ must be an internal vertex on $\omega$. If either endpoint is an internal vertex on $\omega$, by Lemma~\ref{lem:separable-graph:minimal-inducing-walk:internal-not-anterial}, it must have an arrowhead endmark at the vertex in the corresponding edge on $\omega'$.

    In all cases, all internal endmarks of edges on $\omega$ are preserved on edges in $\omega'$ and thus $\omega \equiv \omega'$.
\end{proof}

\begin{lemma}\label{lem:separable-graph:miw:miw-in-simplify}
    If $\omega$ with $\defwalkomega$ is a \miw\ in separable graph $G$ then there is a \miw\ $\omega'$ with $\defwalkomegaprime$ in $\simplify(G)$.
\end{lemma}

\begin{proof}
    Let graph $G$ contains a \miw\ $\omega$ with $\defwalkomega$ and let $G'=\simplify(G)$.
    From Lemma~\ref{lem:separable-graph:miw:equivalent-walk-in-simplify} there is a walk $\omega'\equiv \omega$ in $G'.$
    Due to the fact that $\omega$ is an inducing walk and Lemma~\ref{lem:simplify:same-adjacencies}, $\omega'$ must be an inducing walk.
    Next we show that inducing walk $\omega'$ is a \miw. Suppose that $\omega'$ is not a \miw. In this case, there must be a \miw\ $\gamma'$ that uses a subset of the vertices of $\omega'$. Let $\edgeprime{v_i}{v_j}$ be the first edge on $\gamma'$ that is a chord on $\omega'$. 

    \proofcase{1} $\edgeprime{v_i}{v_j}$ is the first edge of $\gamma'$. 

    In this case, $i = 1$ and the edge $\edgeprime{v_i}{v_j}$ has an arrowhead at $v_j$ and there must be an edge $\edge{v_1}{v_j}$ in $G$ with an arrowhead at $v_j$. But this implies that $\seq{\edge{v_1}{v_j}} \append \omega(v_j,v_n)$ is a shorter \iw\ using a subset of the vertices on $\omega$. This implies that $\omega$ is not a \miw\ which is a contradiction.

    \proofcase{2} $\edgeprime{v_i}{v_j}$ is the last edge of $\gamma'$. In this case $j=n$ and the proof follows as in Case 1.

    \proofcase{3} $\edgeprime{v_i}{v_j}$ is an internal edge of $\gamma'$. In this case $v_i$ and $v_j$ must both be on collider sections of $\omega'$. From the fact that the edge is on a \miw\ it must either be bidirected or undirected.

    \proofcase{3.1} $v_i \arrowarrow v_j$. In this case, there must be an edge $v_i \arrowarrow v_j$ in $G$. This, however, implies that $\omega(v_1,v_i)\append \seq{v_j \arrowarrow v_j} \append \omega(v_j,v_n)$ is a shorter inducing walk using a subset of vertices on $\omega$. This implies that $\omega$ is not a \miw\ which is a contradiction.

    \proofcase{3.2} $v_i \tailtail v_j$. In this case, there must either be an chord $v_i \arrowarrow v_j$ or a chord with a directed edge between $v_i$ and $v_j$ in $G$. As in Case 3.1, a bidirected edge leads to a contradiction. If, however, there is a directed edge in $G$, from Lemma~\ref{lem:separable-graph:min-ind-walk:no-chord-on-semi-directed-cycle}, we obtain a contradiction.

    Thus in all cases we have a contradiction and the lemma must hold.
\end{proof}

\begin{lemma}\label{separable-graph:miw-in-simplify:equivalent-walk-in-graph}
    If $G$ is a separable graph and $\omega'$ is a \miw\ in $\simplify(G)$ 
    then there is a walk $\omega \equiv \omega'$ in $G$.
\end{lemma}

\begin{proof}
    Let $G$ be a separable graph and let $G' = \simplify(G)$ contain a \miw\ $\omega'$ with $\defwalkomegaprime$ and $\defsectionomegaprime$.
    
    From Lemma~\ref{lem:simplify:same-adjacencies} there is a walk $\omega$ in $G$ such that $\defwalkomega$. Suppose the lemma is not true and that $\omega \not \equiv \omega'$. Because $\simplify$ only removes arrowheads there is some $\sigma_k'$ with $1 < k < m$ on $\omega'$ that contains an edge $\walkedgeundirected{l}$ such that the edge $\edge{v_l}{v_{l+1}}$ on $\omega$ either has an arrowhead at $v_l$ or $v_{l+1}$.
    Without loss of generality assume that $\edge{v_l}{v_{l+1}}$ has an arrowhead at $v_l$.
    Choose $\sigma_j'$ to be the collider section on $\omega'$ closest to $\sigma_1$ that contains an edge $\walkedgeundirected{i}$ and an edge $\edge{v_i}{v_{i+1}}$ with an arrowhead at $v_i$ on $\omega$.

    \proofcase{1} $\omega(v_1,v_{i+1})$ is a collider walk.

    By Lemma~\ref{lem:simplify:anterial-graph}, $G'$ is an anterial graph.
    This implies that there can be no chord $\edgeprime{v_{i+1}}{v_h}$ for $v_h$ on section $\sigma_{j-1}'$ of $\omega'$ as the chord must be directed due to the fact that $\omega'$ is a \miw\ and this would imply there is a bidirected edge with anterial endpoints. This implies, by Lemma~\ref{lem:simplify:same-adjacencies} and the fact that \simplify\ does not add arrowheads, that $\omega(l_{j-1},v_{i+1})$ with $l_{j-1}=\last(\sigma_j)$ is a \miw\ in $G$. By Lemma~\ref{lem:separable-graph:miw:miw-in-simplify}, this would imply that $\edgeprime{v_{i}}{v_{i+1}}$ on $\omega'$ has an arrowhead at $v_i$ which is a contradiction.

    \proofcase{2} $\omega(v_1,v_{i+1})$ is a not a collider walk.
    
    In this case, there must be an edge $\walkedgeundirected{l}$ on $\omega'(v_1,v_{i+1})$ such that the edge $\edge{v_l}{v_{l+1}}$ has an arrowhead at $v_{l+1}$ on $\omega$.
    Choose $\sigma_f'$ to be the collider section on $\omega'$ closest to and before $\sigma_j'$ that contains an edge $\walkedgeundirected{h}$ and an edge $\edge{v_h}{v_{h+1}}$ with an arrowhead at $v_h$ on $\omega$.

    Again by Lemma~\ref{lem:simplify:anterial-graph}, there can be no chord $\edgeprime{v_{h}}{v_l}$ for $v_l$ on $\sigma_{f+1}'$ as the chord must be directed due to the fact that $\omega'$ is a \miw\ and this would imply the existence of a bidirected edge with anterial endpoints. This implies, by Lemma~\ref{lem:simplify:same-adjacencies} and the fact that \simplify\ does not add arrowheads, that $\omega(l_{j-1},v_{i+1})$  with $l_{j-1}=\last(\sigma_j)$ is a \miw\ in $G$. By Lemma~\ref{lem:separable-graph:miw:miw-in-simplify}, this would imply that $\edgeprime{v_{h}}{v_{h+1}}$ on $\omega'$ has an arrowhead at $v_{h+1}$ which is a contradiction.

    In both cases we have a contradiction and thus $\omega \equiv \omega'$.
\end{proof}

\reftheorem{thm:simplify:separable-graph:equivalent-anterial-graph}

\begin{proof}

    Let $G$ be a graph and $G' = \simplify(G)$. 
    
    First, by Lemma~\ref{lem:simplify:anterial-graph}, the graph $G'$ is anterial.

    Next we show that if $G$ is a separable graph then $G \equiv G'$.
    To prove equivalence, from the definition of equivalence and Lemma~\ref{lem:open-walk::connecting-walk}, we show that $G$ and $\simplify(G)$ have the same open walks.
    For the forward direction, let $\omega$ with $\defwalkomega$ be a shortest open walk given $C$ in $G$. From Lemma~\ref{lem:shortest:mcw:min-inducing-walk}, every maximal collider walk on $\omega$ is a \miw. From Lemma~\ref{lem:separable-graph:miw:equivalent-walk-in-simplify}, there is a walk $\omega'$ in $G'$ such that $\omega \equiv \omega'$. From Lemma~\ref{lem:anterial-walk-in-graph::anterial-walk-in-anterialize}, $\omega'$ must be an open walk given $C$ in $G'$.
    For the backward direction, let $\omega'$ with $\defwalkomegaprime$ be a shortest open walk given $C$ in $G'$.
    From Lemma~\ref{lem:shortest:mcw:min-inducing-walk}, every maximal collider walk on $\omega'$ is a \miw. From Lemma~\ref{separable-graph:miw-in-simplify:equivalent-walk-in-graph}, there is a walk $\omega'$ in $G'$ such that $\omega \equiv \omega'$. From Lemma~\ref{lem:anterial-walk-in-graph::anterial-walk-in-anterialize}, $\omega'$ must be an open walk given $C$ in $G'$.
    Combining the forward and backward directions we have shown that $G\equiv G'$.

    Finally, if $G$ is a separable graph then the graph $G'$ must be separable as it is equivalent to separable graph and has the same adjacencies.
\end{proof}

The next lemma is directly implied by results from \cite{LauritzenSadeghi2018}.

\begin{lemma}\label{lem:acyc-graph:equivalent-separable-anterial-graph}
    if $G$ is acyclic then there is an equivalent separable anterial graph.
\end{lemma}

\begin{proof}
    The result is a consequence of Theorem~3 and Corollary~2 of \cite{LauritzenSadeghi2018}.
\end{proof}

\reftheorem{thm:weakly-separable::essentially-acyclic}

\begin{proof}
    For the forward direction, assume $G$ be an essentially separable graph.
    By definition, there is an equivalent separable graph $G_{sep}$. 
    From Theorem~\ref{thm:simplify:separable-graph:equivalent-anterial-graph} there is an equivalent separable anterial graph $G_{ant}$. As $G_{ant}$ is an acylic graph, thus the forward direction is proved.

    For the backward direction, assume $G$ is essentially acyclic. Thus there is an acyclic equivalent graph $G_{acyc}\equiv G$. From 
    Lemma~\ref{lem:acyc-graph:equivalent-separable-anterial-graph}, there is an equivalent separable acyclic graph $G_{sep}\equiv G_{acyc}$. Thus the backward direction is proved.
\end{proof}

\subsection{Proofs related to statistical equivalence} In order to prove Corollaries~\ref{cor:acyclic-power} and \ref{cor:anterial-power} it is useful to define comparison of graph families. The relations $\subseteq$ and $=$ are the standard set theoretic comparison of the sets of graphs that allow for the comparison of \emph{graphical expressivity} of two graph families.

We define $\sqsubseteq$ and $\equiv$ in terms of the \indepmodel s of the two graph families. 
These relations are a comparison of the \emph{separation expressivity} of graph families and are defined as

\newcommand{\Imodels}[1]{\DefineConstant{Imodels}(#1)}

\begin{eqnarray*}
\GraphFamilyF \equiv \GraphFamilyH &\quad \mathrm{if}\quad& \Imodels{\GraphFamilyF} = \Imodels{\GraphFamilyH} \\
\GraphFamilyF \sqsubseteq \GraphFamilyH &\quad \mathrm{if}\quad& \Imodels{\GraphFamilyF} \subseteq \Imodels{\GraphFamilyH} \\
\end{eqnarray*}
where $\Imodels{\GraphFamilyH}=\set{\imodel{H} \ | \ H\in \GraphFamilyH}$.

The following lemmas related graphical expressivity of graph families to the separational expressivity of graph families and the statistical expressivity of graphical model families.

\begin{lemma}\label{lem:separation-expressivity:statistical-expresivity}
    If $\GraphFamilyF \sqsubseteq \GraphFamilyH$ then $\gmfF \sqsubseteq \gmfH$.
\end{lemma}

\begin{proof}
    The lemma holds due to the fact that equivalent graphs define equivalent \indepmodel s (Proposition~\ref{prop:sep-equiv:stat-equiv}) and the set of distributions defined by a graphical model is defined in terms of the \indepmodel\ of the defining graph and the fixed defining family of distributions.
\end{proof}

\refcorollary{cor:anterial-power}

\begin{proof}
    By definition, $\SepFamily \equiv \WSepFamily$.
    $\SepAnterialFamily = \AnterialFamily \cap \SepFamily$ and thus $\SepAnterialFamily \subseteq \SepFamily$ which implies that $\SepAnterialFamily \sqsubseteq \SepFamily$.
    From Theorem~\ref{thm:simplify:separable-graph:equivalent-anterial-graph}, we have that $\SepFamily \sqsubseteq \SepAnterialFamily$. Combining these facts we have $\SepFamily \equiv \SepAnterialFamily$.
    By transitivity, we have $\WSepFamily \equiv \SepAnterialFamily$, and, by Lemma~\ref{lem:separation-expressivity:statistical-expresivity}, $\gmfWSep \equiv \gmfSepAnterial$.
\end{proof}

\refcorollary{cor:acyclic-power}

\begin{proof}
    From Lemma~\ref{lem:acyc-graph:equivalent-separable-anterial-graph} we have $\AcycFamily \sqsubseteq \SepAnterialFamily$. From the fact that $\SepAnterialFamily \subseteq \AcycFamily$ we also have $\SepAnterialFamily \sqsubseteq \AcycFamily$. This implies that $\AcycFamily \equiv \SepAnterialFamily$.
    Using this fact and Corollary~\ref{cor:anterial-power}, we have $\WSepFamily \equiv \AcycFamily$. 
    Then, by Lemma~\ref{lem:separation-expressivity:statistical-expresivity}, it also follows that $\gmfWSep \equiv \gmfSepAnterial$.
\end{proof}

\section{Proofs related to inducing vertices and graphical characterizations of essential graph families}

We prove propositions related to inducing vertices and graphical characterizations of essential graph families presented in Section~\ref{sec:inducing-vertices} and that are useful in many of the proofs in subsequent sections.

\refproposition{prop:induced-dependence}

\begin{proof}
    Suppose that the induced independence property does not hold in some semi-graphoid independence model.
 
    Let \indepmodel\ $I=\seq{V,T}$ be such a semi-graphoid and let $A,B,C,D$ be disjoint subsets of $V$ such that the induced independence property does not hold. In that case, it must be the case that  (1) $\indep{A}{B}{C}{I}$ and (2) $\dep{A}{B}{CD}{I}$. In order for the property to not hold it must be the case that either $\indep{A}{D}{CB}{I}$ or $\indep{D}{B}{CA}{I}$. We consider these two cases.

    \proofcase{1} (3) $\indep{A}{D}{CB}{I}$. From (1) and (3) and contraction we have (4) $\indep{A}{BD}{C}{I}$. From (4) and weak union we get (5) $\indep{A}{B}{CD}{I}$ which contradicts (2).

    \proofcase{2} (6) $\indep{D}{B}{CA}{I}$. From (1) and (6), symmetry and contraction we have (7) $\indep{B}{AD}{C}{I}$. From (7) and weak union and symmetry we get (8) $\indep{A}{B}{CD}{I}$ which contradicts (2).

    In either case we have a contradiction so the induced independence property must hold in every semi-graphoid independence model.
\end{proof}

\subsection{Proofs related to essentially undirected graphs}
In this section we prove Theorem~\ref{thm:stat-undirected::no-inducing-vertex}.

An independence model I satisfies the \emph{pairwise upward stable} property if for all $C\subseteq D$, if $\indep{i}{j}{C}{I} \Rightarrow \indep{i}{j}{D}{I}$.

The following was proved in \cite{Sadeghi2017}.

\begin{theorem}\label{thm:cor29}[Corollary 29, Sadeghi 2017]
If $G$ is a graph and independence model $I$ is such that $I=\imodel{G}$ then $G$ is essentially undirected if and only if $I$ is a graphoid and both the pairwise upward stability and weak transitivity independence properties hold.
\end{theorem}

\begin{lemma}\label{lem:graph-e-undirected::imodel-pairwise-upward}
    Graph $G$ is essentially undirected if and only if $\imodel{G}$ satisfies pairwise upward stability.
\end{lemma}

\begin{proof}
    Let $G$ be a graph. From Propositions~\ref{prop:compgraph} and \ref{prop:weaktrans}, $\imodel{G}$ is a graphoid that satisfies weak-transitivity. The result then follow from Theorem~\ref{thm:cor29}.
\end{proof}

An independence model $I$ is \emph{graphical} if there is a graph $G$ such that $I=\imodel{G}$. Note that graphical independence mdoels as defined in \cite{Sadeghi2017} is subset of graphical independence models in our paper as \cite{Sadeghi2017} only considers acyclic graph.

\begin{lemma}\label{lem:graphical:pairwise-upward::no-inducing-vertex}
    If independence model $I$ is graphical then $I$ satisfies pairwise upward stable if and only if contains no inducing vertices.
\end{lemma}

\begin{proof}
    Let independence model $I$ be graphical. This implies there is a graph $G$ such that $I=\imodel{G}$. From Propositions~\ref{prop:compgraph} and \ref{prop:weaktrans}, $\imodel{G}$ is a composition graphoid that satisfies weak-transitivity.

    For the forward direction, assume that $I$ satisfies pairwise upward stability. Aiming for a contradiction, assume that $I$ contains an inducing vertex $k$ for $\testcaseABC$. Thus we have $\indep{A}{B}{C}{I}$ and $\dep{A}{B}{Ck}{I}$. From $\dep{A}{B}{Ck}{I}$ and composition there is some $a\in A$ and $b\in B$ such that $\dep{a}{b}{Ck}{I}$. From $\indep{A}{B}{C}{I}$ and decomposition we have $\indep{a}{b}{C}{I}$. Thus, $I$ does not satisfy pairwise upward stability. This is a  contradiction and this direction must hold.

    For the reverse direction, assume that $I$ contains no inducing vertex. Aiming for a contradiction, assume $I$ does not satisfy pairwise upward stability. In this case there are vertices $i,j$ and sets $C\subset D$ such that $\indep{i}{j}{C}{I}$ and $\dep{i}{j}{D}{I}$. If $|D\setminus C|>1$, we can choose sets $C' \supseteq C$ and $D'\subseteq D$ such that $|D' \setminus C'| < |D\setminus C|$, $\indep{i}{j}{C}{I}$ and $\dep{i}{j}{D}{I}$. Thus, without loss of generality, we assume $|D\setminus C|=1$. Let $k=D\setminus C$. In this case $k$ is an inducing vertex for $\seq{i,j,C}$. This is a contradiction and this direction must hold.
\end{proof}

\reftheorem{thm:stat-undirected::no-inducing-vertex}

\begin{proof}
    Let $G$ be a graph. 
    For the forward direction, assume $G$ is essentially undirected. By Lemma~\ref{lem:graph-e-undirected::imodel-pairwise-upward}, $\imodel{G}$ satisfies pairwise upward stability. In addition, $\imodel{G}$ is graphical, thus by Lemma~\ref{lem:graphical:pairwise-upward::no-inducing-vertex}, we have that $\imodel{G}$ has no inducing vertices.

    For the backward direction, assume that $\imodel{G}$ contains no inducing vertices. By Lemma~\ref{lem:graphical:pairwise-upward::no-inducing-vertex}, we have that $\imodel{G}$ satisfies pairwise upward stability. By Lemma~\ref{lem:graph-e-undirected::imodel-pairwise-upward} we have that $I(G)$ must be equivalent to some undirected graph and thus $G$ is essentially undirected.
\end{proof}

\subsection{Proof related to ordering}

In Section~\ref{sec:inducing-vertices}, we defined the concept of an induced preorder $\oleqG$ of a graph $G$. We consider several ordered independence properties for graphs $G$ and prove two propositions related to ordering; Propositions~\ref{prop:inducing-vertex:not-anterior-to-separated-vertices-or-separating-set} and \ref{prop:inducing-equivalent:one-is-inducing-vertex:each-is-inducing-vertex}.

\begin{lemma}\label{lem:inducing-vertex:walks-into-inducing-vertex} If $b\in \IV(i,j,C,G)$ then there exists a walk between $i$ and $b$ that is open given $Cj$ and, furthermore, any walk between $i$ and $b$ that is open given $Cj$ is weakly into $b$.
\end{lemma}

\begin{proof}
    Let $G$ be a graph such $b\in \IV(i,j,C,G)$.
    From $b\in \IV(i,j,C,G)$ we have that $\indep{i}{j}{C}{G}$ and $\dep{i}{j}{Cb}{G}$. From Proposition~\ref{prop:induced-dependence} we have $\dep{i}{b}{Cj}{G}$ and $\dep{j}{b}{Ci}{G}$. This implies that the there are walks $\omega_i$ between $i$ and $b$ that is open given $Cj$ and $\omega_j$ between $j$ and $b$ that is open given $Ci$. The walk $\omega_i(i,b) \append \omega_j(b,j)$ is a walk between $i$ and $j$ and is open given $Cb$ if and only if both $\omega_i$ and $\omega_j$ are into $b$. Suppose that there is a walk $\omega_i'$ between $i$ and $b$ open given $Cj$ that is not weakly into $b$. In this case, the walk $\omega_i' \append \omega_j$ is a walk open given $C$ which is a contradiction. Similarly, suppose there is a walk $\omega_j'$ between $j$ and $b$ open given $Ci$ that is not weakly into $b$. In this case, the walk $\omega_i \append \omega_j'$ is a walk open given $C$ which is a contradiction. Thus every such walk must be weakly into $b$.
\end{proof}

\begin{lemma}
\label{prop:inducing-vertex:not-anterior-to-separated-vertices}
If $b\in \IV(i,j,C,G)$ then $b\not \in \ant(\set{i,j},G)$.
\end{lemma}

\begin{proof}Suppose the proposition is false and that, $b\in \IV(i,j,C,G)$ and $b\in \antsetij$.
From the definition of an inducing vertex and $b\in \IV(i,j,C,G)$ we know that $\indep{i}{j}{C}{G}$ and $\dep{i}{j}{Cb}{G}$. 
From $\dep{i}{j}{Cb}{G}$, there is a walk $\omega$ between $i$ and $j$ that is a connecting walk given $Cb$.
This implies that $b\not\in \set{i,j}$ otherwise the walks $\omega$ would be closed. Thus we have $b\in \propantij$.
From $\indep{i}{j}{C}{G}$, $\omega$ is not a connecting walk given $C$.
The fact that $\omega$ is open given $Cb$, implies that $\omega$ contains no non-collider section with a vertex in $Cb$. Thus $\omega$ must contain one or more collider sections closed given $C$ but open given $Cb$. 
Thus, from $\dep{i}{j}{Cb}{G}$, each such closed collider sections must contain $b$.
Because $b\in \propantij$ there must be an anterior walk from $b$ to a vertex in  $C\cup\set{i,j}$ in $G$. Thus, $\omega$ is an open walk given $C$ which, by Lemma~\ref{lem:open-walk::connecting-walk} implies there is a connecting walk given $C$ which is a contradiction. Therefore, the proposition must be true.
\end{proof}

An \indepmodel\ $I$ satisfies \emph{ordered upward separation stability} property with respect to $\oleq$ if $\indep{A}{B}{C}{I}$ and $D \subseteq \vant{ABC}\setminus ABC$ then $\indep{A}{B}{CD}{I}$.

\begin{lemma}\label{lem:indep-model-graph:ordered-upward}
$\imodel{G}$ satisfies ordered upward separation stability with respect to $\oleqG$.
\end{lemma}

\begin{proof}
    Let $G$ be a graph such that, $\indep{A}{B}{C}{G}$, $D \subseteq \vant{ABC} = \ant(ABC,G) \setminus ABC$. 
    We prove by induction on the size of $D$

    For the base case, $|D|=0$ we have $C=CD$ and thus the independence $\indep{A}{B}{CD}{G}$ holds by assumption.

    For this induction case we assume that the property holds for all sets $D'\subset D$ (i.e., $|D'|\leq n-1$) and and show that it holds for $|D|=n$.
    Choose a vertex $k\in D$ and let $E=CD\setminus \set{k}.$
    By the induction hypothesis we have $\indep{A}{B}{E}{G}.$
    Aiming at a contradiction assume that $\dep{A}{B}{Ek}{G}.$ This implies that there is some $i\in A$ and $b\in B$ such that $\dep{i}{j}{Ek}{G}.$ From $\indep{A}{B}{E}{G}$ we have that $\indep{i}{j}{E}{G}.$  Thus we have that  $k\in \IV(i,j,E,G)$.
    From Proposition~\ref{prop:inducing-vertex:not-anterior-to-separated-vertices}, $k\not \in \ant(\set{i,j},G)$ and, thus, $k\in \ant(E\setminus\set{i,j},G).$
    From $\dep{i}{j}{Ek}{G}$, there is some connecting walk between $i$ and $j$ given $Ck$ in $G$. Let $\omega$ be any such connecting walk.
    It must be the case that $k$ in not on any non-collider section of $\omega$ and every collider section on $\omega$ contains a vertex in $Ek$ otherwise $\omega$ would not be a connecting walk. Next, from $\indep{i}{j}{E}{G}$, there must be on some collider section in $\omega$ that does not contain any vertex in $E$ but contains $k$. This, however, implies that $\omega$ is open given $E$ as every such collider section is anterior to $E$ because $k\in \ant(E\setminus\set{i,j},G).$ From Lemma~\ref{lem:open-walk::connecting-walk}, this implies there is a connecting walk between $i$ and $j$ given $E$ which contradicts $\indep{i}{j}{E}{G}.$
    Thus, we have a contradiction and the lemma must be true.
\end{proof}

\refproposition{prop:inducing-set:not-all-anterior-to-separated-vertices-or-separating-set}

\begin{proof}
    Follows from Lemma~\ref{lem:indep-model-graph:ordered-upward}.
\end{proof}

\refproposition{prop:inducing-vertex:not-anterior-to-separated-vertices-or-separating-set}

\begin{proof}
    Follows from Lemma~\ref{lem:indep-model-graph:ordered-upward}.
\end{proof}

An \indepmodel\ $I$ satisfies \emph{ordered downward separation stability} property with respect to $\oleq$ if $\indep{A}{B}{CD}{I}$ and $D \cap \vant{ABC}=\emptyset$ then $\indep{A}{B}{C}{I}$.

\begin{lemma}\label{lem:indep-model-graph:ordered-downward}
$\imodel{G}$ satisfies ordered downward separation stability with respect to $\oleqG$.
\end{lemma}

\begin{proof}
    Aiming for a contradiction, assume that the claim does not hold for graph $G$.
    
    Choose $D$ to be a smallest set such that there exists $A,B,C$ such that 
    $\indep{A}{B}{CD}{G}$, $D\cap \vant{ABC} = D\cap \ant(ABC,G) =\emptyset$ and $\dep{A}{B}{C}{G}$.
    Due to the fact that $D\cap \ant(ABC,G) =\emptyset$, there exists a subset $E \subseteq D$ such that $E \cap \ant(CD\setminus E,G) =\emptyset$. 
    From the fact, $D$ is chosen to be a smallest set and $D \cap \ant(ABC,G) = \emptyset$, it must be the case that $\dep{A}{B}{CD\setminus E}{G}$.
    This implies that there is a connecting walk between $i\in A$ and $j\in B$ given $CD\setminus E$ in $G$.
    Let $\omega$ be any such connecting walk. 
    Because $\omega$ is a connecting walk given $CD\setminus E$, every collider section of $\omega$ must contain a member of $CD\setminus E$. From the choice of $E \cap \ant(CD\setminus E,G) =\emptyset$, no collider section containing a vertex in $E$ can contain a vertex in $CD\setminus E.$ This implies that $\omega$ must also be a connecting walk given $CD$. This is a contradiction so the proposition holds.
\end{proof}

An important property of vertices in the same equivalence class is that they have the same anterior sets.
\begin{lemma}\label{lem:anterior-equivalence:same-anterior-set}
If $a \oeqG b$ then $\ant(a,G)=\ant(b,G)$.
\end{lemma}
\begin{proof}
    We assume that $a \oeqG b$ in graph $G$. In this case, there is an anterior walk $\omega_{ab}$ from $a$ to $b$ in $G$ and an anterior walk $\omega_{ba}$ from $b$ to $a$ in $G$. For the first direction, let $c\in \ant(a,G)$. This implies that there is an walk $\omega_{ca}$ from $c$ to $a$ in $G$. The walk $\omega_{ca} \append \omega_{ab}$ is an anterior walk from $c$ to $b$ which implies that $c\in \ant(b,G)$. The other directions follows by a similar argument.
\end{proof}

\refproposition{prop:inducing-equivalent:one-is-inducing-vertex:each-is-inducing-vertex}

\begin{proof}
    Let graph $G$ be such that $a\in \IV(i,j,C)$ and $b\in \oequivclassG{a}$.
    This implies that $\indep{i}{j}{C}{G}$ and $\dep{i}{j}{Ca}{G}$. Thus there is an open walk $\omega$ between $i$ and $j$ given $Ca$ which implies $\omega$ contains no non-collider section that contains a vertex in $Ca$. Because, $\indep{i}{j}{C}{G}$, there are no non-collider section that is anterior to $a$ but not $C$. From Lemma~\ref{lem:anterior-equivalence:same-anterior-set}, such a section is anterior to $b$ for any $b\in \oequivclass{a}$. This implies that the walk is open given $Cb$ and thus the claim holds.    
\end{proof}

\subsection{Proofs related to essentially separable graphs}

We next prove Theorem~\ref{thm:essentially-separable::inducing-equivalence}

A useful equivalence relation for graphs is $a \sim_G b$ which holds if there is an undirected walk between $a$ and $b$.

\begin{lemma}\label{lem:acyclic:order-equiv:undirected-walk}
    if $G$ is acyclic and $a \oeqG b$ then $a \sim_G b.$
\end{lemma}
\begin{proof}
    Let $\defgraphG$ be acyclic, that is, it contains no semi-directed circuit, and let $a =_G b$.

    \proofcase{1} $a = b$. In this case, the equality gives $a \sim_G b$.

    \proofcase{2} $a \neq b.$ In this case, we have $a \not\in \ant(b,G) \wedge b \not\in \ant(a,G)$. From that fact that $a\neq b$ there must be anterior walks from $a$ to $b$ and from $b$ to $a$. Let $\gamma_{ab}$ be an arbitrary anterior walk from $a$ to $b$ and let $\gamma_{ba}$ be an arbitrary anterior walk from $b$ to $a$. If either or both of these walks is not undirected then the walk $\gamma_{ab} \append \gamma_{ba}$ would be a semi-directed circuit and $G$ would not be acyclic. This implies that all anterior walk between $a$ and $b$ are undirected and thus $a \sim_G b.$
    In either case we have $a \sim_G b$ so the lemma holds.
\end{proof}

\begin{lemma}\label{lem:acyclic:set-wise-inducing-equivalence}
    If $G$ is acyclic then $\imodel{G}$ satisfies inducing set equivalence property with respect to $\oleqG$.
\end{lemma}

\begin{proof} 
Let $G$ be an acyclic graph. Assume that the inducing set property does not hold. In that case there is a set $D$ such that $d\in D$ and $D=\oequivclassG{d}$ and $D \in \IS(A,B,C,G)$ and for $E$, a non-empty subset $E\subseteq \oequivclassG{d}$ we have $\indep{A}{B}{CE}{G}$.
From $D \in \IS(A,B,C,G)$, we have $\indep{A}{B}{C}{I}$ and $\dep{A}{B}{CD}{I}$.
From $\dep{A}{B}{CD}{I}$ there must be a walk $\omega$ between $i\in A$ and $j\in B$ that is connecting given $CD$. 
From Lemma~\ref{lem:open-walk-given-C:every-vertex-in-anterior}, every vertex on $\omega$ is in $\ant(CDij,G)$. This implies that no vertex in $D$ can be a non-collider on $\omega$ otherwise the walk $\omega$ would not be a connecting walk given $CD$. Note that for every non-empty $E\subset \oequivclassG{d}$ it is the case that $\ant(CDij,G)=\ant(CDij,G)$. This fact implies that $\omega$ is also open given $CE$ which is a contradiction. Thus the claim holds.
\end{proof}

The proof of the following lemma relies on definitions of ordered upward-stability (OUS) and ordered downward-stability (ODS) with respect to $\oleq$ from \cite{Sadeghi2017}.

An \indepmodel\ $I$ satisfies \emph{ordered upward-stability (OUS)} with respect to $\oleq$ if $\indep{i}{j}{C}{I}$ and $k\in \vant{ij}$ or $k \oeq l\in C$ then $\indep{i}{j}{Ck}{I}$.

An \indepmodel\ $I$ satisfies \emph{ordered downward-stability (ODS)} with respect to $\oleq$ if $\indep{i}{j}{C}{I}$ and  $k\in C$ is such that
$k\in V \setminus V_{> C\setminus k}$
and $k\in V \setminus \vant{ij}$ then $\indep{i}{j}{C\setminus k}{}$

\begin{lemma}\label{lem:ordered:OUS-ODS}
    If there exists a preeorder $\oleq$ for $V$ such that \indepmodel\ $I$ for $V$ satisfies ordered upward separation stability, ordered downward separation stability, and the inducing set equivalence property with respect to $\oleq$ then OUS and ODS with respect to $\oleq$ hold for $I$.
\end{lemma}

\begin{proof}
    Let $\oleq$ be a preorder for $V$ such that \indepmodel\ $I$ for $V$ satisfies ordered upward separation stability, ordered downward separation stability, and the inducing set equivalence property with respect to $\oleq$. First, if $I$ satisfies ordered upward separation stability it satisfies OUS. 
    
    Next, we show that $I$ must satisfy ODS.
    Aiming for a contradiction, assume that  $\indep{i}{j}{C}{I}$ and $k$ be a vertex such that 
    $k \in V \setminus V_{>C\setminus k}$ 
    and $k\in V \setminus \vant{ij}$ but that $\dep{i}{j}{C\setminus k}{I}.$

    Let $E= C\cap \oequivclass{k}$ and $D = C\setminus E.$
    From ordered downward separation stability (Lemma~\ref{lem:indep-model-graph:ordered-downward}), we have that $\indep{i}{j}{D}{I}.$

    We continue by cases.
    
    \proofcase{1} $\indep{i}{j}{Dk}{I}$. We consider two sub-cases.
    
    \proofcase{1.1}, $|E|=1$. 
    In this case, $C=Dk$, and from the fact that $\indep{i}{j}{D}{I}$, we have $\indep{i}{j}{C\setminus k}{I}$ which is a contradiction. 

    \proofcase{1.2} $|E|>1$. In this case, we again have $C=DE$.
    From the fact that $\dep{i}{j}{C\setminus k}{I}$ and $\indep{i}{j}{D}{I}$, we have $E\setminus k \in \IS(i,j,D,I)$. Thus we have that $E\in \IS(i,j,D,I)$ and, by the inducing set equivalence property, we have $\dep{i}{j}{C}{I}$ which is a contradiction.
 
    
    \proofcase{2} $\dep{i}{j}{Dk}{}$. In this case,  $k\in \IS(i,j,D,G)$.
    From the inducing set equivalence property we have that $E\in \IS(i,j,D,I)$ and it follows that $\dep{i}{j}{C}{I}$. Again we have a contradiction.

    In every case we have a contradiction and thus ODS must hold in $I$.
\end{proof}

An \indepmodel\ $I$ is \emph{acyclic} if there is an acyclic graph $G$ such that $I=\imodel{G}$. Note that in \cite{Sadeghi2017}, only acyclic graphs are considered so graphical in \cite{Sadeghi2017} is equivalent to acyclic in this paper.

\begin{theorem}\label{thm:properties:acyclic-indep-model}[Theorem 17, Sadeghi 2017]
If $G$ is an acyclic graph then independence model $I = \imodel{G}$ if and only if $\imodel{G}$ is a compositional graphoid that satisfies weak-transitivity and $\oleqG$ such that $I$ satisfies 
ordered upward-stability and ordered downward-stability with respect to $\oleqG$.
\end{theorem}

\reftheorem{thm:essentially-separable::inducing-equivalence}

\begin{proof}
    For forward direction, we assume $G$ is essentially separable. By Theorem~\ref{thm:weakly-separable::essentially-acyclic}, $G$ is essentially acyclic which implies that there is an equivalent acyclic graph $H \equiv G$.
    By Proposition~\ref{lem:acyclic:set-wise-inducing-equivalence}, $\imodel{G}=\imodel{H}$ satisfies the inducing set equivalence property with respect to $\oleqG$.

    For the other direction, we assume there is a graph $G$ such that $\imodel{G}$ satisfies the inducing set equivalence property with respect to preorder 
    $\oleqG$.
    From Propositions~\ref{prop:compgraph} and \ref{prop:weaktrans}, $\imodel{G}$ is a compositional graphoid that satisfies weak-transitivity.
    From Lemmas~\ref{lem:indep-model-graph:ordered-upward} and \ref{lem:indep-model-graph:ordered-downward}, $\imodel{G}$ satisfies ordered upward separation stability and ordered downward separation stability with respect to $\oleqG$. From Lemma~\ref{lem:ordered:OUS-ODS}, $\imodel{G}$ satisfies ODS and OUS with respect to $\oleqG$. Thus we can apply Theorem~\ref{thm:properties:acyclic-indep-model} to obtain the conclusion that $I$ is acyclic. From the definitions of an acyclic independence model and essentially acyclic and Theorem~\ref{thm:weakly-separable::essentially-acyclic}, we have that $G$ is essentially separable.
\end{proof}

\section{Proofs related to the characterizations of separation equivalence}

\TextVersion{In this section we prove Theorems~\ref{thm:separable-graphs:equivalent::same-adj-and-same-miw}, and
\ref{thm:separable-graphs:equivalent::same-adj-same-inducing-arrowheads}. Each of these theorems provides an alternative equivalent characterization of the equivalence of two separable graphs. In addition, we prove Proposition~\ref{prop:induced-arrowhead:exclusive-arrowhead} that proves that an edge with an induced arrowhead has an exclusive arrowhead.
}{
The goal of this section is to prove Theorems~\ref{thm:separable-graphs:equivalent::same-adj-and-same-miw},
\ref{thm:separable-graphs:equivalent::same-adj-same-mdiw}, and
\ref{thm:separable-graphs:equivalent::same-adj-same-inducing-arrowheads}. Each of these theorems provides an alternative equivalent characterization of the equivalence of two separable graphs.
}

\subsection{Equivalence to same adjacencies and \iarrowhead s}

\begin{lemma}\label{lem:separable-graphs:equivalent:same-vertex-separability-same-induced-arrowheads}
    Two equivalent separable graphs have the same vertex separability and the same \iarrowhead s.
\end{lemma}

\begin{proof}
    Let $G$ and $H$ be two equivalent separable graphs.
    The properties of vertex separability and \iarrowhead s are separational properties and invariant to separation equivalence. That is, these properties are completely determined by the independence and dependence statements true in the graph. As $G$ and $H$ are equivalent they must agree on these properties.
\end{proof}

\begin{lemma}\label{lem:separable-graphs:same-vertex-separability:same-adj}
Two separable graphs have the same vertex separability if and only if they have the same adjacencies.
\end{lemma}

\begin{proof}
    Let $G$ and $H$ be two separable graph.
    From the fact that the two graphs are separable, vertex separability corresponds to adjacency. Thus the two graphs have the same adjacency.
\end{proof}

\subsection{Same adjacencies and same \iarrowhead s to same \miw s}

\begin{lemma}\label{lem:mdiw:vertex-not-on-discriminated:not-inducing-vertex}
    If $\omega$ with $\defwalkomega$ and $\defsectionomega$ is a \mdiw\ that discriminates section $\sigma_k$ in a graph $G$, and $d$ is on $\sigma_i \neq \sigma_k$ then for all $C\subseteq V$ it is the case that $d \not\in \IV(v_1,v_n,C,G)$
\end{lemma}

\begin{proof}
    Let $\omega$ with $\defwalkomega$ and $\defsectionomega$ be a \mdiw\ that discriminates section $\sigma_k$ in a graph $G$, and $d$ is on $\sigma_i \neq \sigma_j$.
    Suppose that there is a $C\subseteq V$ such that $d\in \IV(v_1,v_n,C,G)$. Then by Proposition~\ref{prop:inducing-vertex:not-anterior-to-separated-vertices}, $d\not\in \antomegaendpoints$. This cannot be the case as all sections other than $\sigma_k$ must be anterior to the endpoints of $\omega$. Thus we have a contradiction and for all $C$ it is the case that $d \not\in \IV(v_1,v_n,C,G)$.
\end{proof}

\begin{lemma}\label{lem:separable-mdiw:post-discriminated:inducing-vertex}
    If $\omega$ with $\defwalkomega$ is a \mdiw\ that discriminates section $\sigma$ in a separable graph $G$, and $b\in \post(\sigma,G)$ then there is a $C\subseteq V$ such that $b\in \IV(v_1,v_n,C,G)$.
\end{lemma}

\begin{proof}
    Let $\omega$ with $\defwalkomega$ be a \mdiw\ that discriminates section $\sigma$ in a separable graph $G$ and $b\in \post(\sigma,G)$.
    From separability of $G$ we have that $\indep{v_1}{v_n}{C}{G}$ for $C=\propantomegaendpoints$. The walk $\omega$ is open given $b$ and thus, by Lemma~\ref{lem:open-walk::connecting-walk}, $\dep{v_1}{v_n}{Cb}{G}$. Thus $b \in \IV(v_1,v_n,C,G)$.
\end{proof}

\begin{lemma}\label{lem:separable-graph:mdiw-internal-arrowhead:induced-arrowhead}
    If separable graph $G$ contains a \mdiw\ $\omega$ with $\defsectionomega$ that discriminates $\sigma_j$ and edge $\edge{d}{b}$ such that $d$ is on a section $\sigma_i\neq \sigma_j$ of $\omega$ and $b\in \post(\sigma_j,G)$ then the edge $\edge{d}{b}$ has an induced arrowhead at $b$ in $G$.
\end{lemma}

\begin{proof}
    Follows from Lemmas~\ref{lem:mdiw:vertex-not-on-discriminated:not-inducing-vertex} and \ref{lem:separable-mdiw:post-discriminated:inducing-vertex}.
\end{proof}

\refproposition{prop:induced-arrowhead:exclusive-arrowhead}

\begin{proof}
Let $G$ be a graph containing an edge $\edgedb$ that is an induced arrowhead at $b$. In this case, there are vertices $v_1$ and $v_n$ vertex set $C$ such that $b\in \IV(v_1,v_n,C,G)$ and $d\not\in \IV(v_1,v_n,C,G)$. 

From $b\in \IV(v_1,v_n,C,G)$ and Lemma~\ref{lem:inducing-vertex:walks-into-inducing-vertex} there is a vertex set $C$ and connecting walks $\gamma_1$ between $v_1$ and $b$ open given $Cv_n$ and $\gamma_n$ between $v_n$ and $b$ open given $Cv_1$ both weakly into $b$ such that $\indep{v_1}{v_n}{C}{G}$.

Suppose that $\edgedb$ does not have an exclusive arrowhead at $b$ in $G$. Let $\edgeprimedb$ be an edge with a tail at $b$ in $G$.
Consider the walk $\omega'=\gamma_1(v_1,b) \append \seq{\edgeprimedb, \edgeprimedb} \append \gamma_n(b,v_n)$. The walk $\omega'$ is open given $Cd$. This however implies that $d\in \IV(i,j,C,G)$ and we have a contradiction. Thus the edge $\edgedb$ must have an exclusive arrowhead at $b$ in $G$.
\end{proof}

\begin{lemma}\label{lem:separable-graph-and-same-adj:induced-arrowhead:same-miw}
If separable graphs have the same adjacencies and same \iarrowhead s then they have the same \miw s.    
\end{lemma}

\begin{proof}
    Let $G$ and $G'$ be anterial graphs with the same adjacencies and same \iarrowhead s.
    
    We prove that every \miw\ in $G$ is a \miw\ in $G'$ by induction on the number of collider sections.

    For the base case, we consider a \miw\ $\omega$ with $\defwalkomega$ and collider section $\sigma=\seq{v_2,v_{n-1}}$ in $G$. By Lemma~\ref{lem:separable-graph:mdiw-internal-arrowhead:induced-arrowhead}, the edge $\edge{v_1}{v_2}$ on $\omega$ has an \iarrowhead\ at $v_2$ and the edge $\edge{v_{n-1}}{v_n}$ on $\omega$ has an \iarrowhead\ at $v_{n-1}$. This implies the edge $\edgeprime{v_1}{v_2}$ on $\omega'$ has an \iarrowhead\ at $v_2$ and the edge $\edgeprime{v_{n-1}}{v_n}$ on $\omega'$ has an \iarrowhead\ at $v_{n-1}$ in $G'$.
    By Proposition~\ref{prop:induced-arrowhead:exclusive-arrowhead}, these edges have arrowheads at $v_2$ and $v_{n-1}$.

    By Lemma~\ref{lem:separable-graph:minimal-inducing-walk:no-chord-in-collider-section}, $\sigma$ is chordless in $\omega$ and must also be chordless in $\omega'$ due to $G$ and $G'$ having the same adjacencies.

    Suppose $\sigma$ is not a collider section on $\omega'$ in $G'$. In this case there must be an edge $\walkedgeprimev{i}$ with $1 < i < n-1$ with an arrowhead at either $v_i$ or $v_{i+1}$. \Wlog\ assume there is an arrowhead at $v_i$. Furthermore choose smallest $j > 1$ such that there is an edge $\walkedgeprimev{j}$ with an arrowhead at $v_j$. In this case the walk $\omega'(v_1,v_{j+1})$ is an unshielded collider trisection due to Lemma~\ref{lem:separable-graph:minimal-inducing-walk:no-chord-between-adjacent-sections} and the fact that $G$ and $G'$ have the same adjacencies.    
    This implies that $\walkedgeprimev{j}$ has an induced arrowhead at $v_j$ which in turn implies that $\walkedgev{j}$ on $\omega$ has an induced arrowhead. Finally, by Proposition~\ref{prop:induced-arrowhead:exclusive-arrowhead}, the arrowhead must be exclusive in $G$ and $\sigma$ is not a collider section which is a contradiction.

    For the inductive case, we assume that the lemma is true for \miw s with $n$ collider sections and prove for $n+1$ collider section. Let $\omega$ with $\defwalkomega$ and $\defsectionomega$ be a \miw\ with $n+1$ collider section. $\omega$ must discriminate some section $\sigma_j$ otherwise it would be a \siw\ and that would imply that $G$ is not separable by Theorem~\ref{thm:separable-graph::no-self-inducing-walks}.

    Let $\gamma'$ with $\defwalkWLIJ{\gamma'}{v}{i}{k+1}$ be the collider trisection containing $\sigma_j$ on $\omega'$. 
    In that case, by an argument similar to the base case, the edge $\walkedgeprimev{i}$ and the edge $\walkedgeprimev{k}$ have arrowheads at their internal vertices and $\gamma(v_{i+1},v_{k})$ is a collider section.

    If $i > 2$ then by Lemmas~\ref{lem:decomposing-minimal-inducing-walks} and
\ref{lem:miw:no-chord-between-endpoints-and-discriminated-section}, $\omega(v_1,v_{i+1})$ must be a \miw\ and there must be must be a \miw\ $\omega'(v_1,v_{i+1})$.

    Similarly, if $k < n-1$ then by Lemmas~\ref{lem:decomposing-minimal-inducing-walks} and
\ref{lem:miw:no-chord-between-endpoints-and-discriminated-section}, $\omega(v_k, v_n)$ must be a \miw\ and there must be a \miw\ $\omega'(v_k, v_n)$.

    Thus, $\omega'(v_1,v_n)$ is an inducing walk. If $\omega'(v_1,v_n)$ is not a \miw\ then there must be an edge $\edge(v_h,v_j)$ with $h < i$ and $v_j$ on $\sigma_j$ with an arrowhead at $v_j$ or $\edge(l,v_j)$ with $l > k$ and $v_j$ with an arrowhead at $v_j$. In either case, $\omega(v_1,v_n)$ cannot be a \miw\ which is a contradiction. Thus $\omega'(v_1,v_n)$ is a \miw.
\end{proof}

\subsection{Same adjacencies and \miw s to equivalence}

\begin{lemma}\label{lem:anterial-graphs:same-adj-same-miw:equivalent-shortest-open-walks}
    If anterial graphs $G$ and $G'$ have the same adjacencies and the same \miw s and $\omega$ is a shortest open walk given $C$ in $G$ then there is a walk $\omega'$ such that $\omega \equiv \omega'$.
\end{lemma}

\begin{proof}
    Let $G$ and $G'$ be two anterial graphs with the same adjacencies and same \miw s. Let walk $\omega$ be a shortest open walk given $C$ in $G$ and let $\defwalkomega$ and $\defwalkdecompomegaI{m}$.
    From $G$ and $G'$ having the same adjacencies, there is a walk $\omega'$ in $G'$ with $\defwalkomegaprime$.
    Suppose that $\omega \not \equiv \omega'$.
    In order for this to be the case, the walk decompositions of the two walks must be different. First note that if two walks have different maximal treks they must have different maximal collider walks.
    Thus there must be a $\gamma_i$ on $\omega$ such that $\gamma_i'$ on $\omega'$ is not a maximal collider walk on $G'$. 
    From Lemma~\ref{lem:shortest:mcw:min-inducing-walk}, every $\gamma_i$ on $\omega$ must be a \miw\ in $G$. Anterial graphs are simple graphs and thus every endmark on an edge is exclusive, thus $\gamma_i \equiv \gamma_i'$.
    Suppose, however, that $\gamma_i'$ is not a maximal collider walk.
    In this case there must be a subwalk $\omega'(v_h,v_j)$ that contains $\gamma_i'$ that is a \mcw\ on $\omega'$. This must be an \iw\ due to 
    Lemma~\ref{lem:shortest:no-chord-between-distinct-maximal-treks} and the fact that $G$ and $G'$ have the same adjacencies. In this case, there must be a \miw\ $\gamma'$ between $v_h$ and $v_j$ that contains a subset of the vertices on $\omega'(v_h,v_j)$ in $G'$. Thus there must be a  \miw\ $\gamma \equiv \gamma'$ in $G$. This however contradicts the assumption that $\omega$ is a shortest open walk as the walk $\omega(v_1,v_h)\append \gamma \append \omega(v_j,v_n)$ would be a shorter open walk given $C$. Thus there is no such $\gamma_i$ and $\omega\equiv \omega'$.
\end{proof}

\begin{lemma}\label{lem:separable-diw:post-discriminated-collider-section:inducing-vertex}
    If the endpoints of a discriminating inducing walk are separable and the walk discriminates section $\sigma$ then every vertex in $\post(\sigma,G)$ is an inducing vertex for the endpoints of the walk.
\end{lemma}

\begin{proof}
    Let $\omega$ be a discriminating inducing walk between $i$ and $j$ that discriminates section $\sigma$ in graph $G$. Because $i$ and $j$ are separable there is some set that separates them. Let $C$ be an arbitrary separating set for $i$ and $j$, that is, $\indep{i}{j}{C}{G}$. Let $v$ be a arbitrary vertex in $\post(\sigma,G)$. The walk $\omega$ is open given $C\cup\set{v}$ and, by Lemma~\ref{lem:open-walk::connecting-walk} there is a connecting walk $\omega'$ in $G$ and $\dep{i}{j}{Cv}{G}$. Thus $v$ is an inducing vertex for $i$ and $j$ given $C$.
\end{proof}

\begin{lemma}\label{lem:anterial-graph:shortest-anterial:chordless}
    Every shortest anterior walk in an anterial graph is chordless.
\end{lemma}

\begin{proof}
    Let $G$ be an anterial graph and $\omega$ be a shortest anterior walk with vertex sequence $\vertexseq{\omega}=\seq{v_1,\ldots,v_n}$.

    Suppose that there $\omega$ has a chord $\edge{v_i}{v_j}$ for $i < j-1$.
    If the edge has a tail at $b_i$ then there would be a shorter anterior walk $\gamma(v_1,v_i)\append \edge{v_i}{v_j}\append \gamma(v_j,v_m)$ an thus there must be an arrowhead at $v_i$. However, if  $v_i \arrowarrow v_j$ then there would be a bidirected edge with one endpoint having anterior walk to the other which contradicts $G$ being an anterial graph. Finally, if $v_i \arrowtail v_j$ then there would be a semi-directed circuit in $G$ which again contradicts $G$ being an anterial graph. Thus, a shortest anterior walk must be chordless in an anterial graph.
\end{proof}

\begin{lemma}\label{lem:anterial-graph:minimal-inducing-walk:internal-not-anterial}
    In an anterial graph, no internal section on a \miw\ can be anterior to an adjacent section.
\end{lemma}

\begin{proof}
    Let $G$ be an anterial graph. Assume the lemma is not true and that $\omega$ with $\defsectionomega$ is a \miw\ that contains a section anterior to an adjacent section. \Wlog\ assume that $\sigma_i$ is anterior to $\sigma_{i+1}$ and that there is an anterior walk $\gamma'=\defwalkWLIJ{\gamma'}{v}{p}{q}$ with $v_p$ on $\sigma_i$ and $v_q$ on $\sigma_j$. Let $l_i=\last(\sigma_i)$ and $f_{i+1}=\first(\sigma_{i+1})$ and $\gamma= \omega(l_i, v_p) \append \gamma' \append \omega(v_q, f_{i+1})$.
    The edge $\edge{l_i}{f_{i+1}}$ on $\omega$ must either be $l_i \tailarrow f_{i+1}$ or $l_i \arrowarrow f_{i+1}$. In the first case, the circuit $\gamma \append \seq{\edge{l_i}{f_{i+1}}}$ is a semi-directed circuit and $G$ is not an anterial graph. This is a contradiction. In the second case, $l_i \arrowarrow f_{i+1}$ has an anterior walk from one endpoint to the other and thus $G$ is not an anterial graph. This is a contradiction. In either case, we have a contradiction and thus the lemma must be true.
\end{proof}

\begin{lemma}\label{lem:separable-anterial-graphs:same-adj-and-same-miw:equivalent}
    If two separable anterial graph $G$ and $G'$ have the same adjacencies and same \miw s then they are equivalent; that is  $G\equiv G'$.
\end{lemma}

\begin{proof}
    Let $G$ and $G'$ be separable anterial graphs with the same adjacencies and same \miw s.

    Aiming for a contradiction, assume that $G\not\equiv G'$.
    Let there be a walk between $v_1$ and $v_n$ that is connecting given $C$ in one walk but not the other. \Wlog, assume $\dep{v_1}{v_n}{C}{G}$ and $\indep{v_1}{v_n}{C}{G'}$.

    Choose $\omega$ with $\defwalkomega$ and $\defwalkdecompomegaI{m}$ to be a shortest open walk given $C$ in $G$ and let $\omega'$ be a walk in $G'$ with $\defwalkomegaprime$. $\omega'$ must not be open given $C$ in $G'$ due to $\indep{v_1}{v_n}{C}{G'}$.

    From Lemma~\ref{lem:anterial-graphs:same-adj-same-miw:equivalent-shortest-open-walks}, we have $\omega \equiv \omega'$ which implies that $\sectionseq{\omega}=\sectionseq{\omega'}$. 
    Let $\defsectionWLIJ{\omega}{\sigma}{1}{p}$, $f_i=\first(\sigma_i)$ and $l_i = \last(\sigma_i)$.     
    Note that $\sectionseq{\omega}=\sectionseq{\omega'}$ due to the fact that $\omega \equiv \omega'$.

    Choose smallest $j$ with $1 < j < p$ such that $\omega'(v_1,f_{j+1})$ is not open given $C$. There must be one
    otherwise $\omega'(v_1,v_l)$ would be an open walk given $C$.
    
    Let $\gamma_d=\omega(l_h,f_{j+1})$ be a \mdiw\ that discriminates $\sigma_j$ on walk $\omega(v_1,f_{j+1})$. Note that $1 \leq h < j$.
    It must be the case that $\sigma_j$ is anterior to $Cv_1$ in $G$ otherwise the walk $\omega$ would not be open in $G$. Furthermore, $\sigma_j$ is not anterior to $C\cup\set{v_1,f_{j+1}}$ in $G'$ otherwise the walk $\omega'(v_1,f_{j+1})$ would be open given $C$.
    
    Let $\gamma_f= \defwalkWLIJ{\gamma'}{a}{1}{q}$ be a shortest anterior walk from $f_j$ to a vertex in $C\cup \set{v_1,f_{j+1}}$ in $G$ and let $\gamma_f'$ be the corresponding walk in $G'$.
    By Lemma~\ref{lem:anterial-graph:shortest-anterial:chordless}, $\gamma_f$ must be chordless and, thus, by $G$ and $G'$ having the same adjacencies, $\gamma_f'$ is also chordless.
        
    Choose $r$ to be largest such that $a_r\in \post(\sigma_j,G')$ and let $s=r+1$ and thus $a_s \not \in \post(\sigma_j,G')$. Note that $\edgeprime{a_r}{a_s}$ on $\gamma_f'$ must have an arrowhead at $a_r$.

    Suppose $a_s \not \in \adj(l_{j-1},G')$. 
    In this case the walk $\omega'(l_{j-1},f_j) \append \gamma'_f(f_j,a_r) \append \seq{\edgeprime{a_r}{a_s}}$ in $G'$ must contain a unshielded collider trisection not in $G$ which is a contradiction due to the fact that the two graphs have the same \miw s. Thus $a_s \in \adj(l_{j-1},G')$.

    Suppose that $a_s \not \in \adj(f_{j+1},G')$ and let
    $\gamma_l$ be a shortest anterior walk from $l_j$ to $a_r$ in $G'$. By Lemma~\ref{lem:anterial-graph:shortest-anterial:chordless}, $\gamma'_l$ is chordless.
    In this case, the walk $\omega'(f_{j+1},l_j) \append \gamma'_l(l_j,a_r) \append \seq{\edgeprime{a_r}{a_s}}$ in $G'$ must contain an unshielded collider trisection not in $G$ which is a contradiction. Thus $a_s \in \adj(f_{j+1},G')$.

    Suppose that $a_r \neq f_j$. In this case, The walk $\gamma_f'(f_j,a_s)$ must be semi-directed since otherwise $\sigma_j$ would be anterior to $\sigma_{j-1}$. This implies that the walk $\gamma_f'(f_j,a_s)$ must contain an unshielded collider trisection not in $G$. This is a contradiction. Thus $a_r = f_j$.

    Both $\edge{l_{j-1}}{a_s}$ and $\edge{f_{j+1}}{a_s}$ must be into $a_s$ 
    by Lemma~\ref{lem:separable-diw:post-discriminated-collider-section:inducing-vertex}
    due to the fact that $l_{j-1}$ and $f_{j+1}$ are on sections of \diw\ $\gamma_d$ that discriminates $\sigma_j$.

    This implies that both $\edgeprime{l_{j-1}}{a_s}$ and $\edgeprime{f_{j+1}}{a_s}$ must be into $a_s$ as 
    $G$ and $G'$ have the same \miw s    
    and the walk $\seq{\edge{l_{j-1}}{a_s}, \edge{f_{j+1}}{a_s}}$ is a \miw.

\proofcase{1} There is no $h$ with $1 \leq h < j$ such that $l_h \arrowarrow a_s$ in $G'$ or $l_h \not \in adj(a_s,G')$.

In this case, $l_1 \tailarrow b$ and the walk $\omega(v_1,l_1) \append \seq{\edge{l_1}{a_s}, \edge{a_s}{f_{j+1}}}$ is an open walk given $Cb$ which is a contradiction.

\proofcase{2} There is an $h$ with $1 \leq h < j$ such that $l_h \arrowarrow a_s$ in $G'$ or $l_h \not \in adj(a_s,G')$. Choose $h$ be as large as possible such that $l_h \arrowarrow a_s$ in $G'$ or $l_h \not \in adj(a_s,G')$.

\proofcase{2.1} $l_h \arrowarrow a_s$. In this case, the walk $\omega(v_1,l_h) \append \seq{\edge{l_h}{a_s}, \edge{a_s}{f_{j+1}}}$ is an open walk given $Cb$ which is a contradiction.

\proofcase{2.2} $l_h \not \in \adj(a_s,G')$. In this case, the walk $\omega'(l_h,f_j)\append \gamma_d'(f_j, a_s)$ is a minimal inducing walk. By assumption, this walk must be in $G$, however, this cannot be the case as $a_s \in \ant(a_r,G)$ which leads to a contradiction by Lemma~\ref{lem:anterial-graph:minimal-inducing-walk:internal-not-anterial}.

In all cases, we have a contradiction and thus the lemma holds.
\end{proof}

\begin{lemma}\label{lem:separable-graphs:same-adj-and-same-miw:equiv}
    If two separable graphs $G$ and $G'$ have the same adjacencies and same \miw s then they are equivalent.
\end{lemma}

\begin{proof}
    Let $G$ and $G'$ be two separable graphs with the same adjacencies and same \miw s.
    let $H=\simplify(G)$ and $H'=\simplify(G')$.
    By Theorem~\ref{thm:simplify:separable-graph:equivalent-anterial-graph} we have that $G\equiv H$ and $H$ is a separable anterial graph, and $G'\equiv H'$ and $H'$ is a separable anterial graph. 
    From Lemmas~\ref{lem:separable-graphs:equivalent:same-vertex-separability-same-induced-arrowheads}, \ref{lem:separable-graphs:same-vertex-separability:same-adj}, and \ref{lem:separable-graph-and-same-adj:induced-arrowhead:same-miw},
    we have that $G$ and $H$ have the same adjacencies and same \miw s and $H'$ and $G'$ have the same adjacencies and same \miw s. Thus we also have that $H$ and $H'$ have the same adjacencies and same \miw s.
    The conclusion then follows from Lemma~\ref{lem:separable-anterial-graphs:same-adj-and-same-miw:equivalent} and the transitivity of equivalence.
\end{proof}

\subsection{\Miw s and \mdiw s}

We prove Proposition~\ref{prop:diw-for-section:section-ivs} and a lemma relating \miw s and \mdiw s.

\refproposition{prop:diw-for-section:section-ivs}

\begin{proof}
    Let $\omega$ with $\defwalkomega$ be a \diw\ in that discriminates collider section $\sigma$ on $\omega$ and let $\indep{v_1}{v_n}{C}{G}$.
    Neither $v_1$ nor $v_n$ is in $\IV(v_1,v_n,C,G).$
    Next consider a vertex $v_k$ not on $\sigma$ with $1<k<n$ and $i \neq j$.
    From the definition of \diw, $v_k \in \ant(\set{v_1,v_n},G)$.
    From ordered upward separation stability (Lemma~\ref{lem:indep-model-graph:ordered-upward}), we have that $\indep{v_1}{v_n}{Cv_k}{G}$ which implies that $v_k\not\in \IV(v_1,v_n,C,G).$
    Finally consider a vertex $v_l$ on $\sigma_j$. In this case, $\omega$ is open given $Cv_l$ as every collider section is either anterior to an endpoint or a vertex in $Cv_l$. This implies that $v_l\in \IV(v_1,v_n,C,G).$
\end{proof}

The next lemma relates \miw s and \mdiw s in separable graphs with the same adjacencies.

\begin{lemma}\label{lem:separable-graph:same-miw::same-mdiw}
    Two separable graphs with the same adjacencies have the same \miw s if and only if they have the same \mdiw s.
\end{lemma}

\begin{proof}
    Let $G$ and $G'$ be two separable graphs that have the same adjacencies.

    The forward direction follows from the fact that every \mdiw\ is a \miw.
 
    Next we consider the backward direction. Assume the lemma does not hold in the backward direction and let $G$ and $G'$ be two graph with the same adjacencies and the same \mdiw s but different \miw s. 
    Without loss of generality, assume there is a \miw\ $\omega$ with $\defwalkomega$ in $G$ such that there is no $\omega'$ with $\defwalkomegaprime$ that is a \miw\ in $G'$.
    Consider the \miw\ decomposition $\iwdecompose(\omega, G)= \set{\gamma_1, \ldots, \gamma_d}$ that contains a set of \mdiw s. There must be \mdiw s $\gamma_i'$ in $G'$. 
    Next consider the walk $\omega'=\iwcompose(\set{\gamma_1',\ldots,\gamma_d'}, G')$ created by combining these \mdiw s. $\omega'$ must be an inducing walk. If $\omega'$ is a \miw\ then we have a contradiction. If $\omega'$ is not a \miw\ then there must be a shorter \miw\ $\omega_m'$ in $G'$ that uses a subset of vertices on $\omega'$.
    Consider the decomposition $\iwdecompose(\omega_m', G')=\set{\beta_1',\ldots,\beta_e'}$. It must be the case that there are \mdiw s $\beta_i$ in $G$. Consider the walk $\omega_m=\iwcompose(\set{\beta_1,\ldots, \beta_e}, G)$. It must be an inducing walk between $v_1$ and $v_n$ that uses a subset of the vertices of $\omega$ and thus $\omega$ is not a \miw. This is a contradiction. Thus the backward direction must hold. 
    
    Both the forward and backward directions hold so the lemma must be true.
 \end{proof}

\subsection{Characterization theorems}

\begin{lemma}\label{lem:combined-chracterization}
    The following statements are equivalent for two separable graphs $G$ and $H$:

    (1) $G$ and $H$ are separation equivalent,

    (2) $G$ and $H$ have the same vertex separability and same \iarrowhead s,
    
    (3) $G$ and $H$ have the same adjacencies and same \miw s, and

    (4) $G$ and $H$ have the same adjacencies and same \mdiw s.

\end{lemma}

\begin{proof}
    The implication $(1)\implies(2)$ follows from Lemma~\ref{lem:separable-graphs:equivalent:same-vertex-separability-same-induced-arrowheads}.
    The implication $(1) \text{ and } (2)\implies(3)$ follows from Lemmas~\ref{lem:separable-graphs:same-vertex-separability:same-adj} and \ref{lem:separable-graph-and-same-adj:induced-arrowhead:same-miw}.
    The implication $(3)\implies(1)$ follows from Lemma~\ref{lem:separable-graphs:same-adj-and-same-miw:equiv}.
    Finally the bi-implication $(3) \iff (4)$ follows from Lemma~\ref{lem:separable-graph:same-miw::same-mdiw}.
    Thus the lemma is true.
\end{proof}

\reftheorem{thm:separable-graphs:equivalent::same-adj-and-same-miw}
\begin{proof}
    Follows from Lemma~\ref{lem:combined-chracterization}.
\end{proof}

\reftheorem{thm:separable-graphs:equivalent::same-adj-same-mdiw}

\begin{proof}
    Follows from Lemma~\ref{lem:combined-chracterization}.
\end{proof}

\reftheorem{thm:separable-graphs:equivalent::same-adj-same-inducing-arrowheads}

\begin{proof}
    Follows from Lemma~\ref{lem:combined-chracterization}.
\end{proof}

\TextVersion{}{
\section{Proof related to testing equivalence of separable graphs} In this subsection we prove the correctness of our equivalence testing algorithm.
\vspace{0.2in}

\reftheorem{thm:equivalence:correct}

\begin{proof}
    From Theorem~\ref{thm:simplify:separable-graph:equivalent-anterial-graph} we replace $G$ with $\simplify(G)$ and $H$ with $\simplify(H)$ as the \simplify\ algorithm preserves equivalence. The correctness then follows from Theorem~\ref{thm:separable-graphs:equivalent::same-adj-same-mdiw}.
\end{proof}
}

\section{Proofs related representing equivalence classes of separable graphs}

In this section we show that the \inducedarrowheads\ function is a canonical projection for separable graphs with respect to separation equivalence and that \inducedarrowheads\ preserves separation equivalence when applied to separable graphs.

\begin{lemma}\label{lem:induced-arrowhead:same-adj}
    The graphs $G$ and $\inducedarrowheads(G)$ have the same adjacencies.
\end{lemma}

\begin{proof}
    The \inducedarrowheads\ algorithm applied to $G$ adds an edge if and only if there is an edge in $G$ and thus the graphs have the same adjacencies.
\end{proof}

 By \emph{removing arrowhead} at $i$ on the edge $e(i,j)$ we mean turning $j\arrowarrow i$ into $j\arrowtail i$ or turning $j\tailarrow i$ into $j\tailtail i$. Removing arrowhead at $i$ and $j$ turns $j\arrowarrow i$ into $j\tailtail i$. This process of removing arrowheads is how the 
 \inducedarrowheads\ algorithm transforms  $G$ into  $\inducedarrowheads(G)$.

\begin{lemma}\label{lem:minimal-inducing-induced-arrow}
All internal arrowheads on a minimal inducing walk in a separable graph are induced arrowheads.
\end{lemma}
\begin{proof}
Consider a minimal inducing walk $\gamma$ in a separable graph $G$. We show that all internal arrowheads on $\gamma$ are induced. Choose an arbitrary internal arrowhead at $b$ on $e(b,d)$ on $\gamma$. Denote the section containing $b$ by $\sigma$. By Lemma~\ref{lem:separable:miw:every-internal-section-discriminated}, there is an \diw\ $\pi$ that discriminates $\sigma$.  Denote the endpoints of $\pi$ by $i$ and $j$.
Because $\pi$ is a \diw, $i\not\in \adj(j,G).$ From Theorem~\ref{thm:vertex-separable:pairwise-anterial-separable}, we have that $\indep{i}{j}{C}{G}$ where $C=\propantij$. In addition, again due to the fact that $\pi$ is a \diw, $\dep{i}{j}{Cb}{G}$. These imply that 
$b\in \IV(i,j,C,G).$ Since $d\in C,$ we also have $d\not\in IV(i,j,C,G)$, which then implies that the arrowhead at $b$ on $\edge{b}{d}$ is an induced arrowhead.
\end{proof}

\begin{lemma}\label{lem:same-indu-walk-g- and-arrowheadsg-new}
Assume $G$ is a separable graph. Then a minimal inducing walk in $G$ is an inducing walk in $\inducedarrowheads(G)$.
\end{lemma}
\begin{proof}
Consider a minimal inducing walk $\gamma$ in separable graph $G$.
From Lemma~\ref{lem:induced-arrowhead:same-adj} there is a walk $\gamma'$ in $\inducedarrowheads(G)$ with the same vertex sequence as $\gamma$. Again from Lemma~\ref{lem:induced-arrowhead:same-adj}, the endpoints of $\gamma'$ are not adjacent. This fact, Lemma \ref{lem:minimal-inducing-induced-arrow}, and the fact that the \inducedarrowheads\ algorithm only remove arrowheads that are not induced arrowheads imply that $\gamma'$ is an inducing walk. 
\end{proof}

\begin{lemma}\label{lem:equiv-min-ind-walk-in-inducedarr}
If $G$ is separable and $\gamma'$ is a \miw\ in the graph $\inducedarrowheads(G)$ then there exists a walk $\gamma$ in $G$ such that $\gamma' \equiv \gamma$.
\end{lemma}
\begin{proof}
Denote the endpoints of $\gamma'$ by $i$ and $j$. By  Lemma~\ref{lem:induced-arrowhead:same-adj} we have that there is  walk $\gamma$ in $G$ with $\vertexseq{\gamma}=\vertexseq{\gamma'}$.  

If we show that all internal arrowheads on $\gamma$ are induced arrowheads in $G$ then we are done since $\inducedarrowheads(G)$ only removes arrowheads from $G$. 
Suppose, for contradiction, that an internal arrowhead on $\gamma$ at vertex $b$ on $e(d,b)$ is not induced.

We note that on $\gamma'$, the edge $e(d,b)$ is undirected  since $\gamma'$ is an inducing walk. Without loss of generality, we assume $b$ is the closest vertex with a removed arrowhead to the endpoint of its section on its side on $\gamma'$, i.e.\ $b$ is such that on the subsection between $b$ and an endpoint of the section that do not contain $d$, the is no arrowhead removed. Consider the endpoint of this section that is close to $b$ by $b'$, and denote the vertex on $\gamma'$ outside this section that is adjacent to $b'$ by $d'$. If $b=b'$, we choose $d'$ to be the other adjacent vertex of $b$ than $d$. 

Notice that there exists at least a $C$ such that $\indep{i}{j}{C}{G}$ because $i$ and $j$ are not adjacent and $G$ is separable. Since the arrowhead at $b$ is not induced, by the contrapositive of the definition of induced arrowheads, we have that, for every $i,j,C$, such that $\indep{i}{j}{C}{G}$, either (i) $\indep{i}{j}{Cb}{G}$ or (ii) $\dep{i}{j}{Cb}{G}$ and $\dep{i}{j}{Cd}{G}$. 

We show that the arrowhead at $b'$ on $e(b',d')$ is not induced: Since $b$ and $b'$ are in the same section we have that $\indep{i}{j}{Cb}{G}\iff \indep{i}{j}{Cb'}{G}$. This implies that if case (i) holds for $b$, the analogous of case (i) holds for $b'$. In addition, if case (ii) holds for $b$ then it holds that $\dep{i}{j}{Cb'}{G}$.  

What is left to show is if case (ii) holds for $d$ then $\dep{i}{j}{Cd'}{G}$: Since $G$ is separable, by Lemma \ref{lem:minimal-inducing-induced-arrow}, there is an edge between $d$ and $d'$. This edge cannot have an arrowhead at $d$ since otherwise $\gamma'$ is not the shortest self-inducing walk between $u$ and $v$. Therefore, regardless of whether there is an arrowhead at $d'$ on $e(d,d')$, if there is a connecting walk between $i$ and $j$ given $Cd$, there is a  connecting walk between $i$ and $j$ given $Cd'$. 

This completes the proof of the arrowhead at $b'$ on $e(b',d')$ is not induced, and consequently, there is no arrowhead at $b'$ on this edge in $\inducedarrowheads(G)$. This is a contradiction since $b'$ is an endpoint of a section on $\gamma'$. Therefore, all internal arrowheads on $\gamma$ are induced.

\end{proof}

\begin{lemma}\label{lem:same-indu-walk-g- and-arrowheadsg}
Assume $G$ is a separable graph. A minimal inducing walk in $G$ is a minimal inducing walk in $\inducedarrowheads(G)$ and vice versa.
\end{lemma}
\begin{proof}
Let $\gamma$ be a minimal inducing walk in $G$.
Lemma \ref{lem:same-indu-walk-g- and-arrowheadsg-new} shows there is an inducing walk $\gamma' \equiv \gamma$ in $\inducedarrowheads(G)$. Suppose $\gamma'$  is not minimal.
This implies that there is another walk $\omega$ that is minimal in $\inducedarrowheads(G)$.
From Lemma \ref{lem:equiv-min-ind-walk-in-inducedarr}, there is an inducing walk in $G$ that uses a subset of the vertices in $\gamma$. 
This contradicts $\gamma$ being minimal. Thus $\gamma'$  must be minimal inducing walk in $\inducedarrowheads(G)$.

The other direction is proven with a similar argument.

\end{proof}

\begin{lemma}\label{lem:sep-to-sep-inducedarrows}
If $G$ is separable then $\inducedarrowheads(G)$ is separable.
\end{lemma}
\begin{proof}
By Theorem~\ref{thm:separable-graph::no-self-inducing-walks}, $G$ does not contain a self-inducing walk, and we need to show that $\inducedarrowheads(G)$ does not contain a self-inducing walk. Suppose, for contradiction, that there is a self-inducing walk $\gamma'$ in $\inducedarrowheads(G)$. Assume, without loss of generality, that $\gamma'$ is the shortest such walk, and denote the endpoints of $\gamma'$ by $u$ and $v$. This implies that $\gamma'$ is a \miw.

  We note, by Lemma~\ref{lem:induced-arrowhead:same-adj}, that $G$ and  $\inducedarrowheads(G)$ have the same adjacencies. Hence, each walk $\nu'$ in $\inducedarrowheads(G)$  has a corresponding walk with the same vertices and edges, denoted by $\nu$, in $G$ on which some arrowheads are potentially removed on $\nu'$.  Also, each internal vertex of $\gamma'$ is an anterior to one of its endpoints. For each vertex $w$ on $\gamma'$, consider the anterial walk from $w$ to an endpoint of $\gamma'$, denoted by $\nu'_w$, such that  $\nu'_w$ has the fewest number of arrowheads removed among such anterior walks.    Consider the subgraph of $\inducedarrowheads(G)$ consisting of the vertices of $\gamma'$ and the vertices of all the walks $\nu'_w$ for every internal vertex $w$ of $\gamma'$.   Call this subgraph $K'$. 

Below we show that in the induced subgraph $K$ of $G$ with respect to the vertices of $K'$, there exists an arrowhead at $b$ on $e(b,d)$ in $K$ that does not exist in $K'$, but is induced. This leads to a contradiction with the arrowhead not existing in $K'$.

 Again, we note, by Lemma~\ref{lem:induced-arrowhead:same-adj}, that $K$ and $K'$ have the same adjacencies, and any arrowhead in $\inducedarrowheads(G)$ exists in $G$. Since $\gamma'$ is a \miw, by Lemma \ref{lem:same-indu-walk-g- and-arrowheadsg}, there exists a walk $\gamma$ in $G$ with the same vetices that is a \miw. 
Since all arrowheads on $\gamma$ are induced in $G$, as proven in Lemma \ref{lem:minimal-inducing-induced-arrow}, the arrowhead at $b$ must be on an edge in $K'$ that is not on $\gamma$. 


If vertices at which the arrowhead is removed are in $\propant(\set{u,v},G)$ then $\gamma$ is self-inducing in $G$, which is not possible. Hence, choose $e(b,d)$ such that the arrowhead at $b$ is removed and $b\notin\propant(\set{u,v},G)$.

Consider a vertex $i$ in $\gamma$ that is an anterior of $b$ in the anterior walk from $i$ to $\{u,v\}$  in $\inducedarrowheads(G)$ ($i$ can be $b$ itself). Denote the section of $i$ on $\gamma$ by $\rho$. Let $C=\propant(\{u,v\},G)\cup\vertexseq{\gamma}\setminus \rho$, i.e, be all proper anteriors of $u$ and $v$ and vertices of $\gamma$ except the vertices in the section that is an anterior of $b$. 

First, $\indep{u}{v}{C}{G}$ since otherwise $i\in\propant(\{u,v\},G)$, which implies that $b\in\propant(\{u,v\},G)$. 

Secondly, let $b$ be the closest vertex to $i$.  Clearly, $\dep{u}{v}{Cb}{G}$.

Thirdly, let $b$ be the closest vertex to $u$.  The only option where $d$ can be on a connecting walk between $u$ and $v$ given $Cd$ is via the anterior walk from $d$ to $u$. However $d$ is on a non-collider section in this walk and cannot make a connecting walk. Therefore, $\indep{u}{v}{Cd}{G}$. 

These imply that the arrowhead at $b$ is indeed induced, which is a contradiction.
\end{proof}

\reftheorem{thm:separable:inducedarrowheads-equivalent}

\begin{proof}
Since, by Lemma \ref{lem:sep-to-sep-inducedarrows}, $G$ and $\inducedarrowheads(G)$ are separable, by Theorem \ref{thm:separable-graphs:equivalent::same-adj-and-same-miw}, it is enough to show that $G$ and $\inducedarrowheads(G)$ have the same adjacencies and minimal inducing walks. The former is Lemma~\ref{lem:induced-arrowhead:same-adj} and the latter is Lemma \ref{lem:same-indu-walk-g- and-arrowheadsg}. 
\end{proof}

\reftheorem{thm:induced-arrowheads:sound:identifying}

\begin{proof}
    If two graphs $G$ and $G'$ are separation equivalent and separable then they must have the same adjacencies and induced arrowheads by Theorem~\ref{thm:separable-graphs:equivalent::same-adj-same-inducing-arrowheads}. This implies that $\inducedarrowheads(G)=\inducedarrowheads(G')$. If $G'\not \equiv G$ then, again by Theorem~\ref{thm:separable-graphs:equivalent::same-adj-same-inducing-arrowheads}, they must have different adjacencies or induced arrowheads with implies that $\inducedarrowheads(G) \not \equiv \inducedarrowheads(G')$. This proves that \inducedarrowheads\ is a canonical representation function for separable graphs. Then, by Lemma~\ref{lem:sep-to-sep-inducedarrows} and Theorem~\ref{thm:separable:inducedarrowheads-equivalent}, it follows that \inducedarrowheads\ also a canonical projection function.
\end{proof}

\cut{
\begin{proof}
    The \inducedarrowheads is separation perfect by Lemma~\ref{lem:induced-arrowhead:same-adj} and is induced arrowhead perfect because the arrowheads in $\inducedarrowheads(G)$ correspond precisely to the induced arrowheads in $G$. Thus by Lemma~\ref{lem:graph-representation:separation-perfect:induced-arrowhead-perfect:canonical}, \inducedarrowheads\ is a canonical representation function for separation equivalence of separable graphs.
\end{proof}
}

\refcorollary{cor:largest}

\begin{proof}
    Let $G$ be a separable graph. Every graph in $\largest(G)$ is separable and equivalent to $G$. Let $H,H' \in \largest(G)$. 
    By Theorem~\ref{thm:separable-graphs:equivalent::same-adj-same-inducing-arrowheads}, $H$ and $H'$ have the same adjacencies and same induced arrowheads. 
    Consider the graph $\inducedarrowheads{G}$. It has the same adjacencies and induced arrowheads. Furthermore, all of its arrowheads are induced arrowheads so it must be in $\largest(G)$. 
    If any of the edges in $H$ or $H'$ have an arrowhead that is not in $\inducedarrowheads{G}$ then it is not in $\largest(G)$ thus $H=H'=\inducedarrowheads{G}$.
\end{proof}

\refcorollary{cor:separable:induced-arrowhead-anterial}

\begin{proof}
    Let $G$ be a separable graph. From Lemma~\ref{lem:sep-to-sep-inducedarrows}, $G'=\inducedarrowheads(G)$ is separable and from Theorem~\ref{thm:separable:inducedarrowheads-equivalent}, $G\equiv G'$.
    
    From Theorem~\ref{thm:simplify:separable-graph:equivalent-anterial-graph}, $G''= \simplify(G')$ is an anterial graph equivalent to $G'$. Thus we also have, $G''\equiv G$. From the fact that the \simplify\ algorithm only removes arrowheads and the fact that these graphs are equivalent it must be the case that $G'=G''$ and thus $G'$ is an anterial graph.
\end{proof}
\section{Proofs related to the identification of essentially separable graphs}

\refproposition{prop:minimal:anterior}

\begin{proof}
From Lemma~\ref{lem:indep-model-graph:ordered-downward}, if $C$ is a minimal separating set for $i$ and $j$ in graph $G$ then $C\subseteq \ant(\set{i,j},G)$. Thus any minimal separating set is an anterior separating set.
\end{proof}

\begin{lemma}\label{lem:induced-arrowhead:d-in-anterior-of-some-separating-pair}
    If an edge $\edgedb$  has an induced arrowhead at $b$ in graph $G$ then there exist vertices $i$ and $j$ and vertex set $C$ such that $b\in \IV(i,j,C,G)$ and $d\in \antsetij$.
\end{lemma}

\begin{proof}
    let $\edgedb$ have an induced arrowhead at $b$ in a graph $G$.
    From the definition of induced arrowheads, there is some pair of vertices $p$ and $q$ and vertex set $D$ such that $b\in \IV(p,q,D,G)$ and $d\not \in \IV(p,q,D,G)$.
    From $b\in \IV(p,q,D,G)$, we have that $\indep{p}{q}{D}{G}$ and $\dep{p}{q}{Db}{G}$.

    \proofcase{1} $d\in \ant(\set{p,q},G)\setminus \set{p,q}$. In this case, the lemma holds for vertices $i=p$ and $j=q$ and vertex set $C=D$ and note that using this assignment we have $b\in \IV(i,j,C,G)$ and $d\in \ant(\set{i,j},G)$. Thus the lemma holds in this case.
    
    \proofcase{2} $d\not\in \ant(\set{p,q},G)\setminus \set{p,q}$.
    From $d\not\in \IV(p,q,D,G)$, it must be the case that $\indep{p}{q}{Dd}{G}$.
    From $\indep{p}{q}{D}{G}$ and $\indep{p}{q}{Dd}{G}$, and weak transitivity, it must be the case that $\indep{p}{d}{D}{G}$ or $\indep{d}{q}{D}{G}$.
    Without loss of generality, assume that $\indep{d}{q}{D}{G}$.
    From $\indep{p}{q}{D}{G}$, $\indep{d}{q}{D}{G}$, and composition we have $\indep{dp}{q}{D}{G}$ and, by weak union we have $\indep{d}{q}{Dp}{G}$.
    From $\indep{p}{q}{D}{G}$ and $\dep{p}{q}{Db}{G}$ and   
    Lemma~\ref{lem:inducing-vertex:walks-into-inducing-vertex}, there is a walk $\gamma_q$ between $b$ and $q$ open given $Dp$ and weakly into $b$.
    In this case, the walk $\seq{\edge{d}{b}}\append \gamma_q$ is open given $Dpb$ which implies that 
    $\dep{d}{q}{Dpb}{G}$. 
    Thus we have shown $\indep{d}{q}{Dp}{G}$ and $\dep{d}{q}{Dpb}{G}$ which implies that $b\in IV(d,q,Dp,G)$. In this case, let $i=d$, $j=q$ and $C=Dp$ and note that using this assignment we have $b\in \IV(i,j,C,G)$ and $d\in \ant(\set{i=d,j},G)$. Thus the lemma holds in this case.
\end{proof}

\subsection{Proof of anterior induced adjacency proposition} Next we prove Proposition~\ref{prop:induced-adj-separation}

\refproposition{prop:induced-adj-separation}

\begin{proof}
    Assume that that claim is not true and that there are separable vertices $a$ and $b$ in graph $G$ such that $\dep{a}{b}{C}{G}$ where $C=\iadj(a,b,G)\cap \ant(ab,G)=\aiadj(a,b,G)$. This implies that there is a connecting walk $\omega$ with $\defwalkomega$ with $v_1=a$ and $v_n=b$ that is open given $C$. We consider three possible types of walks for $\omega$.

    \proofcase{1} $\omega$ consist of a single edge $\edge{a}{b}$. In this case, by Lemma~\ref{lem:adjacent:not-vertex-separable}, $a$ and $b$ are not separable. This is a contradiction with the assumption that $a$ and $b$ are separable.
    
    \proofcase{2} $\omega$ is a collider walk. In this case, every collider section on $\omega$ must be anterior to the endpoints and thus is a \siw. This implies, by Theorem~\ref{lem:siw:endpoints-not-separable}, that $a$ and $b$ are not separable. This is again a contradiction. 

    \proofcase{3} $\omega$ is not a collider walk.
    In this case there is some non-trivial maximal trek. let $\omega(v_{i-1},v_{j+1})$ with $i\leq j$ be the maximal trek closest to $a=v_1$ on $\omega$. 
    By Lemma~\ref{lem:open-walk-given-C:every-vertex-in-anterior}, we have that $v_i \in \ant(abC,G)$. Furthermore, from $C=\aiadj(a,b,G) \subseteq \ant(\set{a,b},G)$ we have $v_i \in \ant(\set{a,b},G)$. 
    Next, consider the subwalk $\omega(a,v_i)$. The subwalk is a collider walk. From the fact that $\omega$ is a connecting walk given $C$, every collider section on the subwalk contains a vertex of $C$. From $C=\aiadj(a,b,G) \subseteq \ant(\set{a,b},G)$, every collider section is anterior to $\set{a,b}$ and thus $v_i \in \iadj(a,b,G)$.
    Thus $v_i\in C$ and $\omega$ is not a connecting walk as a non-collider section contains a vertex in $C$. This is a contradiction.

    In each case we have a contradiction so the claim must hold.
\end{proof}

\subsection{Ordering and separation properties of anterial graphs}
Anterial graphs play an essential role in proving the correctness of the \sgi\ algorithm. In this section we show several ordering and separation properties of anterial graphs.

\newcommand{\lG}{<_G}

Recall that $\oleqG$ is the induced preorder of a graph $G$ from which we define (1) $a \oeqG b$ if $a \oleqG b$ and $b \oleqG a$, (2) $a \oincompG b$ if $a \onotleqG b$ and $b \onotleqG a$, (3) $a \lG b$ if $a \oleqG b$ and $b \onotleqG a$. 
If $a \oincompG b$ we say that $a$ and $b$ are \emph{incomparable}. If $a \lG b$ or $b \lG a$ then $a$ and $b$ are \emph{strictly ordered}.
Recall also, the equivalence relations $a \eqIG b$ defined to hold if there is an undirected walk between $a$ and $b$ in $G$ or $a = b$.
Note that $\oeqG$ need not be equivalent to $\eqIG$ in general, but are equivalent if $G$ is an anterial graph due to Lemma~\ref{lem:acyclic:order-equiv:undirected-walk}.

The next lemma shows that in an anterial graph every distinct pair of vertices is either connected by an undirected walk, strictly ordered, or incomparable. 

\begin{lemma}\label{lem:anterial:disjoint-relations}
If $G$ is an anterial graph and $a \neq b$ then exactly one of the following properties holds $a \oincompG b$, $a \olessG b$, $b \olessG a$,  $a \eqIG b$.
\end{lemma}

\begin{proof}
    Let $\defgraphG$ be an anterial graph and for $a,b\in V$ it is the case that $a\neq b$. 

    \proofcase{1} $a \oincompG b$. In this case, by definition, we have that $a \neqIG b$, $a \onlessG b$, and $b \onlessG a$ and the property is exclusive.

    \proofcase{2} $a \eqIG b$. First, by definition, we have that $a \ocompG b$. 
    Next, due the fact that $a \eqIG b$ and that $G$ contains no semi-directed cycle, we must have $a \onlessG b$ and $b \onlessG a$. In this case, the property is exclusive.

    \proofcase{3} $a \olessG b$ or $b \olessG a$. \Wlog, assume $a \olessG b$. In this case, by definition, we have that $a \ocompG b$. It cannot be the case that $a \eqIG b$ or $b \olessG a$ otherwise, in either case, there would be a semi-directed cycle and $G$ contains no semi-directed cycle. In this case, the property is exclusive.
    
    Thus exactly one of the following holds for any distinct pair of vertices $a,b\in V$, $a \oincompG b, a\eqIG b, a\olessG b$, and $b \olessG a$ and thus the proposition holds.
\end{proof}

The next lemma, first proved in \cite{Sadeghi2017}, shows that the relationships between adjacent vertices in an anterial graph uniquely determines the edge type between those vertices.

\begin{lemma}\label{lem:anterial:adjacent:order-relations::edge-type}
If $G$ is anterial and $a \in \adj(b,G)$ then
\begin{list}{}{\leftmargin=4em \rightmargin=2em}
\item
    \begin{enumerate}
        \item [(i)] $a\eqIG b$ if and only if $a \tailtail b$,
        \item [(ii)] $a \oincompG b$ if and only if $a \arrowarrow b$,
        \item[(iii)]  $a \olessG b$ if and only if $a \tailarrow b$, and
        \item[(iv)] $b \olessG a$  if and only if $b \tailarrow a$.
    \end{enumerate}    
\end{list}
    
\end{lemma}

\begin{proof}
Let $G$ be an anterial graph such that $a\in \adj(b,G)$.  

\proofcase{(i)} For the forward direction, let $a \eqIG b$. If $a \tailarrow b$, or $a \arrowtail b$ then there would be a semi-directed walk which contradicts $G$ being anterial. Similarly, if $a \arrowarrow b$, then there would be a bidirected edge with an endpoint anterior to its other endpoint which also contradicts $G$ being anterial. Thus the edge must be oriented $a \tailtail b$. For the other direction, assume $a \tailtail b$. In this case, the walk containing this edge is an undirected walk between $a$ and $b$ and thus $a\eqIG b$.

\proofcase{(ii)} For the forward direction, let $a \oincompG b$. By Lemma~\ref{lem:anterial:disjoint-relations}, we have that $a$ and $b$ are not ordered which rules out $a \tailarrow b$ and $a \arrowtail b$. Again, by Lemma~\ref{lem:anterial:disjoint-relations}, we have that $a \neqIG b$ which rules out $a \tailtail b$. This implies that $a \arrowarrow b$. For the other direction, let $a \arrowarrow b$. If either $a \olessG b$, $b \olessG a$ or $a \eqIG b$ there would be a bidirected edge with an endpoint anterior to its other endpoint which contradicts $G$ being anterial. Thus $a \oincompG b$.

\proofcase{(iii)} For the forward direction, let $a \olessG b$. By Lemma~\ref{lem:anterial:disjoint-relations}, we have that $b \onlessG a$ and $a \neqIG b$ which rules out $a \arrowtail b$ and $a \tailtail b$. If $a \arrowarrow b$ then there would be a bidirected edge with an endpoint anterior to its other endpoint which cannot be the case due to $G$ being anterial. This implies that $a \tailarrow b$
For the other direction, let $a \tailarrow b$. This implies that $a \olessG b$.

\proofcase{(iv)} Similar to case (iii).
\end{proof}


\newcommand{\antadj}{\DefineConstant{aAdj}}
Let $\antadj(b,G)$ be the set of \emph{anterior adjacent} vertices defined to be the set of vertices $c\in V$ such that either $c\tailarrow b$ or $c \tailtail b$ in graph $G$.

\begin{lemma}\label{lem:anterial:eqIG:anterior-adjacent-separate}
If $G$ is an anterial graph containing separable vertices $a$ and $b$ such that $a \eqIG b$ then $\indep{a}{b}{\antadj(a,G)}{G}$.
\end{lemma}

\begin{proof}
    Let $G$ be an anterial graph containing separable vertices $a$ and $b$ such that $a \eqIG b$. Suppose that the claim does not hold and that $\dep{a}{b}{C}{G}$ where $C=\antadj(a,G)$. This implies that there is a walk $\omega$ with $\defwalkomega$ and $\defsectionomega$ where $v_1=a$ and $v_n=b$ that is connecting given $C$. Note that $v_2 \neq b$, otherwise $a$ and $b$ are not separable.

    From Lemma~\ref{lem:anterial:disjoint-relations}, it must be the case that $a \onlessG b$, $b \onlessG a$, and $a \ocompG b$.

    \proofcase{1} $a \tailarrow v_2$ or $a\arrowarrow v_2$. From anterial endmark property, Lemma~\ref{lem:simple:anterial-graph::anterial-endmark}, we have that $v_2 \not\in \ant(a, G)$. As $a\eqIG b$ we also have $v_2 \not\in \ant(b,G).$
    There must be some collider section $\sigma_i$ for $1<i<m$ as $\sigma_2 \not\in \ant(ab,G)$ and this section cannot contain a vertex in $C$. This implies that the walk $\omega$ is not a connecting walk given $C$.
    
    \proofcase{2} $a \tailtail v_2$ or $a \arrowtail v_2$. In either case, $v_2$ is on a non-collider section of $\omega$ and $v_2\in C$. This implies that the walk $\omega$ is not a connecting walk given $C$.

    In either case, the walk $\omega$ is not a connecting walk given $C$ so the claim holds.
\end{proof}

\begin{lemma}\label{lem:anterial:ordered:anterior-adjacent-separate}
If $G$ is an anterial graph containing separable vertices $a$ and $b$ such that $b \olessG a$ then $\indep{a}{b}{\antadj(a,G)}{G}$.
\end{lemma}

\begin{proof}
    Let $G$ be an anterial graph containing separable vertices $a$ and $b$ such that $b \olessG a$. Suppose that the claim does not hold and that $\dep{a}{b}{C}{G}$ where $C=\antadj(a,G)$. This implies that there is a walk $\omega$ with $\defwalkomega$ and $\defsectionomega$ where $v_1=a$ and $v_n=b$ that is connecting given $C$. Note that $v_2 \neq b$, otherwise $a$ and $b$ are not separable.

    From Lemma~\ref{lem:anterial:disjoint-relations}, it must be the case that $a \onlessG b$, $a \neqIG b$, and $a \ocompG b$.

    We show $\indep{a}{b}{\antadj(a,G)}{G}$.

    \proofcase{1} $a \tailarrow v_2$ or $a\arrowarrow v_2$. From the anterial endmark property, Lemma~\ref{lem:simple:anterial-graph::anterial-endmark}, we have that $v_2 \not\in \ant(a, G)$. As $b \olessG a$ and $G$ is anterial, we also have $v_2 \not\in \ant(b,G).$
    There must be some collider section $\sigma_i$ for $1<i<m$ as $\sigma_2 \not\in \ant(ab,G)$ and this section cannot contain a vertex in $C$. This implies that the walk $\omega$ is not a connecting walk given $C$.
        
    \proofcase{2} $a \tailtail v_2$ or $a \arrowtail v_2$. In either case, $v_2$ is on a non-collider section of $\omega$ and $v_2\in C$. This implies that the walk $\omega$ is not a connecting walk given $C$.

    In either case (1) or (2), the walk $\omega$ is not a connecting walk given $C$ so the claim must hold.
\end{proof}

\subsection{Proofs related to the \sgi\ algorithm}

\newcommand{\Gfinal}{G_{final}}

\begin{lemma}\label{lem:perfect-testing:remove-only-separable}
    If $\Gfinal=\sgi(\imodel{\genstruct})$ and $a\not \in \adj(b,\Gfinal)$ then $a$ and $b$ are separable in $\genstruct$.
\end{lemma}

\begin{proof}
    Let $\Gfinal=\sgi(\imodel{\genstruct})$ and $G$ denote that current graph at various steps of running the \sgi\ algorithm. 
    Initially, $G$ is the complete graph. An edge $\edge{a}{b}$ is removed only if a set $C$ is found and $\indep{a}{b}{C}{\imodel{\gendist}}$. This implies that $a$ and $b$ are separable in $\genstruct$.
\end{proof}

\begin{lemma}\label{lem:anterior-preserving}
    If $\Gfinal=\sgi(\imodel{\genstruct})$ then for all $a\in V$, 
    $\ant(a,\Gfinal) \supseteq \ant(a,\genstruct)$.
\end{lemma}

\begin{proof}
    Let $\Gfinal=\sgi(\imodel{\genstruct})$ and $G$ denote that current graph at various steps of running the \sgi\ algorithm. 
    Initially, $G$ is the complete graph and $\ant(i,G)=V\setminus i$ so the claim holds. The claim then follows from the soundness of the orientation rule used to add arrowheads to edges (Proposition~\ref{prop:induced-arrowhead:exclusive-arrowhead}).
\end{proof}

\begin{lemma}\label{lem:anterior-adjacency-preserving}
    If $\Gfinal=\sgi(\imodel{\genstruct})$ then for all $a\in V$, $\antadj(a,\genstruct) \subseteq \antadj(a, \Gfinal)$.
\end{lemma}

\begin{proof}
    This follows from Lemmas~\ref{lem:perfect-testing:remove-only-separable} and \ref{lem:anterior-preserving}.
\end{proof}

The following lemma shows that every vertex equivalence class of $\genstruct$ is a subset of some vertex equivalence class in $\Gfinal$, the output of the \sgi\ algorithm.

\begin{lemma}\label{lem:equiv-implies-tails-in-Gfinal}
    If $a \oeq_{\genstruct} b$, and $a\in \adj(b,\Gfinal)$ where $\Gfinal= \sgi(\imodel{\genstruct})$ then  $a \tailtail b$ in $\Gfinal$.
\end{lemma}

\begin{proof}
    Assume $a \oeq_{\genstruct} b$, $\Gfinal=\sgi(\imodel{\genstruct})$ and $a\in \adj(b,\Gfinal)$. Suppose the claim is not true and that the edge $\edge{a}{b}$ in $\Gfinal$ has an arrowhead at $a$, $b$ or both $a$ and $b$. \Wlog, assume there is one at $a$. The edge must have an induced arrowhead at $a$ in $\imodel{\gendist}$. This implies that there exists $\triple{i}{j}{C}$ such that $a\in IV(i,j,C,\genstruct)$. This, however, cannot be the case due to Proposition~\ref{prop:inducing-equivalent:one-is-inducing-vertex:each-is-inducing-vertex}. Thus the claim must be true.
\end{proof}

\begin{lemma}\label{lem:remove-cover-edge-of-diw:identify-section}
    If $\omega$ with $\defwalkomega$ is a \diw\ that discriminates collider section $\sigma=\seq{v_i,v_j}$ in $\genstruct$  then if \sgi\ removed the edge $\edge{v_1}{v_n}$ then the edge $\edgevIJ{i-1}{i}$ has an arrowhead at $v_i$ in $\Gfinal$ and the edge $\edgevIJ{j}{j+1}$ has an arrowhead at $v_j$ in $\Gfinal$.
\end{lemma}

\begin{proof}
    Let $\omega$ with $\defwalkomega$ be a \diw\ that discriminates collider section $\sigma=\seq{v_i,v_j}$ in $\genstruct$ and assume that \sgi\ removed the edge $\edge{v_1}{v_n}$.
    The fact that the edge $\edgevIJ{1}{n}$ was removed implies that there is a set $C$ that separates $v_1$ and $v_n$ in $\genstruct$ that was used by the algorithm to remove the edge.
    From Proposition~\ref{prop:diw-for-section:section-ivs}, every vertex on a discriminated section of a \diw\ is an inducing vertex for the endpoints of the \diw\ (i.e., in $\IV(v_1,v_n,C,\genstruct)$)  and that every other vertex on $\omega$ is not an inducing vertex (i.e., not in $\IV(v_1,v_n,C,\genstruct)$). This implies that the edge closed to $\sigma$ on the walk between $a$ and $\sigma$ has an induced arrowhead at $\first(\sigma)$. 
    This implies that the \sgi\ algorithm will add these arrowheads after removing the edge $\edgevIJ{1}{n}$.
\end{proof}

The next lemma captures a key algorithmic invariant about the \sgi\ algorithm. In particular, the invariant captures the fact that the current graph $G$ in \sgi\ before and after considering candidate separating sets is an anterial graph.

\begin{lemma}\label{lem:current-graph-always-anterial}
    For graph $\genstruct$, before and after considering a candidate separating set, the current graph $G$ of the \sgi\ algorithm applies to $\genstruct$ (i.e. a trace of $\sgi(\genstruct)$) is an anterial graph.
\end{lemma}

\begin{proof}
    We show this algorithmic invariant inductively on the length of the trace of $\sgi(\genstruct)$.
    
    For the base case we consider the initial graph. This initial graph is a simple complete undirected graph and thus satisfies the anterial endmark property. By Lemma~\ref{lem:simple:anterial-graph::anterial-endmark} the graph is anterial.
    
    For the inductive case we assume that the current graph is anterial and show that after identifying a separating set $C$ for $i$ and $j$, removing the edge between $i$ and $j$ and orienting, the resulting graph is again anterial. First note that the resulting graph is simple.
    Consider the set $\IV=\IV(i,j,C, \genstruct)$ and $D=V \setminus \IV$. Any new endmark is at an endpoint of an edge $\edge{c}{d}$ with $c\in D$ and $d \in \IV$ with the new arrowhead at $d$. The resulting graph satisfies the anterial endmark property as (1) for new endmarks (necessarily arrowheads), every edge between a vertex in $D$ and a vertex in $\IV$ has an arrowhead endmark at the vertex in $\IV$ and thus there can be no anterior walk violating the anterial endmark property, (2) for old endmarks, the addition of an arrowhead does not create any new anterial walks. Thus, by Lemma~\ref{lem:simple:anterial-graph::anterial-endmark} the graph is anterial.
\end{proof}


The anterior sets of a graph $\defgraphG$ is the set 
$$A(G)=\setcomprehensionLSE{A}{\powerset{V}}{A=\ant(A,G)}.$$

Let anterial graph $G$ have anterior sets $\defAcomponents$. The pair $\seq{A(G),\subseteq}$ defines a bounded partially ordered set closed under union, that is, a join-semilattice with bottom. The bottom is the empty set $\emptyset$ and the top if the set $V$. The graph $G$ defined to be $a\tailarrow b \tailarrow c$ has $A(G)=\set{\emptyset, \set{a},\set{a,b},\set{a,b,c}}$ which illustrates that $A(G)$ is not necessarily closed under intersection.

The following proof uses well-founded induction. An alternative approach is to prove by induction on the number of chain components in the equivalent anterial graph --- one must exist by Theorem~\ref{thm:simplify:separable-graph:equivalent-anterial-graph}. This approach, however, requires defining chain components and yields a proof that is similar to the proof below.

\begin{lemma}\label{lem:sgi-removes-all-edges-between-separable}
    If $\genstruct$ is essentially separable, $a$ and $b$ are separable in $\genstruct$ then $a\not\in \adj(b,\Gfinal)$ where $\Gfinal=\sgi(\imodel{\genstruct})$.
\end{lemma}

\begin{proof} 

    Let $\genstruct$ be an essentially separable graph,  $\Gfinal=\sgi(\imodel{\gendist})$.

    let $G$ be a separable anterial graph such that $G \equiv \genstruct$; one exists due to the fact that $\genstruct$ has an equivalent separable graph and Theorem~\ref{thm:simplify:separable-graph:equivalent-anterial-graph}.

    First we recast the lemma in terms of separation completeness.
    A graph $\Gfinal$ is \emph{separation complete} with respect to graph $G$ and vertex set $A\subseteq V$ if for all pairs of distinct vertices $a,b \in A$, if $a\not\in \adj(b,G)$ then $a \not\in \adj(b,\Gfinal)$.
    Thus the lemma holds if we can show that $\Gfinal$ is separation complete with respect to $G$ and $V$.

    Next, note that $\seq{A(G), \subseteq}$ is well-founded as $A(G)$ is a finite set and $\subseteq$ is a partial order over $A(G)$.    
    We prove the claim by well-founded induction on $A(G)$ with respect to $\subset$. 
    In particular, we prove that $\Gfinal$ is separation complete with respect to $G$ and $A_j$ under the induction hypothesis that $\Gfinal$ is separation complete with respect to $G$ and all $A_i \subset A_j$.

    Define $\max(A)=\setcomprehensionLSE{a}{A}{\forall b\in A, b \oleqG a}.$ 
    Let $A_i = A_j \setminus \max(A_j)$. The set $A_i$ is an ancestral set, that is, $A_i\in A(G)$. By induction assumption, all edges between separable vertices in $A_i$ have been removed. Thus we need only consider pairs of vertices $a,b \in A_j$ that are separable in $G$ and such that at least one of the vertices is in the set $\max(A_j)$. \Wlog, assume $a\in \max(A_j)$. 

    \proofcase{1} $a \oeqG b$. In this case, from Lemmas~\ref{lem:anterior-adjacency-preserving} and \ref{lem:anterial:ordered:anterior-adjacent-separate} we know that $\antadj(a,\Gfinal)$ and $\antadj(b,\Gfinal)$ separate $a$ and $b$ and at least one of them must be smaller then $\sepsetbound(\Gfinal)$.    

    \proofcase{2} $a \olessG b$ or $b \olessG a$. \Wlog, assume $b \olessG a$. In this case, from Lemmas~\ref{lem:anterior-adjacency-preserving} and \ref{lem:anterial:eqIG:anterior-adjacent-separate} we know that $\antadj(a,\Gfinal)$ separates $a$ and $b$ and that the set is smaller then $\sepsetbound(\Gfinal)$.

    \proofcase{3} $a \oincompG b$. Let $A_b= A_j \setminus \oequivclassG{a}$ and $A_a = A_j \setminus \oequivclassG{b}$. In this case we have that $b\in A_b$, $a\in A_a$, $A_a \subset A_j$, $A_b \subset A_j$, $A_a\in A(G)$, and $A_b\in A(G)$.

    From Proposition~\ref{prop:induced-adj-separation}, we know that $\indep{a}{b}{C}{G}$ when $C=\aiadj(a,b,G)$ and when $C=\aiadj(b,a,G)$.

    \Wlog, we show that $\aiadj(a,b,\Gfinal) \supseteq \aiadj(a,b,G)$.

    Suppose that this is not the case and there is a vertex $d\in \aiadj(a,b,G) \setminus \aiadj(a,b,\Gfinal)$.

    \proofcase{3.1} $d=b$. This cannot be the case as there would be a self-inducing walk between $a$ and $b$ and $a$ and $b$ would not be separable.

    \proofcase{3.2} $d\in \antadj(a,G)$. In this case, by Lemma~\ref{lem:anterior-adjacency-preserving}, we have that $\antadj(a,\Gfinal)\subseteq \antadj(a, G)$ so this cannot be the case.

    \proofcase{3.3} $d \in \aiadj(a,b,G)\setminus \antadj(a,G)$. In this case, there must be an inducing walk between $a$ and $d$ that is open given $b$ in $G$. Let $\omega$ be a minimal inducing walk between $a$ and $d$ in $G$ that is open given $b$. There is a walk $\omega'$ in $\Gfinal$ with the same vertex sequence. We show that $\omega'$ is an inducing walk open given $b$ in $\Gfinal$.

    From Lemma~\ref{lem:equiv-implies-tails-in-Gfinal}, we know that every collider section on $\omega$ corresponds to an undirected walk on $\omega'$.

    From Lemma~\ref{lem:open-walk-given-C:every-vertex-in-anterior}, every vertex is anterior to either $a$ or $b$ and thus every vertex on $\omega$ is in $A_j$ as, 

    From Lemma~\ref{lem:separable:miw:every-internal-section-discriminated}, every collider section on $\omega$ is discriminated by some \diw\ that is a subwalk of $\omega$.

    Consider a \diw\ between $v_i$ and $v_k$ on $\omega$ that discriminates a section $\sigma$ and let $A_{ik}=\ant(\set{v_i,v_k},G)$. Clearly the set is an anterior set in $G$ (i.e., $A_{ik}\in A(G)$).
    
    We know that $\sigma \cap A_{ik}=\emptyset$. This implies that $A_{ik}\subseteq A_j$. From the induction hypothesis, the edge between $v_i$ and $v_j$ has been removed. By Lemma~\ref{lem:remove-cover-edge-of-diw:identify-section}, The arrowheads into $\sigma$ on $\omega$ are oriented in $\omega'$. This implies that $\omega'$ is an inducing walk in $\Gfinal$. Finally, by Lemma~\ref{lem:anterior-preserving}, $d\in \aiadj(a,b, \Gfinal).$ This implies that $\aiadj(a,b,\Gfinal)\subseteq \aiadj(a,b,G)$.

    From this it must be the case that the edge between $a$ and $b$ must be removed in $\Gfinal$ as there is a subset of $\aiadj(a,b,G)$ which will be found as $socs  \geq \sepsetbound(\Gfinal)>|\aiadj(a,b,G)|$.
\end{proof}

\reftheorem{thm:SGI:identifies-graph}

\begin{proof}
    Assume $\genstruct$ is an essentially separable graph.     
    \newcommand{\genstructequiv}{G^*_{sep}}
    Let $\genstructequiv \equiv \genstruct$ be a separable graph equivalent to $\genstruct$, $G_{SGI}=\sgi(\imodel{\genstruct})$ and $G_{IND}=\inducedarrowheads(\genstructequiv)$.

    We show $\sgi(\imodel{\genstruct})=\inducedarrowheads(\genstructequiv)$, that is, $G_{SGI}=G_{IND}$.

    First we show that $G_{SGI}$ and $G_{IND}$ have the same adjacencies.
    From Lemma~\ref{lem:perfect-testing:remove-only-separable} we have  $\adjrep(G_{SGI}) \supseteq \adjrep(\genstructequiv)$ and from 
    Lemma~\ref{lem:sgi-removes-all-edges-between-separable} we have $\adjrep(G_{SGI}) \subseteq \adjrep(\genstructequiv)$.
    Thus $G_{SGI}$ and $\genstructequiv$ have the same adjacencies. 
    Next, $G_{IND}$ and $\genstructequiv$ have the same adjacencies and \miw s by Theorems~\ref{thm:induced-arrowheads:sound:identifying} and \ref{thm:separable-graphs:equivalent::same-adj-and-same-miw}.
    Combining these facts we have that $G_{SGI}$ and $G_{IND}$ have the same adjacencies.

    Second, we show that $G_{SGI}$ and $G_{IND}$ have the same arrowheads.
    
    We identify the arrowheads in $G_{IND}$ with induced arrowheads in $\genstructequiv$.
    By construction, $G_{IND}$ has an arrowhead at $b$ on edge $\edgedb$ if and only if $\edgedb$ has an induced arrowhead at $b$ in $\genstructequiv$.
    
    Next we identify the arrowheads in $G_{SGI}$ with induced arrowheads in $\genstructequiv$.

        
    Every arrowhead in $G_{SGI}$ is to due to an induced arrowhead in $\imodel{\genstruct}$, or equivalently, $\imodel{\genstructequiv}$.
    Next we show that if edge $\edgedb$ has in \iarrowhead\ at $b$ in $\genstructequiv$ then there is an edge with $\edgeprimedb$ in $G_{SGI}$ with an arrowhead at $b$.
    From the argument above, there must be an edge $\edgeprimedb$ in $G_{SGI}$.
    From Lemma~\ref{lem:induced-arrowhead:d-in-anterior-of-some-separating-pair}, there exist vertices $i$ and $j$ and vertex set $D$ such that $b\in \IV(i,j,D,\genstructequiv)$ and $d\in \antsetij$. At some stage, the \sgi\ algorithm will identify a minimal separating set $C$ for $i$ and $j$. After identifying the set $C$ that separates $i$ and $j$ the algorithm will add the arrowhead at $b$ on edge $\edgeprimedb$ due to the fact that it will, determine that $d\not\in \IV(i,j,G_{SGI})$ by testing $\indep{i}{j}{Cd}{\genstruct}$ which must hold by Lemma~\ref{lem:indep-model-graph:ordered-upward}.
    This implies that in $G_{SGI}$ there is an arrowhead at $b$ on edge $\edgedb$ if and only if $\edgedb$ has an induced arrowhead at $b$ in $\genstructequiv$.    
    This implies that $G_{SGI}$ and $G_{IND}$ have the same arrowheads.
    
    From the fact that $G_{SGI}$ and $G_{IND}$ are simple graphs, have the same adjacencies and same arrowheads we have that $G_{SGI}=G_{IND}$.
\end{proof}

\refcorollary{cor:SGI:identifies}

\begin{proof}
    Assume $\genstruct \in \WSepFamily$ is an essentially separable graph and that $\imodel{\gendist}$ provides perfect testing with respect to $\genstruct$.

    \newcommand{\genstructequiv}{G^*_{sep}}
    Let $\genstructequiv \equiv \genstruct$ be a separable graph equivalent to $\genstruct$, $G_{SGI}=\sgi(\imodel{\gendist})$ and $G_{IND}=\inducedarrowheads(\genstructequiv)$.
    From perfect testing, we have $\sgi(\imodel{\gendist}) = \sgi(\imodel{\genstructequiv})$. 
    From Theorem~\ref{thm:SGI:identifies-graph},
    we have $\sgi(\imodel{\gendist})=\inducedarrowheads(\genstructequiv)$.
    From Theorem~\ref{thm:induced-arrowheads:sound:identifying}, 
    for all separable graphs $G,G'$ it is the case that  $\sgi(\imodel{G}) = \sgi(\imodel{G'})$ if and only if $G \equiv G'$.
    Thus if $\genstruct \in \WSepFamily$ then $\sgi$ identifies the equivalence class $\equivclassSL{\WSepFamily}{\genstruct}$ under perfect testing.
\end{proof}

\reftheorem{thm:polynomial-bound}

\begin{proof}
    Let $\genstruct$ be separable and $\gendist$ provide perfect testing for $\genstruct$.
    Let $\Gfinal = \sgi(\imodel{\gendist})$ and note that $G=\Gfinal= \inducedarrowheads(\genstruct)$ by Corollary~\ref{cor:SGI:identifies}.
    In this case, the algorithm terminates when $socs \geq \sepsetbound(\Gfinal)$.
    This implies that we consider at most $O(|V|^{\sepsetbound(G)})$ independence tests to guarantee that edges between separable pairs of vertices have been removed. 
    In addition, the determination of the set of inducing vertices for a separated pair of vertices is bounded by $O(|V|).$ Thus the total number of independence tests is bounded by $O(|V|^{\sepsetbound(G)+1})$.
\end{proof}

\end{document}